\documentclass[11pt]{article}
\ifdefined\pdfminorversion\pdfminorversion=7\fi

\usepackage[a4paper,margin=1in]{geometry}

\usepackage{amsmath,amssymb,amsthm,mathtools,bm}

\usepackage[protrusion=true,expansion=false]{microtype}
\usepackage{enumitem}
\usepackage{booktabs}
\usepackage{longtable}
\usepackage{graphicx}
\usepackage{xcolor}
\usepackage{float}

\usepackage{algorithm}
\usepackage[noend]{algpseudocode}

\usepackage[round,authoryear]{natbib}

\usepackage{tikz}
\usetikzlibrary{
arrows.meta,
positioning,
calc,
backgrounds,
fit
}

\usepackage{pgfplots}
\pgfplotsset{compat=1.17}

\usepackage[most]{tcolorbox}
\usepackage{subcaption}

\usepackage[
colorlinks=true,
linkcolor=black,
citecolor=blue!55!black,
urlcolor=blue!55!black
]{hyperref}

\usepackage{titletoc}

\theoremstyle{plain}
\newtheorem{theorem}{Theorem}
\newtheorem{lemma}{Lemma}
\newtheorem{proposition}{Proposition}
\newtheorem{corollary}{Corollary}

\theoremstyle{definition}
\newtheorem{assumption}{Assumption}
\newtheorem{definition}{Definition}

\theoremstyle{remark}

\newcommand{\X}{\mathcal{X}}
\newcommand{\D}{\mathcal{D}}
\newcommand{\C}{\mathcal{C}}

\newcommand{\E}{\mathbb{E}}
\newcommand{\Prob}{\mathbb{P}}
\newcommand{\KL}{\mathrm{KL}}

\newcommand{\acq}{a}
\newcommand{\regret}{\mathrm{Reg}}

\DeclareMathOperator*{\argmax}{arg\,max}

\definecolor{knownblue}{RGB}{221,235,247}
\definecolor{uncertainorange}{RGB}{252,229,205}
\definecolor{latentgreen}{RGB}{226,239,218}
\definecolor{unknowngray}{RGB}{238,238,238}

\definecolor{deepblue}{RGB}{45,92,136}
\definecolor{deeporange}{RGB}{174,92,28}
\definecolor{deepgreen}{RGB}{63,117,65}
\definecolor{deepgray}{RGB}{85,85,85}

\definecolor{chatred}{RGB}{247,214,214}
\definecolor{chatredline}{RGB}{193,66,66}
\definecolor{chatdeep}{RGB}{140,38,38}
\definecolor{chathdr}{RGB}{233,182,182}

\definecolor{reasonblue}{RGB}{210,228,246}
\definecolor{reasonblueline}{RGB}{52,104,160}
\definecolor{reasondeep}{RGB}{30,66,110}
\definecolor{reasonhdr}{RGB}{180,208,238}

\definecolor{discgreen}{RGB}{221,237,212}
\definecolor{discgreenline}{RGB}{58,112,60}
\definecolor{discdeep}{RGB}{34,74,36}
\definecolor{dischdr}{RGB}{194,224,184}

\definecolor{axisgray}{RGB}{120,120,120}
\definecolor{bodygray}{RGB}{55,55,55}

\definecolor{ldmred}{HTML}{A33A3A}
\definecolor{ldmdarkred}{HTML}{7F2929}
\definecolor{ldmtext}{HTML}{282323}
\definecolor{ldmbg}{HTML}{FBF5F5}
\definecolor{ldmedge}{HTML}{E8D3D3}

\begin{document}


\begin{tcolorbox}[
    enhanced,
    breakable,
    colframe=ldmedge,
    boxrule=0.45pt,
    arc=12pt,
    left=14pt,
    right=14pt,
    top=14pt,
    bottom=13pt,
    before skip=0pt,
    after skip=15pt,
    interior style={
      shade,
      shading angle=315,
      left color=white!97!ldmbg,
      right color=ldmred!7!ldmbg
    }
]
    {
    \centering
    {\color{ldmdarkred}\sffamily\bfseries
      {\LARGE Large Discovery Models:}\par
    
      \vspace{4pt}
    
      {\Large
        Empirically-Grounded Model-Based Open-Ended Search}\par
    }
    
    \vspace{11pt}
    
    {\color{ldmtext}\sffamily
      {\large\bfseries
        Zhongwei Yu\,$^{*26}$\quad
        Yan Song\,$^{*16}$\quad
        Xue Yan\,$^{*6}$\quad
        Anjie Liu\,$^{\ddagger26}$\quad
        Xingyu Lu\,$^{\ddagger26}$\quad \\
        Yihang Chen\,$^{\ddagger16}$\quad 
        Huichi Zhou\,$^{16}$\quad
        Siyuan Guo\,$^{4}$\quad 
        Luoyang Sun\,$^{36}$\quad
        \\
        Sihan Chen\,$^{6}$\quad
        Xiangning Yu\,$^{56}$\quad
        Jun Wang\,$^{1\dagger}$
      }\par
    
      \vspace{6pt}
    
      {\small
        $^{1}$ University College London\\
        $^{2}$ The Hong Kong University of Science and Technology (GZ)\quad 
        $^{3}$ Institute of Automation, CAS\\
        $^{4}$ Jilin University\quad
        $^{5}$ Tianjin University\quad
        $^{6}$ AI Lab, The Yangtze River Delta\\
        $^{*}$Equal contribution as co-first authors. \quad
        $^{\ddagger}$Equal contribution as co-second authors.\\
        $^\dagger$ Corresponding author. \texttt{jun.wang@ucl.ac.uk}.
      }\par
    }
    }
    
    \vspace{10pt}
    
    {\color{ldmedge}\hrule height 0.5pt}
    
    \vspace{10pt}
    
    {\color{ldmdarkred}\sffamily\bfseries\large
    Abstract
    }\par
    
    \vspace{4pt}
    
    {\color{ldmtext}\small
    Scientific discovery often involves optimising expensive-to-evaluate objectives over vast, structured, and open-ended hypothesis spaces, such as molecules, protein sequences, and computer programs. Generative models such as large language models (LLMs) provide expressive priors over such spaces, but their likelihoods and self-assessments are unreliable proxies for the objectives and calibrated epistemic uncertainty, especially for novel candidates outside the observed data distribution. We introduce the Large Discovery Model (LDM), an empirically grounded recurrent architecture that couples a generative model with a Bayesian non-parametric reward surrogate model. The generative model proposes and refines candidate designs, while the surrogate predicts their performance and quantifies uncertainty, yielding an uncertainty-aware value that guides candidate generation, refinement, and selection. The discovery memory and the surrogate model are continually updated as each new experimental observation arrives. We evaluate LDM on three scenarios spanning different design modalities and objectives, including neural-network training, antibody design, and molecular optimisation. Compared to LLM-only reflection or traditional statistical search across these domains, LDM achieves a $2.4\times$ greater reduction in validation BPB, an $18.2\%$ relative decrease in binding energy, and more than $60\%$ relative gains in molecular multi-objective performance. These results suggests that LDM could serve as a general-purpose discovery engine for effective search over open-ended hypothesis spaces. 

\vspace{0.5em}

\textit{This paper serves as an initial release of LDM (v0.1), with code and benchmarks available at \url{https://github.com/yzailab/Large-Discovery-Models}.}
\par
    }
\end{tcolorbox}

\vspace{-30pt}
\begin{figure}[H]
    \centering
    \includegraphics[width=\textwidth]{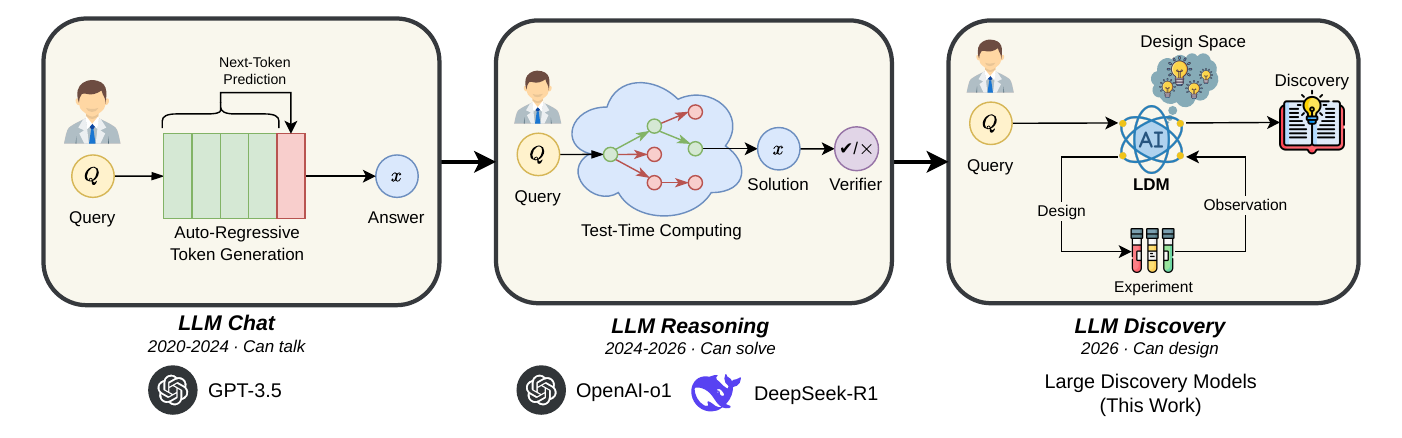}
    \vspace{-20pt}
    \caption{
    \textbf{From generation and reasoning to open-ended discovery.}
    \emph{LLM generation} produces responses through autoregressive prediction.
    \emph{LLM reasoning} adds inference-time computation (e.g., chain-of-thought,
    tree search, or best-of-$N$) with a cheap verifier to select solutions
    \citep{brownLanguageModelsAre2020,openaiOpenAIO1System2024,wang2024openr,
    deepseek-aiDeepSeekR1IncentivizingReasoning2025}.
    \emph{Large Discovery Models} augment LLMs with an empirically grounded signal
    that guides exploration and discovery, enabling search in open-ended scientific 
    domains under expensive black-box evaluation.
    }
    \label{fig:evolution}
\end{figure}

\section{Introduction}
\label{sec:intro}

The development of large language models (LLMs) has progressed from
open-ended generation towards increasingly capable forms of reasoning.
Pretraining provides broad linguistic and conceptual knowledge, while
inference-time scaling allows a model to generate, compare, and refine
multiple candidate solutions. This approach has produced substantial gains in
mathematics, coding, theorem proving, and game playing, where solutions may be
difficult to find but comparatively easy to evaluate. Formal derivations can
be assessed by process or outcome verifiers, programs can be executed against
unit tests, and game strategies can be evaluated using exact rules or
simulators
\citep{yao2023tree,madaan2023selfrefine,snell2025scaling}.
AlphaZero-like methods make this feedback loop explicit by combining
generation, evaluation, and search
\citep{silver2017mastering,fengAlphazerolikeTreeSearchCan2024}.
OpenR~\citep{wang2024openr} provides an open framework for studying such
reasoning systems, while AlphaEvolve~\citep{novikov2025alphaevolve} couples
language-model proposals with automated evaluators at scale. Together, these
developments demonstrate that inference-time computation is particularly
effective when reliable and repeatable feedback can turn generation into
directed search.

Scientific discovery poses a more challenging setting. The candidate space may be vast, structured, and only partially specified, while feedback often depends on expensive simulation, accelerator experiments, physical assays, or real-world experiment. Evaluations cannot be repeated freely, and the resulting feedback may be delayed or noisy. A discovery system must therefore decide not only which candidates to evaluate, but also which hypotheses and designs should enter the search process in the first place. Figure~\ref{fig:evolution} illustrates this progression from generation and reasoning to discovery.

\begin{figure*}[th]
    \centering
    \includegraphics[width=\textwidth]
    {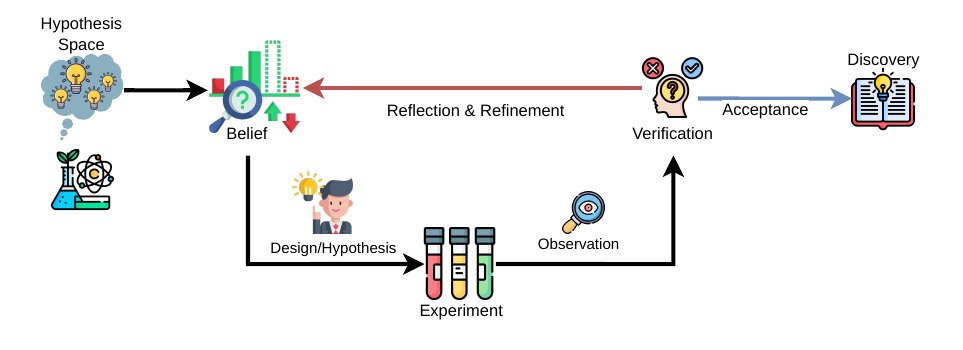}
    \caption{
    \textbf{The empirical scientific discovery loop.}
    An agent generates candidate hypotheses or designs, selects a subset for
    costly evaluation, observes the outcomes, and uses those observations to
    update its beliefs and guide the next round of search, as described in
    Definition~\ref{def:scientific-discovery}.
    }
    \label{fig:scientific-discovery}
\end{figure*}

\begin{definition}[Empirical Scientific Discovery]
\label{def:scientific-discovery}
\emph{Scientific discovery, in the setting considered here, is a sequential process
in which an agent generates hypotheses or designs, selects a limited number
for evaluation through interaction with an external evaluator, and uses the
resulting empirical observations to update its beliefs and guide subsequent
search. A discovery occurs when this
process identifies a previously unknown and non-trivial relationship,
mechanism, or design that is supported by empirical evidence.}\footnote{This is an operational definition of empirically grounded,
search-based scientific discovery: the class of scientific problems that can
be formulated as sequential search over hypotheses or designs, with selected
candidates evaluated through physical experiments, simulations, computational
tests, or other external sources of evidence. It does not encompass all forms
of scientific discovery, particularly those for which the relevant search
space, evaluation procedure, or objective cannot yet be specified. Moreover,
\emph{unknown} and \emph{non-trivial} are relational notions, defined relative
to the agent's prior knowledge and, where appropriate, the knowledge of the
relevant scientific community. A result may therefore be a discovery for an
agent while constituting a rediscovery from the community's perspective.}
\end{definition}

Definition~\ref{def:scientific-discovery} highlights two differences between
scientific discovery and reasoning with cheap verifiers. First, scientific
search is often \emph{open-ended}: relevant candidates may lie outside the
region represented by the current model or reachable through its existing
search operators. Second, empirical feedback is \emph{budgeted}: only a small
fraction of the generated candidates can be evaluated. The agent must
therefore learn from limited observations, model an unknown objective, and
allocate its experimental budget to candidates that are expected to be useful
or informative
\citep{shahriari2016taking,frazier2018tutorial}.

Current LLM-based research systems address only part of this problem.
Autonomous research agents can generate hypotheses, write code, execute
experiments, analyse results, and produce scientific manuscripts
\citep{lu2024aiscientist}. LLM-generated research proposals have also been
judged more novel than proposals written by human experts
\citep{si2024novelideas}. However, these proposals exhibit limited diversity
and weaker feasibility, and LLMs do not reliably evaluate their own ideas.
End-to-end assessments similarly identify persistent limitations in
experimental execution, novelty assessment, and methodological reasoning
\citep{beel2025aiscientist,xu2026researchclawbench,
starace2025paperbench}. Thus, fluent generation and procedural automation do
not by themselves provide the empirically grounded value signal needed for
scientific discovery.

This limitation also appears in scientific design. An LLM assigns high
probability to candidates that are plausible under its learned distribution,
but linguistic or sequence likelihood need not correlate with an externally
measured scientific property
\citep{guptaLLMsBayesianOptimization2025,chang2025llinbo}. LLM-only search can therefore produce
valid candidates without reliably determining which ones merit expensive
evaluation. Bayesian optimisation (BO) offers a complementary capability. Given a limited
set of observations, BO learns a probabilistic surrogate of the unknown
objective and uses an acquisition function to balance predicted performance
against epistemic uncertainty
\citep{jonesEfficientGlobalOptimization1998,srinivas2010gaussian,shahriari2016taking,
frazier2018tutorial}. It can therefore prioritise promising or informative
experiments under a limited evaluation budget. However, BO still requires a
representation of the design space and a mechanism for proposing or reaching
candidates. In structured and combinatorial spaces, such as molecules,
proteins, programs, and experimental protocols, optimising the acquisition
function may itself be difficult
\citep{balandat2020botorch,griffiths2020constrained}. BO can efficiently rank
the candidates exposed by its current representation and search operators,
but it may fail to reveal valuable candidates outside that search support.

These observations expose complementary limitations. LLMs can generate and
modify structured candidates, but lack a reliable estimate of their external
scientific value. BO grounds candidate selection in empirical observations,
but its effectiveness depends on which candidates its representation and
search procedure can expose. Scientific discovery requires both capabilities:
the search frontier must expand, and experimental resources must be allocated
using an uncertainty-aware value signal that balances expected value with information gain.
\begin{figure*}[t!]
    \centering
\includegraphics[width=\textwidth]
{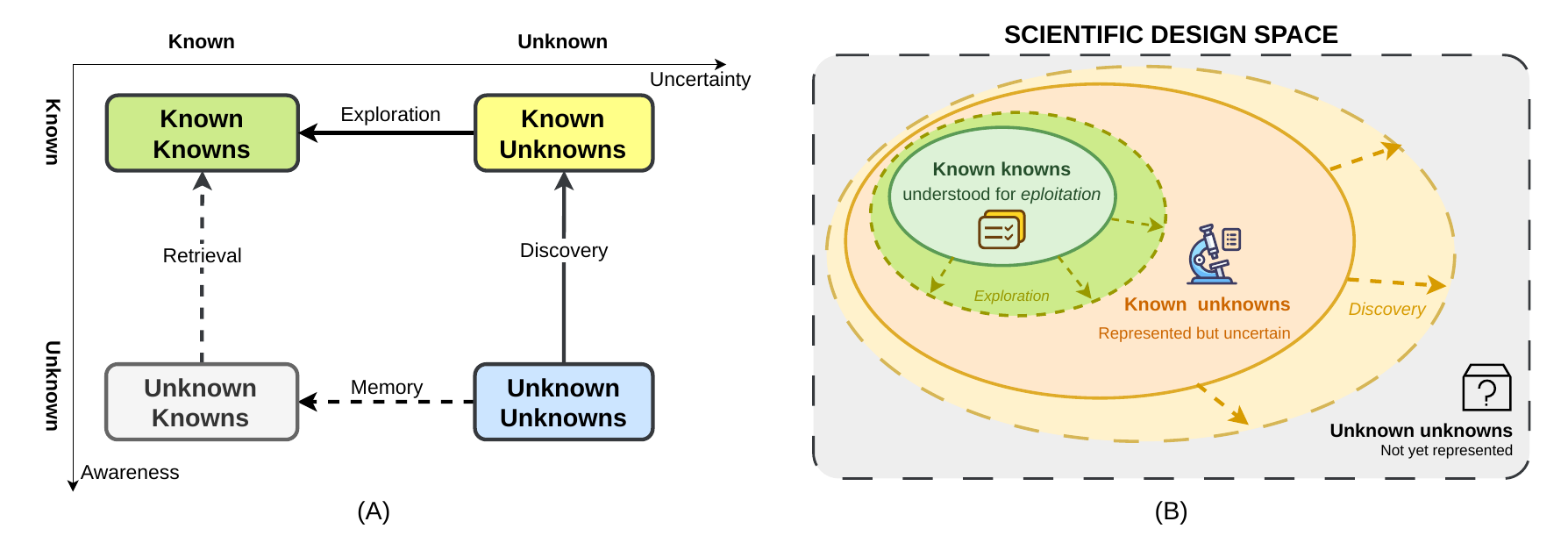}
    \caption{
    \textbf{Epistemic regimes of experimental scientific search.}
    \textbf{(A)} An adaptation of the Rumsfeld Matrix
\citep{krogerus2018decision}, organised by search awareness and objective
    uncertainty.
    \textbf{(B)} A scientific agent moves among three regimes in the design
    space. 
    }
    \label{fig:knowledge-regimes}
\end{figure*}

Thus, to combine the strength of both methods, in this paper, we formulate scientific discovery as \emph{sequential inverse
design}. Given a vast, structured, and potentially unbounded design space, an
agent must repeatedly decide which candidate to evaluate next under a limited
experimental budget. Because the reward function is unknown and the design
space cannot be enumerated, it is generally infeasible to model rewards
accurately over the entire space. The agent therefore maintains a
\emph{probabilistic surrogate}, learned from the candidates evaluated so far,
that predicts both the reward of a candidate and the epistemic uncertainty of
that prediction. An acquisition function converts these beliefs into an
estimate of the value of evaluating or further investigating each candidate,
thereby balancing the exploitation of promising designs against the
exploration of uncertain ones.

This formulation highlights an important distinction between
\emph{predictive uncertainty} and \emph{search awareness}. Predictive
uncertainty concerns what the surrogate does not yet know about candidates
within its current modelling support and can be reduced by collecting
informative observations. Search awareness instead concerns which candidates,
design families, or structural relationships are represented by, or reachable
through, the proposal process. These two limitations give rise to three
different modes of search. \emph{Exploitation} prioritises reachable candidates
with high predicted rewards. \emph{Exploration} evaluates uncertain candidates
within the existing search support to improve the surrogate. \emph{Discovery}
changes the support itself, exposing previously unrepresented candidates,
design families, or relationships that neither exploitation nor exploration
could otherwise reach. Figure~\ref{fig:knowledge-regimes} illustrates this perspective by adapting the four epistemic regimes to scientific design.

To realise this process, we introduce the \emph{Large Discovery Model} (LDM),
an empirically grounded recurrent architecture that couples an LLM proposal
model with a continually updated probabilistic surrogate. At each round, the
LLM generates, edits, and refines structured candidates, thereby constructing
and reshaping a dynamic search frontier. The surrogate predicts the empirical
rewards of these candidates and quantifies its epistemic uncertainty about
them. An acquisition function then uses these predictions to decide where to
allocate further candidate refinement, inference-time computation, and costly
empirical evaluations. The resulting observations are added to the data,
updating the surrogate and producing new acquisition signals that redirect
subsequent generation and search.

The acquisition function is therefore not merely a post-generation filter. It
provides an empirically grounded value signal that governs the allocation of
both computational and experimental resources. The LLM determines which
structured candidates can be proposed and reached, while the surrogate and
acquisition function determine which parts of this evolving frontier are most
valuable to pursue. Through this recurrent interaction, LDM extends
inference-time scaling beyond domains with cheap and repeatable verifiers to
scientific-design problems in which feedback is costly, noisy, sparse, and
available only through a limited number of empirical evaluations.

\begin{figure*}[t]
    \centering
    \begin{subfigure}[t]{0.42\textwidth}
        \centering
        \includegraphics[width=\linewidth]
        {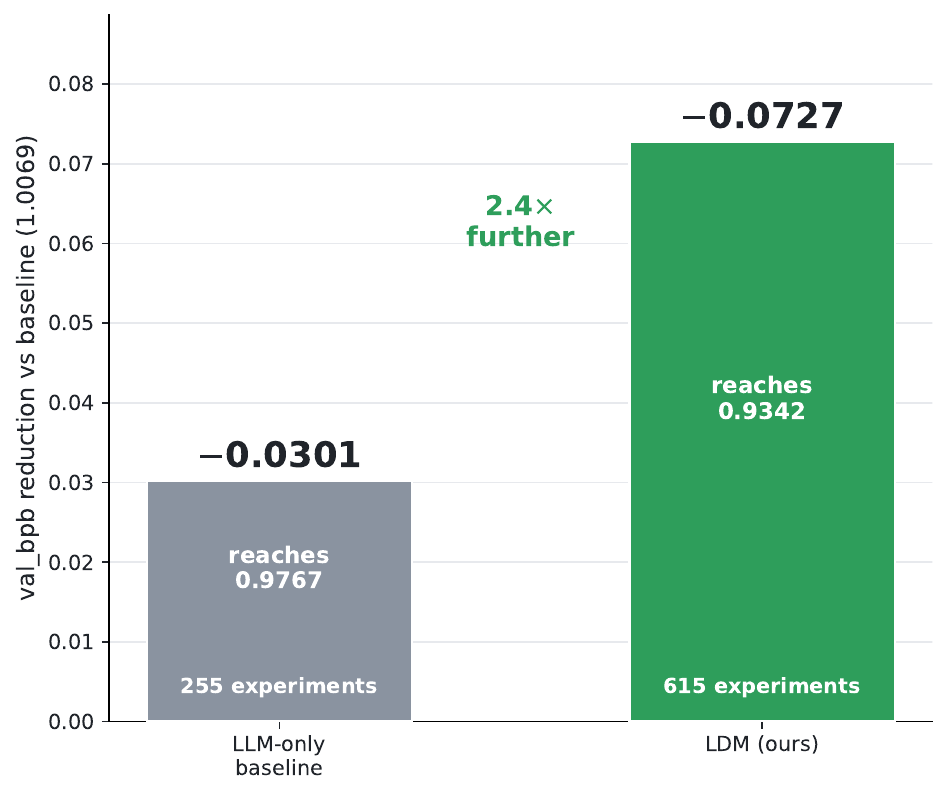}
        \caption{Performance on AutoResearch.}
        \label{fig:discovery-motivation-autoresearch}
    \end{subfigure}
    \hfill
    \begin{subfigure}[t]{0.42\textwidth}
        \centering
        \includegraphics[width=\linewidth]
        {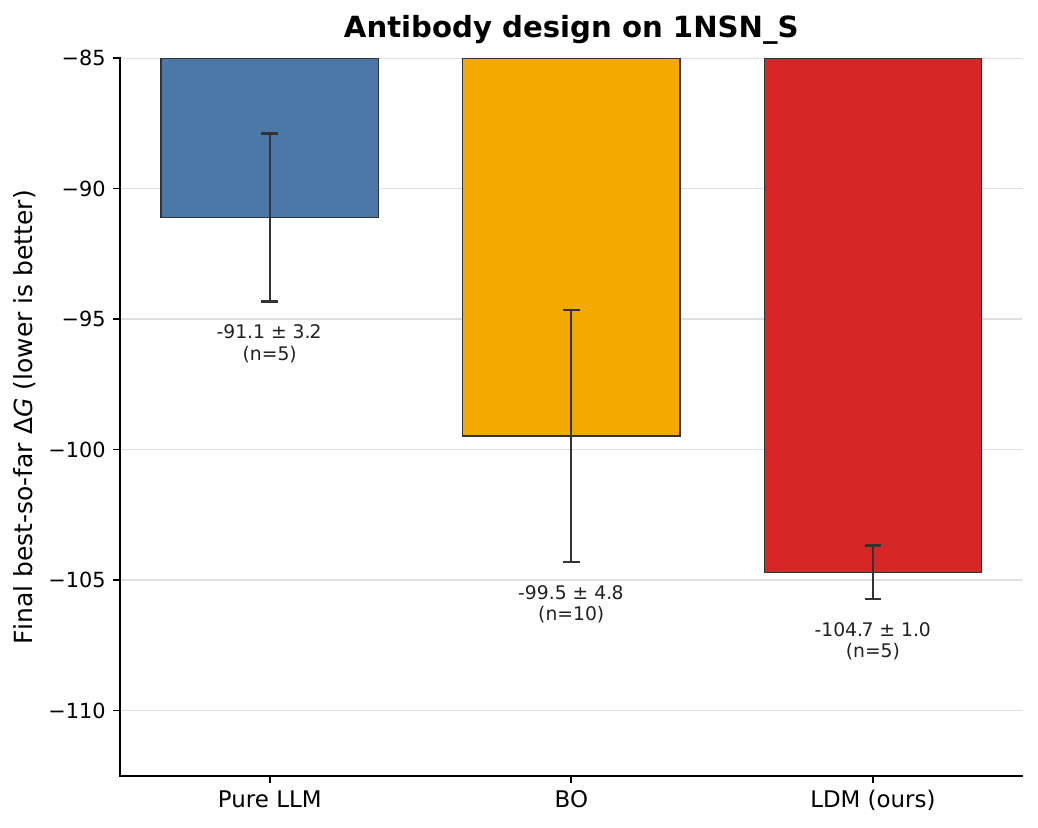}
        \caption{Performance on antibody design.}
        \label{fig:discovery-motivation-antibody}
    \end{subfigure}
    \caption{
    \textbf{Complementary limitations of LLM-only and BO-based search.}
    \textbf{(a)} LDM improves upon the LLM-only AutoResearch baseline under
    the same H100 hardware setting.
    \textbf{(b)} Final best-so-far binding energy on the antibody-design task,
    where LDM outperforms the LLM-only and BO-only baselines. Lower binding
    energy is better. For details, we refer to \S~\ref{sec:eval}.
    }
    \label{fig:discovery-motivation}
\end{figure*}

The proposed LDM has a general discovery capability. To demonstrate this, we evaluate LDM in three diverse structured scientific-design domains:
neural-network training-program search on AutoResearch
\citep{karpathy2026autoresearch}, antibody CDRH3 sequence design, and
multi-objective molecular optimisation. LDM exceeds the best previously
reported AutoResearch result. With a weak language-model prior, it remains
competitive with AntBO~\citep{khan2022antbo}, a specialised BO method for
antibody design. It also achieves substantial improvements over both LLM-only
and BO-only baselines in multi-objective molecular design. On AutoResearch, LDM achieves an approximately $2.4\times$ larger absolute validation-BPB reduction than LLM-only reflection from a common starting point. On antibody design, it achieves an $18.2\%$ lower mean binding energy than LLM-only reflection after $200$ evaluation steps. On molecular optimisation, it improves the hypervolume of the Pareto front by $62.4\%$ over LLM-only reflection and by $63.1\%$ over classical Bayesian optimisation. Figure~\ref{fig:discovery-motivation} presents selected
results, with the full experimental evaluation deferred to
\S~\ref{sec:eval}.

Our contributions are fourfold. First, we formulate scientific discovery as
inference-time search over a dynamic, structured design space, guided by an
uncertainty-aware value model grounded in empirical observations. Second, we
introduce an acquisition-tilted search policy that combines the structured
prior of an LLM with inference-time computation and sequential experimental
feedback. Third, we provide a regret analysis that separates surrogate-model error from the suboptimality of the LLM-based search process. Finally, we
demonstrate the framework across three classes of scientific objects:
computer programs, biological sequences, and molecules.

The remainder of the paper is organised as follows.
\S~\ref{sec:ldm} continues the discussion about the search and discovery in open-ended space.
\S~\ref{sec:method} derives the LDM formulation and its acquisition-tilted search method. \S~\ref{sec:algo} presents the practical algorithm.
\S~\ref{sec:theory-main} provides the regret analysis.
Finally, \S~\ref{sec:eval} reports the empirical results.

\section{Search and Discovery in Open-Ended Design Spaces}
\label{sec:ldm}

In this work, discovery takes place in an \emph{open-ended hypothesis space},
rather than over a fixed and fully enumerated input domain. We seek a design or
hypothesis $x$ with high unknown reward $R(x)$, ideally approaching
$x^\star \in \argmax_{x\in\X} R(x)$, where $\X$ denotes the evolving universe
of hypotheses that may become expressible during search. The agent initially
has access only to a reachable subset $\X_0 \subset \X$ and must expand or
reshape this subset by generating new candidates, representations, and design
families. The challenge is therefore twofold: evaluating $R$ is costly,
black-box, and potentially noisy, while the relevant search space is itself
not fully specified in advance but must be progressively constructed through
the discovery process.

This differs from classical numerical optimisation or combinatorial search,
which generally assumes a fixed and explicitly specified domain together with
a gradient, numerical oracle, or known set of valid search operations. In
scientific discovery, the full hypothesis space may not be enumerable, and the
search may need to formulate hypotheses that lie beyond its current
representation or reach.

To reason about what is currently known, we introduce
\emph{a probabilistic surrogate}, i.e.\ a posterior
$p(R\mid\mathcal{D}_t)$ over the reward surface $R$ conditioned on the dataset
$\mathcal{D}_t=\{(x_i,r_i)\}_{i=1}^{t-1}$ of previously evaluated designs and
their noisy feedback $r_i\approx R(x_i)$. Its predictive belief is usually
summarised by posterior moments $(\mu_t(x),\sigma_t(x))$, representing the
expected reward and predictive uncertainty on the surrogate's modelling
support. The surrogate therefore describes the search's current empirical
knowledge, but it does not by itself determine or enumerate the full
hypothesis space.

Using this surrogate-relative notion of knowledge, we partition $\X$ into four
epistemic regimes, as illustrated in
Figure~\ref{fig:knowledge-regimes}.

\begin{itemize}[leftmargin=1.6em,itemsep=2pt,topsep=2pt]
    \item \emph{Known knowns} are evaluated designs whose rewards are
    sufficiently well understood. The surrogate predicts their values with
    low posterior uncertainty $\sigma_t(x)$.

    \item \emph{Known unknowns} are designs that the search can formulate and
    the surrogate can model, but whose rewards remain unresolved. They are
    represented by high posterior uncertainty $\sigma_t(x)$ and can be
    resolved through further evaluation.

\item \emph{Unknown knowns} are designs or evaluation results that exist in
    an external source, such as a database, a previous experiment, or a
    parallel discovery process, but are absent from the current search
    context. They can become available only through an explicit retrieval
    operation.
    
    \item \emph{Unknown unknowns} are designs that lie beyond the search's
    current effective reach. This can occur for two reasons:
    (i)~$\X$ is too large, combinatorial, or effectively unbounded for the current
    search procedure to reach the design; or (ii)~the design lies outside the
    surrogate's modelling support, so that $(\mu_t(x),\sigma_t(x))$ does not
    constitute a meaningful calibrated prediction.
\end{itemize}

These regimes distinguish uncertainty, reachability, modelling, and
knowledge. A design does not become known merely because it belongs to a
mathematically defined space; a reachable design is not necessarily supported
by the surrogate; and information stored externally is not available unless
it is retrieved into the current context. All four regimes may therefore
persist throughout the search.

The \emph{search frontier} is the conceptual boundary of what the current
procedure can effectively formulate, reach, and model. Designs inside this
frontier are accessible to the search, although their rewards may remain
uncertain. Designs outside it have not yet entered the modelled search
process. This distinction gives rise to three operations closely related to \cite{wang2008infinitely}:

\begin{itemize}[leftmargin=1.6em,itemsep=2pt,topsep=2pt]
    \item \emph{Exploitation} acts on known knowns by selecting designs for
    which the surrogate predicts both high reward $\mu_t(x)$ and low
    uncertainty $\sigma_t(x)$.

    \item \emph{Exploration} resolves known unknowns by evaluating reachable
    but uncertain designs. The resulting observations reduce posterior
    uncertainty and may turn known unknowns into known knowns.

    \item \emph{Discovery} expands the search frontier by formulating
    hypotheses that were previously outside the search's effective reach.
    This brings unknown unknowns into the modelled hypothesis space, where they
    become known unknowns that can subsequently be evaluated.
\end{itemize}

Discovery is therefore not simply the evaluation of an untested design. An
untested design that already lies within the surrogate's modelling support is
a known unknown and belongs to exploration. Discovery instead changes what
the search can express, reach, or meaningfully model. Experimental evaluation
then determines the value of the newly surfaced hypothesis.

The three operations correspond to three transitions in
Figure~\ref{fig:knowledge-regimes}: discovery moves hypotheses from unknown
unknowns to known unknowns, exploration moves them from known unknowns to
known knowns, and exploitation operates within the known-known regime. The
unknown-known regime requires separate memory read and write operations. We
leave memory retrieval and federated discovery to future work and focus here
on discovery, exploration, and exploitation within a single sequential search
process.

This formulation is conceptual and does not prescribe a particular discovery
algorithm. The next section introduces the Large Discovery Model, which
instantiates these operations by combining a generative foundation model, the
probabilistic surrogate introduced above, and an acquisition-guided
inference-time search policy.

\section{The Large Discovery Model}
\label{sec:method}

We are now ready to define the \emph{Large Discovery Model} (LDM). LDM is a policy for sequential inverse design built from three interacting components: a generative foundation model, a probabilistic surrogate, and a model-based acquisition function.

Same as before, at iteration $t$,
$\mathcal{D}_t=\{(x_i,r_i)\}_{i=1}^{t-1}
$
denotes the empirical evaluation history, where $x_i$ is a previously evaluated design and $r_i$ is its observed reward. We use $\mathcal{C}_t$ to denote the broader \emph{search context} available to the generative model. This context may include $\mathcal{D}_t$, previously generated candidates, surrogate predictions, acquisition values, constraint feedback, refinement trajectories, and other forms of search memory. Thus, $\mathcal{D}_t$ contains the empirical evidence used to fit the surrogate, whereas $\mathcal{C}_t$ contains the information used to condition candidate generation. For clarity, a compact list of notation is provided in Appendix~\ref{app:notation}.

\paragraph{(i) A generative foundation model.} We denote by
\begin{equation}
    p_{\theta,\alpha}(x\mid\mathcal{C}_t)
\end{equation}
the proposal distribution over plausible designs conditioned on the current search context. Here, $\theta$ denotes the parameters of the underlying generative model, while $\alpha$ denotes its inference configuration, such as the prompting strategy, sampling temperature, reasoning procedure, refinement scheme, or inference-time compute budget. Thus, $p_\theta$ denotes a base model distribution, whereas $p_{\theta,\alpha}$ denotes the effective proposal after the chosen inference configuration has been applied. We use $p_{\theta,\alpha}$ throughout the paper to make explicit that the effective proposal distribution depends not only on the model parameters but also on how the model is used at inference time.

The generative model supplies a structured prior over the design space $\mathcal{X}$. It enables the search to generate candidates that reflect domain knowledge, semantic plausibility, and structural constraints without requiring $\mathcal{X}$ to be explicitly enumerated. By changing its context or inference procedure, the model can also expand and reshape the set of candidates reachable by the current search.

\paragraph{(ii) A probabilistic surrogate.} As introduced in Section~\ref{sec:ldm}, the posterior $p(R\mid\mathcal{D}_t)$
represents the current belief about the unknown reward function, conditioned on the empirical observations collected so far. For a candidate $x$, the predictive mean $\mu_t(x)$ estimates its expected reward, while the predictive uncertainty $\sigma_t(x)$ quantifies the surrogate's epistemic uncertainty about that estimate. High epistemic uncertainty indicates that the prediction is weakly supported by the available evidence and may be reduced through additional informative evaluations. It is therefore distinct from irreducible randomness or measurement noise in the evaluation process.

The surrogate provides the empirical grounding that the generative model alone generally lacks. Although a foundation model may encode substantial domain knowledge, its likelihoods and self-assessments do not necessarily provide calibrated estimates of a task-specific quantitative reward, particularly for novel candidates outside the observed data distribution.

\paragraph{(iii) A model-based acquisition function.} From the surrogate's predictive distribution, we construct an acquisition function $\acq_t(x)$ 
that assigns each candidate a scalar value to the sequential search process. Depending on its form, the acquisition function may favour candidates with high predicted reward, candidates with high epistemic uncertainty, or a principled combination of the two. It therefore mediates the trade-off between exploitation and exploration.

A central novelty of LDM is that it guides search with an \emph{acquisition
value}, rather than a conventional reward-model score. Whereas a reward model
estimates the expected performance of a candidate, the acquisition function
measures the decision value of allocating additional computation or costly
empirical evaluation to it, accounting for both predicted reward and epistemic
uncertainty. This signal therefore governs not only which candidates are
selected for external evaluation, but also how inference-time computation is
allocated across sampling, ranking, refinement, and branch expansion. Our
ablation study in the late experiment section demonstrates the
importance of this acquisition-guided mechanism relative to search based on
predicted reward alone.

These three components play complementary roles. The generative model proposes structured candidates and can move the search frontier towards designs not previously considered. The surrogate grounds the search in empirical observations and quantifies epistemic uncertainty. The acquisition function uses this predictive belief to determine which reachable candidates are most valuable to develop or evaluate. LDM thereby uses empirically calibrated feedback to direct the inference-time search of the generative model.


We seek a search distribution that attains high expected acquisition value while remaining sufficiently close to the structured prior supplied by the generative model. Assume that $\mathcal{X}$ is a measurable space, with its $\sigma$-algebra omitted for simplicity. At iteration $t$, we define
\begin{equation}
\label{eq:variational}
    \pi_t
    =
    \argmax_{q\in\Delta(\mathcal{X})}
    \left\{
        \mathbb{E}_{x\sim q}\!\left[\acq_t(x)\right]
        -
        \frac{1}{\eta}
        \KL\!\left(
            q
            \,\middle\|\,
            p_{\theta,\alpha}(\cdot\mid\mathcal{C}_t)
        \right)
    \right\},
\end{equation}
where $\Delta(\mathcal{X})$ denotes the set of probability measures over $\mathcal{X}$, $\KL$ is the Kullback--Leibler divergence, and $\eta>0$ controls the strength of acquisition-based tilting. The expected-acquisition term moves probability mass towards candidates that are valuable under the empirically grounded surrogate. The KL term prevents the search policy from departing arbitrarily far from the generative model's prior over plausible and well-formed designs. The parameter $\eta$ controls this trade-off: larger values place greater emphasis on acquisition value, whereas smaller values keep the search closer to the generative prior.

The inference configuration $\alpha$ affects $\pi_t$ through $p_{\theta,\alpha}(\cdot\mid\mathcal{C}_t)$ and therefore changes the reference distribution against which the KL divergence is measured. Increasing inference-time computation may reshape this proposal distribution through additional sampling, reasoning, refinement, or structured search. The acquisition function then determines how that computation is allocated according to predicted empirical value and epistemic uncertainty.

Equation~\eqref{eq:variational} is deliberately model-agnostic: it does not depend on a particular generative architecture, probabilistic surrogate, or acquisition function. Its variational form follows the Gibbs variational principle and is closely related to Gibbs posteriors~\citep{donsker1975variational,bissiri2016general}, KL-regularised policy search~\citep{peters2010relative}, maximum-entropy reinforcement learning~\citep{haarnoja2018soft}, and control as probabilistic inference~\citep{levine2018reinforcement}. Accordingly, our contribution is not the isolated generic KL-regularised objective. The distinctive role of Equation~\eqref{eq:variational} in LDM is to connect a structured generative prior to an acquisition value derived from a continually updated empirical surrogate. This connection allows external observations to direct both inference-time computation and sequential evaluation over structured, potentially open-ended design spaces. As we shall show later, a crucial part of our method is to have the acquisition value updated in a non-parametric fashion with the Gaussian Process to make it sample-efficient on-policy learning \citep{ramos2023bayesian}.

The resulting policy is recurrent (see Fig~\ref{fig:loop}). Candidates generated under the current search context are scored and refined using the surrogate-derived acquisition signal; selected candidates are then evaluated by an external empirical source; and the resulting observations update $\mathcal{D}_t$, the surrogate, and the next search context $\mathcal{C}_{t+1}$. LDM therefore does not merely reweight a fixed set of generated samples. It repeatedly couples generation, empirical evaluation, belief updating, and search-frontier expansion.

In the following sections, we first show that Equation~\eqref{eq:variational} admits a closed-form acquisition-tilted policy. We then describe how this ideal policy can be approximated through inference-time sampling, selection, and iterative refinement, with particular attention to test-time compute scaling. We subsequently instantiate the probabilistic surrogate as a Gaussian process and construct $\acq_t$ from its posterior predictive distribution. Finally, we describe how the same acquisition signal can support parameter adaptation: in addition to directing inference through $\alpha$, it may provide supervision for updating the generative parameters $\theta$ through fine-tuning.

\subsection{An Acquisition-Tilted Search Distribution}
\label{sec:core}

With the LLM prior $p_{\theta,\alpha}$ and the acquisition function $\acq_t$ defined, our desired search policy $\pi_t$ follows directly from the variational objective in Eq.~\eqref{eq:variational}. This variational problem is a canonical instance of the Gibbs variational principle (equivalently, the Donsker–Varadhan representation of the KL divergence)~\citep{donsker1975variational,bissiri2016general}, whose unique maximiser is the exponentially tilted distribution. This form of KL-regularised reward maximisation is the same mathematical foundation underlying maximum-entropy reinforcement learning, control-as-inference \citep{levine2018reinforcement}, and the closed-form policy objective in KL-regularised RLHF \citep{rafailov2023direct}.

\begin{proposition}[Optimal search distribution]
\label{prop:gibbs}
The unique maximiser of \eqref{eq:variational} is
\begin{equation}
\label{eq:core}
\pi_t(x)=\frac{1}{Z_t}\,p_{\theta,\alpha}(x\mid\mathcal{C}_t)\,\exp\!\big\{\eta\,\acq_t(x)\big\},
\qquad Z_t=\int_{x\in\mathcal{X}}p_{\theta,\alpha}(x\mid\mathcal{C}_t)\,\exp\!\big\{\eta\,\acq_t(x)\big\}\,\mathrm{d}x.
\end{equation}
\end{proposition}
\begin{proof}
Please refer to Appendix~\ref{prop_proof}.
\end{proof}

Eq.~\eqref{eq:core} shows that the optimal policy $\pi_t$ takes an acquisition-tilted form, where the base measure is reweighted by the exponent of the acquisition function. Interestingly, this policy draws an analogy to the infinite-armed bandit framework \citep{berry1997bandit,wang2008infinitely}: as scientific designs (arms) cannot be listed in advance, they are revealed from a reservoir distribution, and the policy tilts that reservoir toward arms it judges valuable. In this reading, $p_{\theta,\alpha}(\cdot\mid\mathcal{C}_t)$ plays the role of the reservoir, $\eta\,\acq_t(x)$ acts as the strength-of-reward tilt, and $Z_t$ is the usual log-sum-exp normaliser.

 A clear interpretation of the inference configuration $\alpha$ and the KL-regularisation coefficient $\eta$ can be drawn through standard analysis of infinite-armed bandits. We restrict attention to the common case in which the reservoir admits the temperature form
\begin{equation}
  \label{eq:temp}
  p_{\theta,\alpha}(x\mid\mathcal{C}_t) \propto p_\theta(x\mid\mathcal{C}_t)^\alpha,
\end{equation}
where $p_\theta$ denotes the unconditional base proposal. This is the exponential family of reservoirs commonly used in the bandit literature, and it makes $\alpha$ transparent as the strength of the prior: large values of $\alpha$ concentrate probability mass on modes of $p_\theta$, while $\alpha\to 0$ spreads mass toward a uniform measure on $\mathcal{X}$. Substituting this form into \eqref{eq:core} and taking the mode gives the bandit-style argmax
\begin{equation}
  \label{eq:argmax}
  x_{t+1}=\argmax_{x\in\mathcal{X}}\ \Big[\ \alpha\,\log p_\theta(x\mid\mathcal{C}_t)+\eta\,\acq_t(x)\ \Big],
\end{equation}
which picks the arm that maximises a weighted combination of the prior log-mass $\log p_\theta$ and the reward tilt $\acq_t$, i.e., the bandit analogue of selecting the most promising revealed arm.

Examining the two limiting cases of the temperature form $p_{\theta,\alpha}(x\mid\mathcal{C}_t)\propto p_\theta(x\mid\mathcal{C}_t)^\alpha$ situates LDM relative to related paradigms. As $\eta\to 0$, the tilt vanishes and $\pi_t$ collapses to the raw reservoir $p_{\theta,\alpha}(\cdot\mid\mathcal{C}_t)$, so search is driven entirely by the generative model with no data feedback. As $\alpha\to0$ the reservoir spreads toward a uniform measure on $\mathcal{X}$ and the policy becomes $\pi_t\propto\exp\{\eta\,\acq_t(\cdot)\}$, recovering classical Bayesian optimisation once $\eta$ is also taken large. LDM is the regime between these two limit points, where both the structured prior and the uncertainty-aware acquisition contribute to the policy. We note, however, that some LLM inference backends do not admit the temperature form $p_\theta(x\mid\mathcal{C}_t)^\alpha$, and the preceding analysis should be read as the canonical case under which the role of $\alpha$ is most cleanly interpreted.

With this construction, LDM defines a discovery policy $\pi_t$ from which new designs may be drawn using either a stochastic or a deterministic approach, as illustrated by Eq.~\eqref{eq:argmax}. In practice, the normalising constant $Z_t$ is never computed explicitly, and we instead rely on inference-time search to approximate draws from the tilted distribution.
\begin{figure}[t]
  \centering
\includegraphics[width=\textwidth]{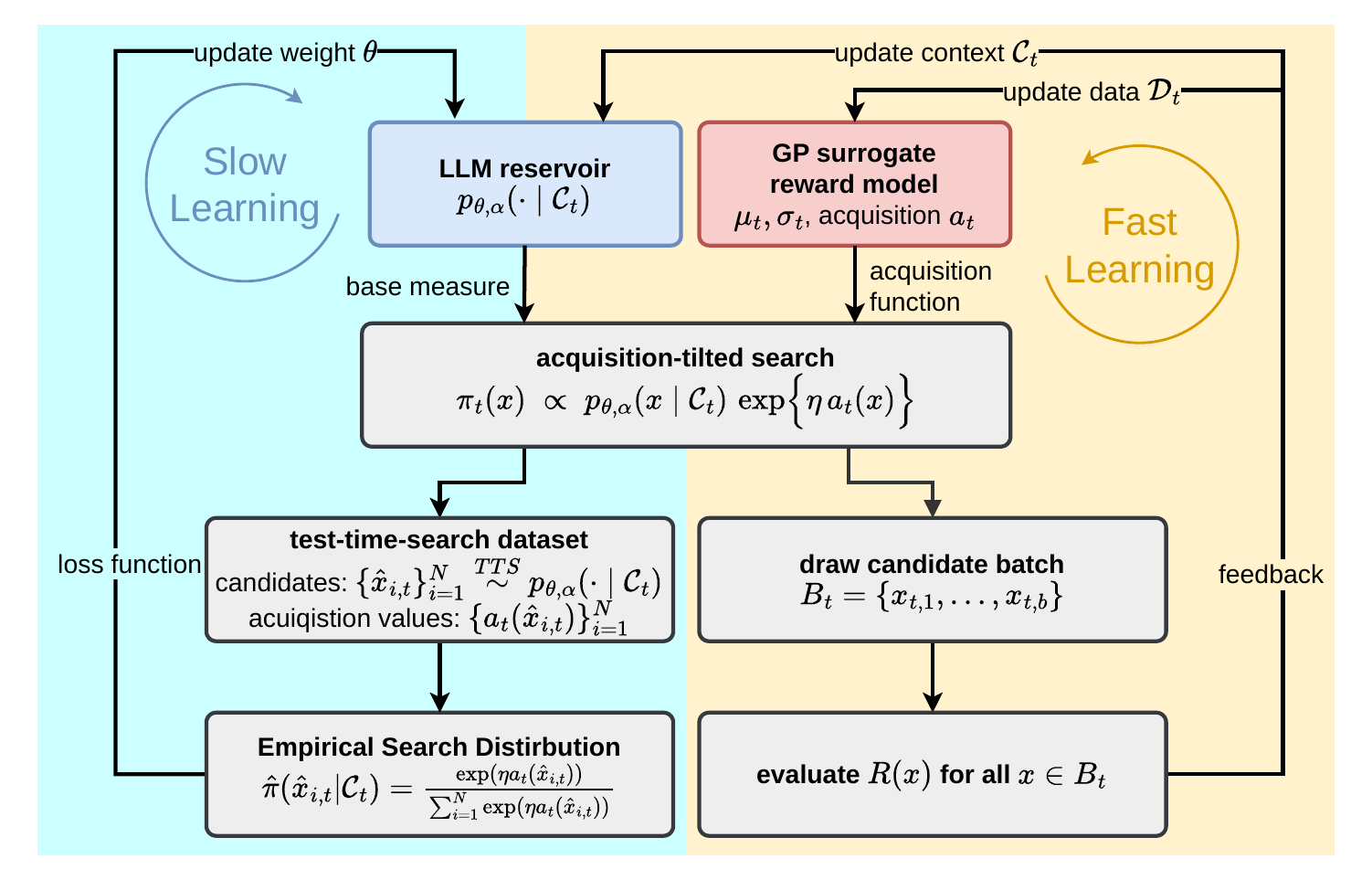}
    \caption{The Large Discovery Model learning loop. The right (cyan) region denotes the \emph{fast learning loop} of the LDM's recurrent discovery: a candidate batch $B_t$ is drawn and evaluated, and the observations update the surrogate data $\D_t$ and the LLM context $\C_t$. The left (yellow) side corresponds to the \emph{slow learning loop}, which amortises the acquisition-tilted search policy into the LLM's weights $\theta$ via LDM-TTS fine-tuning (see~\S\ref{sec:ldm-finetune}). All candidates $\{\hat{x}_{i,t}\}_{i=1}^{N}$ generated by the LLM during test-time search, alongside their respective acquisition values, are recorded to a TTS dataset. Using this dataset, we fit the empirical search distribution into the network, endowing the foundation model with generalisable discovery expertise to enable stronger discovery performance on subsequent tasks. The fast loop iterates with each batch of reward evaluation within a specific discovery task, while the slow loop is executed only after sufficient new samples spanning diverse tasks have been collected for the TTS dataset.}
  \label{fig:loop}
\end{figure}

\subsection{Surrogate Modelling of the Reward Function}
\label{sec:surrogate}

To efficiently navigate the design space, we require a probabilistic model over function space that treats the reward function $R$ as a random function and models the posterior belief $R \mid \mathcal{D}_t$ conditioned on our cumulative evaluations. Our LDM formulation is largely inspired by the uncertainty-aware decision-making paradigm of Bayesian optimisation \citep{garnettBayesianOptimization2023}.

We model this objective using a Gaussian Process (GP) \citep{rasmussen2006gaussian}. GPs serve as highly effective surrogates in this setting because they are exceptionally sample-efficient, admit exact Bayesian updates, and provide principled non-parametric uncertainty estimates across unexplored regions of the design space. Formally, we define the prior distribution over the reward function as:
\begin{equation}
\label{eq:gp_prior_main}
R(x) \sim \mathcal{GP}\bigl(m(x),\, k(x, x')\bigr),
\end{equation}
where $m: \mathcal{X} \to \mathbb{R}$ is the prior mean function and $k: \mathcal{X} \times \mathcal{X} \to \mathbb{R}$ is a positive-definite kernel function capturing spatial correlation.

By conditioning this prior on the historical data $\mathcal{D}_t = \{(x_i, r_i)\}_{i=1}^{t}$, we obtain a Gaussian posterior distribution characterised by a predictive mean $\mu_t(x)$ and a predictive variance $\sigma_t^2(x)$ for any candidate point $x \in \mathcal{X}$. These two statistical moments allow us to construct acquisition functions that elegantly balance exploitation (seeking high predicted rewards) and exploration (targeting regions of high uncertainty). For instance, the popular Upper Confidence Bound (UCB) acquisition function is formulated as:
\begin{equation}
\label{eq:ucb_main}
a^{\mathrm{UCB}}_t(x) = \mu_t(x) + \sqrt{\beta_t}\,\sigma_t(x),
\end{equation}
where $\beta_t > 0$ is a parameter scaling the exploration incentive. The complete technical details of the posterior inference, along with the formulation of other common acquisition functions, are provided in Appendix~\ref{sec:llm-bo}.

Because UCB decomposes the decision value into a high-mean and a high-uncertainty term, the tilted policy $\pi_t(x)\propto p_{\theta,\alpha}(x\mid\C_t)\exp\{\eta\,\acq_t(x)\}$ of Eq.~\eqref{eq:core} instantiates the discovery--exploration--exploitation triangle of Figure~\ref{fig:knowledge-regimes} directly: the LLM reservoir term realises discovery, the high-$\sigma_t$ tail of $\acq_t$ realises exploration, and the high-$\mu_t$ mode realises exploitation. We use UCB as the running example precisely because the trade-off is explicit in its closed form; in richer acquisitions the same triangle is realised only implicitly through the surrogate.

\subsection{Acquisition-Tilted Inference-Time Search}
\label{sec:inference-time-search}

Exact evaluation or direct sampling of $\pi_t$ is generally intractable because the base measure is induced by an LLM decoding process rather than given as a tractable density over a closed finite space. Moreover, the resulting design space is typically large and open-ended, and the normalising constant of the acquisition-tilted distribution cannot be computed exactly. We therefore approximate $\pi_t$ by inference-time search: the LLM first proposes a finite candidate pool, and the acquisition function then reweights or selects among these candidates. This turns test-time compute into search effort guided by the acquisition function, rather than by model-internal confidence or linguistic plausibility~\citep{wang2024openr}.

Concretely, at round $t$ we draw a finite pool of $N$ candidates from the LLM-induced proposal,
\[
\{\hat{x}_{t,i}\}_{i=1}^N \sim p_{\theta,\alpha}(\cdot \mid \mathcal{C}_t),
\]
where stochastic re-sampling or deterministic selection is then performed to produce the design $x_t$ to evaluate at the current step. The \emph{stochastic} approach approximates sampling from the acquisition-tilted distribution by softmax reweighting over this finite pool:
\begin{equation}
\label{eq:softmax_reweight}
\Pr(x_{t} = \hat{x}_{t,i})
=
\frac{\exp\bigl\{\eta \, \acq_t(\hat{x}_{t,i})\bigr\}}
{\sum_{j=1}^N \exp\bigl\{\eta \, \acq_t(\hat{x}_{t,j})\bigr\}}.
\end{equation}
When $\eta=0$, the selected candidate is sampled uniformly from the proposed pool, ignoring the acquisition scores. As $\eta$ increases, probability mass concentrates on candidates with larger acquisition values. In the limit $\eta \to \infty$, the stochastic rule degenerates to \emph{deterministic} best-of-$N$ selection:
\begin{equation}
x_{t+1}
=
\argmax_{i \in \{1,\dots,N\}}
\acq_t(\hat{x}_{t,i}).
\end{equation}

Candidate pool size $N$ plays a critical role in \emph{test-time compute scaling}, a powerful technique for LLM reasoning~\citep{wangTutorialLLMReasoning2025}, and LDM adapts this core advantage to discovery tasks. Larger $N$ improves the fidelity of the finite-pool approximation to the acquisition-tilted distribution and can improve empirical performance, at the cost of additional LLM inference and surrogate evaluation. This aligns with the established finding that test-time compute can substitute for model scale in LLM reasoning, with the key distinction that LDM uses the acquisition function as explicit search guidance rather than relying on the model's own self-evaluation or linguistic plausibility.

The same construction extends naturally to batched inference-time search that chooses multiple designs at the same time. The deterministic limit selects the top-$b$ candidates according to $\acq_t$. Meanwhile, the stochastic version samples a diverse acquisition-weighted batch, for example through Gumbel-top-$b$ sampling over the same softmax weights (see Appendix~\ref{sec:batch}).

The finite-pool reweighting scheme above is simple and generator-agnostic: it only requires the ability to sample candidate designs from the LLM-induced proposal and to evaluate their acquisition scores. It therefore applies uniformly across modalities (e.g., text, molecular strings, graphs, and programmes), without requiring a representation-specific search algorithm. This full-candidate reweighting scheme forms the basis of our experiments. For sequentially generated design spaces such as sequences and programs, the same principle can be extended to intermediate generation steps via a prefix value function:
\[
V(x_{t,\le i})
=
\mathbb{E}_{x_{t,>i}}
\bigl[
\acq_t(x_{t,\le i},x_{t,>i})
\bigr],
\]
enabling more advanced methods such as beam search and Monte Carlo Tree Search~\citep{fengAlphazerolikeTreeSearchCan2024}. These multi-step variants offer further search coverage at higher computational cost, and we leave their full empirical evaluation as future work.

\subsection{Consolidating acquisition-tilted inference via model fine-tuning}
\label{sec:ldm-finetune}


The preceding sections keep the model parameters $\theta$ fixed and realise the acquisition-tilted policy through inference-time computation. The procedure acts as a policy-improvement operator \emph{within} a discovery episode, and require large budget to repeat the same over-generation and scoring process at every search state. We now consider consider the complementary route of scaling computation \emph{across} episodes. Search states, candidate pools, and acquisition decisions accumulated by LDM-TTS become experience for slow-learning the proposal model. There two routes form a recurrent search--learn cycle as illustrated in Figure~\ref{fig:loop}. Search adapts decisions to the current empirical posterior; parameter learning consolidates recurring search improvements across states and tasks; and the learned proposer provides a stronger starting distribution for subsequent search.

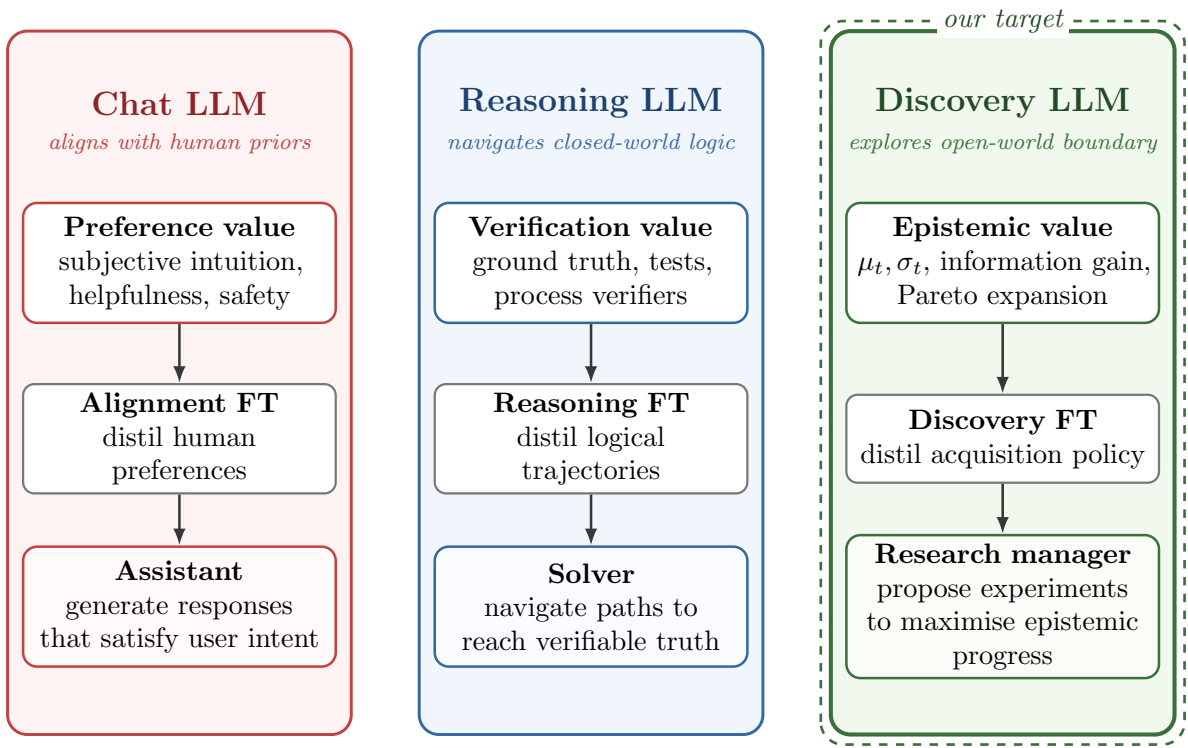
\begin{figure}[t]
\centering
\begingroup
\resizebox{0.98\textwidth}{!}{%
\begin{tikzpicture}[
    >={Latex[length=2.1mm]},
    font=\small,
    lane/.style n args={2}{draw=#1,line width=1.05pt,fill=#2,rounded corners=7pt,
        minimum width=43mm,minimum height=88mm},
    card/.style n args={2}{draw=#1,line width=0.9pt,fill=#2,rounded corners=5pt,
        align=center,inner sep=3pt,text width=37mm,minimum height=15mm},
    train/.style={draw=axisgray,line width=0.8pt,fill=white,rounded corners=4pt,
        align=center,inner sep=3pt,text width=37mm,minimum height=11mm},
    flow/.style={->,line width=0.9pt,draw=bodygray},
    title/.style={font=\bfseries\large,align=center},
    subtitle/.style={font=\scriptsize\itshape,align=center},
    note/.style={font=\scriptsize,align=center,text=bodygray},
]

\coordinate (C1) at (-5.15,0);
\coordinate (C2) at (0,0);
\coordinate (C3) at (5.15,0);

\node[lane={chatredline}{chatred!30}] (lane1) at (C1) {};
\node[lane={reasonblueline}{reasonblue!32}] (lane2) at (C2) {};
\node[lane={discgreenline}{discgreen!40},line width=1.45pt] (lane3) at (C3) {};

\node[title,text=chatdeep] at ($(C1)+(0,3.5)$) {Chat LLM};
\node[subtitle,text=chatredline] at ($(C1)+(0,3.0)$) {aligns with human priors};
\node[title,text=reasondeep] at ($(C2)+(0,3.5)$) {Reasoning LLM};
\node[subtitle,text=reasonblueline] at ($(C2)+(0,3.0)$) {navigates closed-world logic};
\node[title,text=discdeep] at ($(C3)+(0,3.5)$) {Discovery LLM};
\node[subtitle,text=discgreenline] at ($(C3)+(0,3.0)$) {explores open-world boundary}; 

\node[card={chatredline}{white}] (human) at ($(C1)+(0,1.5)$)
  {\textbf{Preference value}\\subjective intuition, helpfulness, safety};
\node[card={reasonblueline}{white}] (logic) at ($(C2)+(0,1.5)$)
  {\textbf{Verification value}\\ground truth, tests, process verifiers};
\node[card={discgreenline}{white}] (acq) at ($(C3)+(0,1.5)$)
  {\textbf{Epistemic value}\\$\mu_t,\sigma_t$, information gain,\\Pareto expansion};

\node[train] (alignft) at ($(C1)+(0,-0.7)$)
  {\textbf{Alignment FT}\\distil human preferences};
\node[train] (reasonft) at ($(C2)+(0,-0.7)$)
  {\textbf{Reasoning FT}\\distil logical trajectories};
\node[train] (discft) at ($(C3)+(0,-0.7)$)
  {\textbf{Discovery FT}\\distil acquisition policy};

\node[card={chatredline}{chatred!12}] (assistant) at ($(C1)+(0,-2.8)$)
  {\textbf{Assistant}\\generate responses that satisfy user intent};
\node[card={reasonblueline}{reasonblue!14}] (solver) at ($(C2)+(0,-2.8)$)
  {\textbf{Solver}\\navigate paths to reach verifiable truth};
\node[card={discgreenline}{discgreen!18}] (manager) at ($(C3)+(0,-2.8)$)
  {\textbf{Research manager}\\propose experiments\\to maximise epistemic\\progress};

\foreach \a/\b in {human/alignft,alignft/assistant,logic/reasonft,reasonft/solver,acq/discft,discft/manager}{
  \draw[flow] (\a) -- (\b);
}

\draw[discgreenline,line width=0.95pt,dashed,rounded corners=8pt]
  ($(lane3.south west)+(-0.10,-0.10)$) rectangle ($(lane3.north east)+(0.10,0.10)$);
\node[font=\small\itshape,text=discdeep,fill=white,inner sep=2pt]
  at ($(lane3.north)+(0,0.10)$) {our target};

\end{tikzpicture}%
}
\endgroup
\caption{\textbf{Fine-tuning as value distillation across model regimes.}
Chat LLMs distil human preference value to align with subjective intuition. Reasoning LLMs distil verification value to navigate deterministic proof paths. In contrast, Discovery LLMs distil epistemic acquisition value. Rather than merely memorising static, domain-specific solutions, discovery fine-tuning trains a research manager that learns an intrinsic acquisition strategy that prioritising high-information experiments to expand open-world boundaries and advance an evolving research trajectory.}
\label{fig:finetune-values}
\end{figure}

\paragraph{Acquisition-guided policy learning, not reward-guided}
The distinguishing feature of this update is the decision value used for supervision. As summarised in Figure~\ref{fig:finetune-values}, alignment learning commonly follows human preference value~\citep{ouyangTrainingLanguageModels2022,rafailov2023direct}, while reasoning-model training commonly follows trajectories leading to verifier-approved answers~\citep{wang2024openr}. Scientific models may also be trained directly on these measured properties, teaching them which completed molecules, proteins, or programs received high reward~\citep{liu2025drugimprovergpt,cao2026reinforcement}. These are useful but different targets compared to LDM. LDM learns from the \emph{state-dependent acquisition value of the next experimental action}: whether a candidate is worth developing or evaluating given the observations, constraints, and search frontier at the current round. 

Acquisition value is not identical to predicted or realised reward. In scientific discovery, the decision value of a move is inherently Bayesian: it balances predicted quality ($\mu_t$), epistemic uncertainty ($\sigma_t$), physical constraint satisfaction, and Pareto-frontier expansion. An experimental move can therefore be highly valuable even when it does not immediately maximize the posterior mean. This mirrors the counter-intuitive, high-information exploratory actions seen in breakthrough AI systems, such as AlphaZero's famous ``Move 37'' \citep{silver2017mastering,schut2025bridging}. A move is valuable because it resolves a critical unknown, opens a novel region of the search space, or alters the future trajectory of the research process. LDM fine-tuning aims to embed this hidden acquisition-value structure directly into the model's internal decision bias. By doing so, the model learns the strategic risk-taking required to manage an evolving research state, rather than merely imitating historical end solutions.


\paragraph{Acquisition-reweighted data collection.}
At search state $\C_t$, let
$\widehat{\mathcal X}_t=\{\hat{x}_{t,i}\}_{i=1}^{N}$ be the candidate pool
sampled from $p_{\theta,\alpha}(\cdot\mid\C_t)$ as in
Section~\ref{sec:inference-time-search}. The surrogate assigns an acquisition
score to every candidate, inducing the finite-pool sampling distribution
\begin{equation}
\label{eq:empirical-tilt}
\widehat{\pi}^{(N)}_t(\hat{x}_{t,i}\mid\C_t)
=
\frac{\exp\!\left\{\eta\,\acq_t(\hat{x}_{t,i})\right\}}
{\sum_{j=1}^{N}\exp\!\left\{\eta\,\acq_t(\hat{x}_{t,j})\right\}}.
\end{equation}
Because the pool itself is sampled from the LLM reservoir, this empirical
reweighting is a finite-sample approximation of the complete tilted policy in
Equation~\eqref{eq:core}; the structured generative prior is implicit in which
candidates enter the pool, while the exponential factor supplies the
empirically grounded policy improvement.

This construction yields more supervision than the single candidate that is
eventually sent for costly evaluation. Candidates generated under the same
history form counterfactual next actions, and their surrogate scores reveal how
their exploitation, exploration, feasibility, and frontier-expansion values
compare at that particular search state. We record these search states and
candidate traces as LDM-TTS search experience and retain high-value actions
according to $\widehat{\pi}^{(N)}_t$. This acquisition-reweighted collection
defines the training distribution. Rather than sampling traces explicitly, the
same distribution can be implemented by weighting each record in proportion
to its acquisition tilt when learning $p_\phi$. In this sense, parameter
learning amortises part of the computation performed by repeated search without
making the learned proposer identical to the online search policy.

\paragraph{Reasoning augmentation.}
Although an acquisition score provides a precise ranking, a scalar alone if less informative and does
not expose the reusable reason why an action advances the search~\citep{feng2024natural,song2026learning}. We therefore
optionally augment a retained action with a concise rationale $z_{t,i}$ that
interprets the decision in the current research state. The rationale identifies
relevant evidence in the history, diagnoses progress or stagnation, and states
whether the proposed move should exploit a promising family, explore an
uncertain one, preserve diversity, satisfy a constraint, or expand the active
support. In contrast to a reasoning trace for a closed problem, it explains
\emph{why the next experiment is worth running}, rather than how to derive a
known correct answer. At the meantime, converting acquisition scalar value to natural language enable the model/agent to allocate more tokens at test time and scale up this "discovery reasoning".


Let $\bar w_{t,i}$ denote normalised acquisition weights derived from
Equation~\eqref{eq:empirical-tilt}. We train the proposer with the weighted
autoregressive objective
\begin{equation}
\label{eq:sft-loss}
\mathcal L_{\mathrm{SFT}}(\phi)
=
-\sum_{t,i}\bar w_{t,i}
\left[
\lambda_z\log p_\phi(z_{t,i}\mid\C_t)
+
\log p_\phi(\hat{x}_{t,i}\mid\C_t,z_{t,i})
\right],
\end{equation}
where $\lambda_z=1$ gives reasoning-augmented policy learning. For action-only
policy learning, $z_{t,i}$ is omitted, $\lambda_z=0$, and the candidate is
conditioned directly on $\C_t$. If traces are sampled from
$\widehat{\pi}^{(N)}_t$ during dataset construction, the same objective may be
implemented with uniform example weights.


\paragraph{Complete LDM training pipeline} Figure~\ref{fig:app-training-pipeline} in Appendix provide an overview of the complete training workflow from test-time search to data construction and supervised fine-tuning. In our experiments, we use strong model to perform high-budget search as teacher, and smaller-size open-sourced model as student, and the intermediate data construction and post-process determine what and how the student learn from its teacher.

The two compute routes in Figure~\ref{fig:loop} therefore remain complementary. Within a task, the fast
search loop uses new observations to recalibrate value and redirect
computation. Across accumulated LDM-TTS trajectories, the slow learning loop
consolidates recurring acquisition-guided decisions into the proposal model.
The experiments in Section~\ref{sec:finetune-experiments} test whether this consolidation improves
the reservoir and whether the learned research-management policy transfers to
targets and tasks absent from the fine-tuning data.


\section{The Algorithm}
\label{sec:algo}

\begin{algorithm}[htbp]
\caption{The Large Discovery Model Loop}
\label{alg:ldm_sequential}
\begin{algorithmic}[1]
\Require Reward function $R:\mathcal{X}\to\mathbb{R}$; prior mean $m(\cdot)$; kernel $k(\cdot,\cdot)$; initial dataset $\mathcal{D}_1$; initial context $\mathcal{C}_1$; evaluation budget $T$; batch size $b$; candidate pool size $N$.
\For{$t = 1, \dots, T$}
  \State Update the active decision domain: $A_t \gets \operatorname{supp}\,p_{\theta,\alpha}(\cdot\mid\C_t)$
  \State Fit Gaussian-process surrogate over $A_t$ using data $\mathcal{D}_t$ to obtain the posterior $R \mid \mathcal{D}_t \sim \mathcal{GP}(\mu_t, k_t)$, given by Eqs.~\eqref{eq:postmean}, \eqref{eq:postvar} \Comment{Full derivation in Appendix~\ref{sec:llm-bo}.}
  \State Construct acquisition function $\acq_t$ by combining $\{\mu_t, \sigma_t\}$ \Comment{e.g. Eq.~\eqref{eq:ucb} or \eqref{eq:ei}.}
  \State Draw $B_t = \{x_{t,1},\dots,x_{t,b}\}$ from $\pi_t \propto p_{\theta,\alpha}(\cdot\mid\C_t)\exp\{\eta\,\acq_t\}$ using Algorithm~\ref{alg:its}
  \State Observe black-box noisy rewards $r_{t,i} = R(x_{t,i}) + \epsilon_{t,i}$ for all $x_{t,i} \in B_t$
  \State Update dataset: $\mathcal{D}_{t+1} \gets \mathcal{D}_t \cup \{(x_{t,i}, r_{t,i})\}_{i=1}^b$
  \State Update context: $\mathcal{C}_{t+1} \gets \mathcal{C}_t \cup \{(x_{t,i}, r_{t,i})\}_{i=1}^b \cup \{\text{Reflection feedback (if applicable)}\}$
\EndFor
\State \Return $\argmax_{(x,r) \in \mathcal{D}_{T+1}} r$
\end{algorithmic}
\end{algorithm}

\begin{algorithm}[htbp]
\caption{Acquisition-Guided Inference-time Sampling}
\label{alg:its}
\begin{algorithmic}[1]
\Require LLM prior $p_{\theta,\alpha}(\cdot \mid \mathcal{C}_t)$; acquisition function $\acq_t$; tilt parameter $\eta$; pool size $N$; batch size $b$; selection mode $\in \{\text{deterministic}, \text{stochastic}\}$.
\State Sample $N$ candidate designs from the LLM prior: $\{\hat{x}_{t,i}\}_{i=1}^N \sim p_{\theta, \alpha}(\cdot \mid \mathcal{C}_t)$
\State Compute acquisition scores $s_i \gets \acq_t(\hat{x}_{t,i})$ for $i=1,\dots,N$
\If{selection mode = deterministic}
  \State Let $i_1,\dots,i_b$ be the indices of the $b$ largest scores $s_i$, and set $B_t \gets \{\hat{x}_{t,i_k}\}_{k=1}^{b}$ \Comment{Top-$b$ by acquisition.}
\ElsIf{selection mode = stochastic}
  \State Compute exponentiated-acquisition weights $w_i \gets \exp\{\eta\, s_i\}$ for $i=1,\dots,N$
  \State Sample $B_t = \{\hat{x}_{t,i_k}\}_{k=1}^{b}$ from the weights $w_1,\dots,w_N$ via Gumbel-top-$b$ \Comment{See \S\ref{sec:batch}.}
\EndIf
\State \Return $B_t$
\end{algorithmic}
\end{algorithm}

In this section, we present the algorithm to apply an LDM in sequential experimental designs for scientific discovery. As illustrated in Figure~\ref{fig:loop}, the LDM drives an iterative loop that tightly coordinates a generative foundation model with a data-driven verification surrogate. Each round constructs the tilted search policy, draws a batch of promising candidates, evaluates them with real experiments, and updates both the numeric results and the generative context with new observations. If the surrogate identifies a stale regime (cf.\ \S\ref{sec:eval}), the LLM is invoked to append a \emph{reflection feedback} summary to $\mathcal{C}_t$, thereby ensuring that subsequent proposals remain aligned with the most recent mechanistic understanding. This operation is orthogonal to the tilted-search core. The complete iterative procedure is formalised in Algorithm~\ref{alg:ldm_sequential}, while the inner inference-time sampling loop is specified in Algorithm~\ref{alg:its}.

\subsection{Policy Construction and Realisation}

We construct the approximate tilted policy by reweighting an LLM-derived base measure with the acquisition function. Depending on the design space, the base measure $p_{\theta,\alpha}(\cdot \mid \mathcal{C}_t)$ can be instantiated in two ways.
\paragraph{Direct Generation.} The LLM directly generates complete candidate designs from its autoregressive decoding distribution. Invalid candidates can be removed by rejection sampling with a validity checker.
\paragraph{Indirect Parameterisation.} When direct generation of valid designs is inefficient, the LLM instead specifies a parameterised search region, such as its centre, radius, or local bounds. An external sampler then generates candidates within this region. This decouples LLM inference from candidate generation and allows the pool size $N$ to be scaled more efficiently.

Both paradigms integrate seamlessly with the inference-time sampling pipeline. Deterministic top-$b$ selection ranks the finite pool by acquisition. Gumbel-top-$b$ draws a Plackett--Luce weighted sample without replacement, avoiding duplicate selections in batched evaluation \citep{kool2019gumbeltopk}. The candidate pool size $N$ serves as the primary dial for test-time compute scaling, enabling a continuous trade-off between inference cost and search quality.

\subsection{Surrogate on the Dynamic Support}
\label{sec:dynamic-support}

Under indirect parameterisation, the LLM defines not only candidate proposals but also the active search domain. We write this domain as
\[
A_t := \operatorname{supp}\, p_{\theta,\alpha}(\cdot\mid\C_t)\subseteq\X.
\]
Here, $\X$ may denote the full space of experimental designs, while $A_t$ is a lower-dimensional parameterised subspace selected by the LLM at round $t$, such as a set of schedule variables, architectural choices, or local search bounds.

This restricted support simplifies both search and surrogate modelling. Rather than modelling the full space $\X$, the GP is defined over $A_t$ and can be fit using historical observations whose designs lie in the current support:
$\mathcal D_t|_{A_t}=
\{(x,r)\in\mathcal D_t : x\in A_t\}$,
where $r$ denotes the observed reward associated with design $x$. Since $A_t$ is explicitly parameterised, the GP kernel and mean function can be defined directly in this reduced space. As the LLM changes or expands the active search domain across rounds, the surrogate is updated accordingly on the new $A_t$.

\subsection{Practical Implementation}

The framework supports several extensions commonly needed in scientific discovery: (1) \emph{constrained and cost-aware acquisition}, where feasibility constraints and heterogeneous evaluation costs are incorporated through constrained expected improvement or cost-normalised acquisition functions; (2) \emph{decomposed acquisition feedback}, where surrogate outputs such as predictive mean, epistemic uncertainty, constraint slack, and estimated cost are provided separately to the LLM to support more targeted proposal updates; (3) \emph{multi-objective optimisation}, where the scalar acquisition function is replaced by Expected Hypervolume Improvement (EHVI) for vector-valued objectives; and (4) \emph{batch selection}, where, for parallel evaluations with $b>1$, candidates are sampled without replacement from the finite-pool tilted distribution using Gumbel-top-$b$ sampling (see \S\ref{sec:batch}). Technical details are provided in Appendix~\ref{sec:impl-details}.

\section{Theoretical Analyses}
\label{sec:theory-main}

We now present a theoretical analysis of LDM that explicitly captures its
distinctive interaction between an LLM-defined, dynamically evolving search
space and acquisition-guided Bayesian optimisation. Unlike standard GP-UCB
analyses, which assume optimisation over a fixed and fully accessible domain,
our analysis allows the set of reachable candidates to depend on the LLM
reservoir and the evolving discovery context. This yields a new regret
decomposition that separates three sources of error: the
\emph{discovery gap} caused by the reservoir not yet reaching high-quality
regions, the conventional surrogate-based optimisation error within the
currently reachable region, and the \emph{LDM sampling shortfall} incurred
because candidates are sampled from an acquisition-tilted proposal rather
than obtained by exact acquisition maximisation. The resulting bound therefore
makes explicit how both the support and the probability allocation of the LLM
proposal affect discovery performance. This section is self-contained, and
readers primarily interested in the empirical results may safely skip it.

Let $\X$ denote the design space, and let
$x^\star\in\argmax_{x\in\X}R(x)$ be a global maximiser of the unknown objective
function. Importantly, performance is measured with respect to the true
objective $R$, rather than the uncertainty-aware acquisition function used to
guide search. After $T$ evaluations, the simple regret is
\begin{equation}
\regret_T
=
R(x^\star)-\max_{1\leq t\leq T}R(x_t).
\end{equation}

At round \(t\), the inference configuration $\alpha$ induces the effective proposal distribution $p_{\theta,\alpha}(\cdot\mid\C_t)$, and its support (the LLM-controlled dynamic search domain of Section~\ref{sec:algo}) is
\[
A_t \;:=\; \operatorname{supp}\, p_{\theta,\alpha}(\cdot\mid\C_t) \;=\; \{x\in\X : p_{\theta,\alpha}(x\mid\C_t) > 0\}.
\]

The LDM then samples the next candidate \(x_t \sim \pi_t\), where
\begin{equation}
\label{eq:theory-pi}
\pi_t(x) =
\frac{\exp(\eta\,\acq_t(x))\,p_{\theta,\alpha}(x\mid\C_t)}
{\sum_{u \in A_t} \exp(\eta\,\acq_t(u))\,p_{\theta,\alpha}(u\mid\C_t)}.
\end{equation}
Here, $\acq_t(x) = \mu_t(x) + \sqrt{\beta_t}\sigma_t(x)$ is the UCB acquisition function (matching the LDM convention in Section~\ref{sec:method}).

To relate this performance criterion to the per-round behaviour of the algorithm, define the instantaneous regret at round \(t\) as $\regret_t = R(x^\star) - R(x_t)$. Here, \(t\) indexes evaluations under the true objective \(R\), rather than internal LLM interaction steps. For simplicity, the analysis focuses on sequential selection and does not model batch selection. Equivalently, the simple regret after \(T\) trials is $\regret_T = \min_{1 \le t \le T}\regret_t$. In the analysis below, we first study the cumulative behaviour of \(\{\regret_t\}_{t=1}^T\), from which a simple regret guarantee follows via the standard relation $\regret_T \le \frac{1}{T}\sum_{t=1}^T \regret_t$.

\noindent\textbf{Instantaneous Regret Decomposition as Discovery Gap and Optimisation Regret.}
For each round \(t\), define the best objective value attainable within the current reservoir support as $R_{t,A}^\star = \max_{x \in A_t} R(x)$. Then, the instantaneous regret admits the exact decomposition
\begin{align}
\regret_t
&= R(x^\star) - R(x_t) \\
&=
\underbrace{R(x^\star) - R_{t,A}^\star}_{\regret_t^{\rm disc}}
+
\underbrace{R_{t,A}^\star - R(x_t)}_{\regret_t^{\rm opt}} .
\end{align}
The first term, \(\regret_t^{\rm disc}\), is the \emph{discovery gap}: it measures the loss incurred because the current reservoir support \(A_t\) may not contain designs sufficiently close to the global optimum. This term is zero whenever \(x^\star \in A_t\). The second term, \(\regret_t^{\rm opt}\), is the \emph{optimisation regret} within the currently discovered region: it measures how far the sampled point \(x_t\) is from the best candidate available in \(A_t\).

This decomposition separates the two roles of the LDM. The reservoir distribution controls whether high-quality regions are discovered at all, while the acquisition tilt controls how effectively the algorithm selects promising candidates among those currently reachable.

Below, we will give the necessary assumptions and definitions to support the regret bound analysis. 
\begin{assumption}[UCB calibration]
\label{ass:ucb-calibration}
For a non-decreasing sequence $(\beta_t)_{t\ge1}$, with probability at least $1-\delta_{\rm gp}$, simultaneously
for all $t\le T$ and $x\in\X$,
\[
|R(x)-\mu_t(x)|
\le
\sqrt{\beta_t}\,\sigma_t(x).
\]
This event is implied by the standard kernel bandit conditions: $R(x)=m(x)+g(\phi(x))$ with $g\in\mathcal H_k$, bounded RKHS norm, bounded kernel diagonal, and conditionally sub-Gaussian observation noise \citep{srinivas2010gaussian,chowdhury2017kernelized}.
\end{assumption}

\begin{assumption}[Reservoir coverage and local regularity]
\label{ass:coverage}
Let $\mathrm d_\X$ be a distance on the design space and define the reservoir
coverage radius $r_t^{\rm cov}=\inf_{x\in A_t}\mathrm d_\X(x,x^\star)$. The objective is one-sided H\"older around the optimum: there exist $L>0$ and $\chi\in(0,1]$ such that, for all
$x\in\X$,
\[
R(x^\star)-R(x)\le L\,\mathrm d_\X(x,x^\star)^\chi.
\]
\end{assumption}
This local H\"older condition is a standard regularity device in continuum- and X-armed bandits: it converts the reservoir coverage distance into an objective gap \citep{bubeck2011xarmed}. Related infinite-armed bandit analyses characterise a reservoir through the probability that a newly sampled arm is near-optimal \citep{wang2008infinitely}; this complementary viewpoint motivates treating the LLM proposal as a reservoir over structured designs.
Let $a_{t,A}^\star=\sup_{x\in A_t}\acq_t(x)$ be the maximum value of the acquisition function at $t$-th round. Then, for any tolerance $\zeta_t\ge0$, let $U_t(\zeta_t)=
\{x\in A_t:\acq_t(x)\ge a_{t,A}^\star-\zeta_t\}$ be the set of candidates in the current reservoir support $A_t$ whose acquisition value is within $\zeta_t$ of the best acquisition value. We denote the reservoir probability mass assigned to this near-best set by
\begin{equation}
\label{eq:coverage}
\kappa_t(\zeta_t)=p_{\theta,\alpha}\bigl(U_t(\zeta_t)\mid\C_t\bigr).
\end{equation}
The bound below depends on the size of $\kappa_t(\zeta_t)$: a small near-UCB mass means that the LLM reservoir makes high-acquisition designs hard to sample even after exponential tilting. As an observable reservoir-quality quantity, $\kappa_t$ records how much proposal mass is available near the best current acquisition value.

\begin{lemma}[Gibbs near-UCB sampling]
\label{lem:gibbs-near-ucb}
Conditioned on the history \(\C_t\), fix any \(\zeta_t\ge 0\) such that the near-UCB set has reservoir mass \(\kappa_t(\zeta_t)\). Then, for any \(\rho_t\in(0,1)\), we have
\begin{equation}
 \mathbb P\!\left(
a_{t,A}^\star-\acq_t(x_t)
\le
\zeta_t+\frac{1}{\eta}\log\frac{1}{\rho_t\kappa_t(\zeta_t)}
\,\middle|\, \C_t
\right)
\ge
1-\rho_t .
\end{equation}

\end{lemma}

Let
\[
\gamma_T=
\sup_{B\subset\X,\ |B|\le T}
\frac12\log\det(I+\lambda^{-1}K_B),
\qquad
C_\lambda=\frac{2}{\log(1+\lambda^{-1})},
\]
where $K_B$ is the kernel matrix on $B$.

\begin{theorem}[Average-regret bound of the large discovery model]
\label{thm:regret}
Suppose Assumptions~\ref{ass:ucb-calibration} and~\ref{ass:coverage} hold. Fix $\delta_{\rm gp},\delta_{\rm samp}\in(0,1)$ and choose
$\rho_t>0$ such that $\sum_{t=1}^T\rho_t\le\delta_{\rm samp}$. Then, for any $\zeta_t\ge0$ with
$\kappa_t(\zeta_t)>0$, with probability at least $1-\delta_{\rm gp}-\delta_{\rm samp}$,
\[
  \boxed{
  \frac{1}{T}\sum_{t=1}^T \regret_t
  \le
  \underbrace{\frac{L}{T}\sum_{t=1}^T(r_t^{\rm cov})^\chi}_{\text{discovery gap}}
  +
  \underbrace{2\sqrt{\frac{C_\lambda\beta_T\gamma_T}{T}}}_{\text{GP-UCB optimisation}}
  +
  \underbrace{
  \frac1T\sum_{t=1}^T
  \left[
  \zeta_t+
  \frac1\eta\log\frac1{\rho_t\kappa_t(\zeta_t)}
  \right]
  }_{\text{LDM sampling shortfall}} .
  }
\]
\end{theorem}

\begin{corollary}[Simple regret]
\label{cor:simple-regret}
Under the conditions of Theorem~\ref{thm:regret}, the same right-hand side also upper-bounds the simple regret $\regret_T=\min_{1\le t\le T}\regret_t$, because $\min_t\regret_t\le T^{-1}\sum_{t=1}^T\regret_t$.
\end{corollary}

The first term is the gap of not yet having generated candidates near the true optimum. The second is the usual GP-UCB information-gain term within the covered region. The third is the cost of sampling from the tilted LDM distribution rather than exactly maximising UCB over $A_t$; it decreases when the acquisition tilt $\eta$ is sharper or when the near-UCB reservoir mass $\kappa_t(\zeta_t)$ is larger.

The LLM affects the regret bound through the reservoir support \(A_t\) and its distribution \(p_{\theta,\alpha}(\cdot\mid\C_t)\). First, an informative LLM can propose candidates closer to the global optimum, thereby reducing the coverage radius \(r_t^{\rm cov}\) and hence the discovery-gap term. Second, it can assign substantial probability mass to candidates with high acquisition values, increasing \(\kappa_t(\zeta_t)\) and reducing the LDM sampling-shortfall term. Thus, the LLM provides a structured, history-dependent proposal distribution that can focus evaluations on semantically plausible and promising regions of the design space, rather than searching the entire space uniformly.

\section{Experiments}
\label{sec:eval}

A central premise of LDM is that the same general discovery engine can operate
across heterogeneous domains without reducing discovery to a fixed,
domain-specific search procedure. We therefore evaluate LDM on three case
studies spanning markedly different scientific objects, representations,
objectives, and evaluation processes: neural-network training programs,
antibody sequences, and small molecules. Together, these case studies test
whether LDM can combine a generative prior, an empirically calibrated value
model, and acquisition-guided search across both discrete and open-ended
hypothesis spaces.

All three case studies are evaluated through in-silico or digital-oracle benchmarks, which make it possible to run controlled sequential-search experiments. The LDM loop is agnostic to the evaluator and can use wet-lab measurements in place of these digital objectives; validating such wet-lab deployments is an important direction for future work.

Section~\ref{sec:cases} introduces the three discovery settings and their
experimental protocols: an autonomous research loop for improving
neural-network training programs \citep{karpathy2026autoresearch}; antibody
CDRH3 sequence design \citep{khan2022antbo}; and multi-objective
small-molecule discovery targeting KRAS G12D
\citep{kumarOverviewKRASG12D2026}. For each setting, we also present a
representative discovery trajectory to illustrate how LDM generates,
evaluates, and refines hypotheses through sequential feedback.
Section~\ref{sec:ablation-main} isolates the mechanisms responsible for this
behaviour, including the test-time search budget, the acquisition function,
and the distinction between acquisition-guided discovery and search without a
calibrated value model. Section~\ref{sec:performance} compares LDM with the
relevant domain-specific baselines. Finally,
Section~\ref{sec:finetune-experiments} investigates whether the discovery
experience accumulated through high-budget LDM test-time search can be
distilled into a reusable proposer, and evaluates its out-of-distribution
generalisation to previously unseen targets and tasks. Additional ablations
and detailed analyses are provided in the appendix.

\subsection{Overview of the Case Studies}
\label{sec:cases}

The three case studies are selected to expose complementary challenges in
general-purpose discovery and to represent different epistemic regimes
introduced in Section~\ref{sec:ldm}. They vary not only in the objects being
designed, but also in how the hypothesis space is represented, how candidates
are modified, and how empirical feedback is obtained.

The \texttt{autoresearch} setting studies discovery through
\emph{multi-turn program editing}. The agent maintains a persistent research
state, repeatedly modifies a neural-network training program, and evaluates
each modification through an actual training run. Because program edits can
introduce new abstractions, mechanisms, and combinations of ideas, the agent
can reshape the effective hypothesis space as discovery proceeds.

Antibody CDRH3 design represents a large discrete combinatorial problem in the
\emph{known-unknown} regime. The sequence space and its basic representation
are specified in advance, but the relationship between a sequence and its
binding affinity is unknown and costly to evaluate. This setting therefore
tests whether LDM can efficiently explore a well-defined but extremely large
design space under limited empirical feedback.

Small-molecule discovery represents the more open-ended
\emph{unknown-unknown} regime. The space of valid and synthetically
plausible
molecular structures cannot be reduced to a fixed candidate catalogue, while
desirable compounds must simultaneously satisfy multiple and potentially
competing objectives. This setting tests whether LDM can propose new
scaffolds, expand the reachable molecular space, and direct evaluations toward
promising regions that were not specified in advance.

Taken together, these settings progress from executable programs, through
biological sequences, to chemical structures. They allow us to examine
whether LDM functions as a common discovery engine across domains while
adapting its generation and empirical evaluation mechanisms to the structure
of each problem. For each case study, we first motivate the discovery setting,
then specify the experimental protocol, and finally examine a representative
trajectory of sequential generation, evaluation, and refinement.
\subsubsection{Autonomous research loop}
\label{sec:autoresearch}

We use a popular public \texttt{autoresearch} project \citep{karpathy2026autoresearch} as a testbed for acquisition-guided LLM search. In \texttt{autoresearch}, a coding agent repeatedly modifies a single file, \texttt{train.py}; each candidate program is evaluated by training a small language model for a fixed five-minute budget and measuring validation bits per byte (\texttt{val\_bpb}, lower is better). The original system provides only the LLM proposal-and-evaluation loop, to which we add the surrogate and acquisition-guided search. Under the LDM notation, a program is a design $x$ with the reward:
\[
R(x)=-\texttt{val\_bpb}(x),
\]
where the agent, conditioned on the current context, defines the proposal distribution $p_\theta(\cdot\mid\mathcal C_t)$, and the accumulated program--performance pairs form the dataset $\mathcal D_t$ used to fit the surrogate. This benchmark provides a natural setting for LDM: candidates are executable programs, each evaluation returns a quantitative non-differentiable objective, and the fixed training schedule enables controlled comparison across search methods. Crucially, it is not a fixed-dimensional hyperparameter problem: as the agent discovers new training-program mechanisms, the surrogate must acquire new coordinates on which to model them, exercising LDM's dynamic search-support mechanism (\S\ref{sec:dynamic-support}).

We instantiate LDM using indirect parameterisation. Rather than generating raw program text, the LLM identifies tunable degrees of freedom in \texttt{train.py}, such as learning rate, width, depth, and related hyperparameters, thereby defining a continuous search space. We fit a Gaussian process with an RBF kernel over this space, keeping its hyperparameters fixed for stability in the small-data regime. Expected Improvement (Eq.~\eqref{eq:ei}) is used as the acquisition function.

Each round, the LLM proposes a pool of candidate configurations conditioned on $\mathcal C_t$, and the surrogate selects the next configuration to evaluate by maximising EI, with a diversity term for batched rounds. When progress stalls, reflection allows the LLM to revise the parameterisation by adding or reshaping search dimensions. The primary comparison holds the coding model, environment, and five-minute-per-run budget fixed, contrasting the original Karpathy LLM-only loop against the full LDM.

\begin{figure}[ht]
    \centering
    \includegraphics[width=0.95\linewidth]{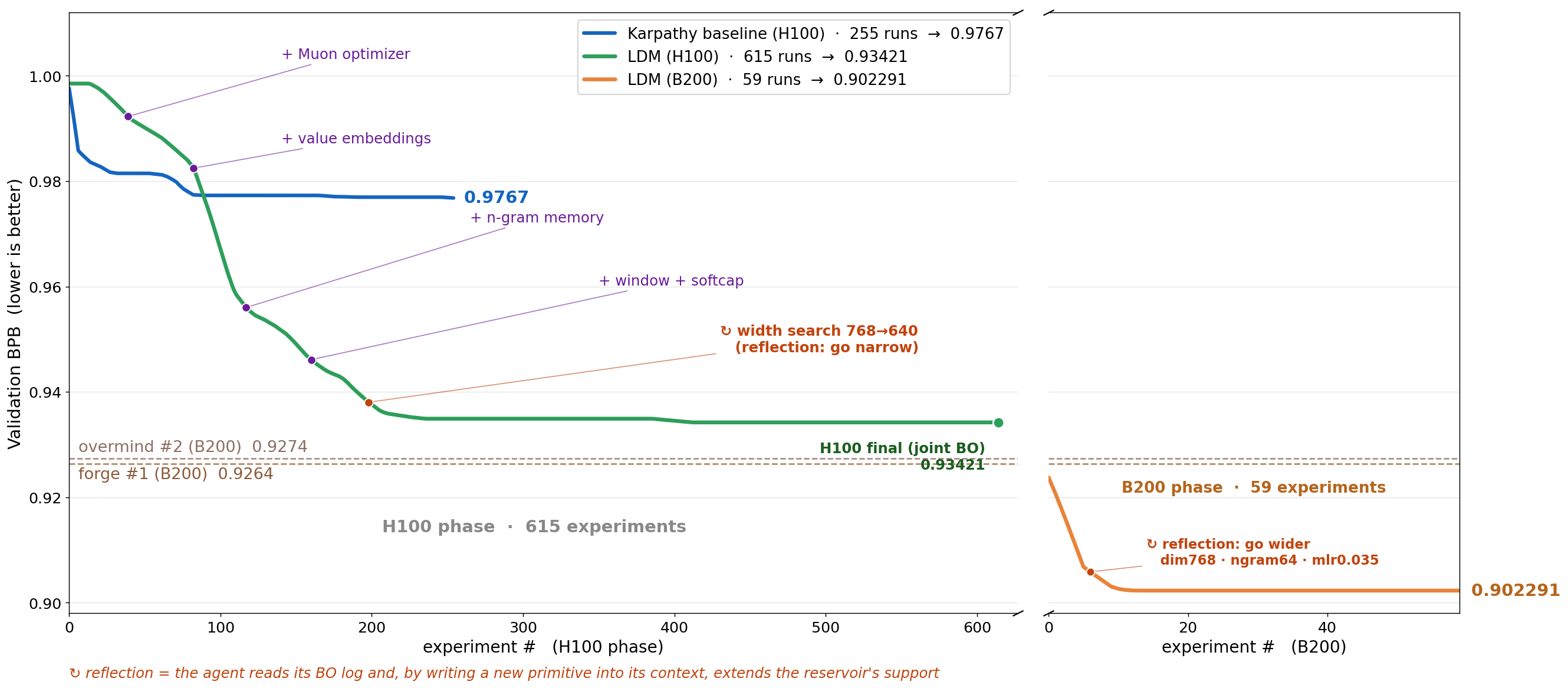}
    \caption{\textbf{The discover loop on \texttt{autoresearch}.} Validation \texttt{val\_bpb} (lower is better) versus search progress. The H100 panel compares the no-discover Karpathy baseline with the LDM; the B200 panel reports a matched-hardware leaderboard run of the LDM. B200 width scaled for legibility.}
    \label{fig:autoresearch}
\end{figure}

Figure~\ref{fig:autoresearch} shows an initial period of improvement followed by a plateau and a second improvement phase after the search space is revised: the LDM detects that local tuning of the current parameterisation has stalled, and the reflection mechanism expands the active search support to reach a new, more promising regime. This plateau-and-revision pattern is the discovery operation of \S\ref{sec:ldm} in action, and it recurs across the three domains.

\begin{table}[ht]
\centering
\footnotesize
\begin{tabular}{@{}llr@{}}
\toprule
Change & Why it helps under a fixed time budget & $\Delta$\texttt{val\_bpb} \\
\midrule
$n$-gram hash memory             & memorises local $n$-gram patterns at zero FLOP     & $-0.0275$ \\
QK-norm                          & steadier steps, more learning per step             & $-0.0137$ \\
Narrower model (dim 768$\to$640) & cheaper steps $\Rightarrow$ more steps in 300\,s   & $-0.0109$ \\
Windowed attention + softcap     & cheaper attention $\Rightarrow$ more tokens/sec    & $-0.0108$ \\
Squared-ReLU                     & richer per-token MLP representation                & $-0.0087$ \\
Value embeddings                 & extra per-token signal at near-zero cost           & $-0.0079$ \\
Muon optimiser                  & orthogonalised gradients, more learning per token  & $-0.0048$ \\
\bottomrule
\end{tabular}
\caption{\textbf{The \texttt{autoresearch} run as physics-grounded discovery.} Each row lists a change the LDM retained under a fixed 300\,s budget, the physical mechanism by which it helps, and the resulting change in \texttt{val\_bpb}.}
    \label{tab:autoresearch-discovery}
\end{table}

Table~\ref{tab:autoresearch-discovery} summarises the retained changes and their measured contributions. Under the fixed training-time budget, improvements must come from processing more tokens or learning more effectively from each token; each retained change is a physics-grounded discovery about how to make a language model learn faster inside a fixed wall-clock budget. Component-level results that isolate the contribution of each design choice are reported in Section~\ref{sec:ablation-main} and Appendix~\ref{app:ablation}.

\subsubsection{Antibody CDRH3 design}
\label{sec:cs-antibody-results}

We select antibody CDRH3 loop design as our second case study because it represents a canonical high-dimensional combinatorial optimisation problem with a well-defined design space. The CDRH3 loop is the primary determinant of antibody binding specificity, with sequences of length \(L \approx 11\) drawn from 20 amino acids, yielding a discrete space of size \(20^{11}\). This setting falls into the \emph{known unknown} regime of \S\ref{sec:ldm}: the syntax and boundaries of the design space are fully specified, but the sequence-to-affinity mapping is unknown, highly rugged, and expensive to evaluate. Traditional combinatorial Bayesian optimisation methods struggle with the intractable size of this discrete space, while pure LLM generation lacks calibrated value guidance. The task therefore tests whether LDM can combine structured proposal generation with acquisition-guided selection in a large discrete search space.

We use the \textsc{Absolut!} structure-based simulator as a black-box oracle that returns the lattice binding energy \(E_{\mathrm{bind}}\) \citep{khan2022antbo}. We maximise \(R(x) = -E_{\mathrm{bind}}(x)\) because lower binding energy implies stronger affinity, and each evaluation is a single expensive oracle call.

Experiments run sequentially, with one antibody candidate selected per iteration. We evaluate across five PDB targets, including 1ADQ\_A, 1FBI\_X, 1H0D\_C, 1NSN\_S, and 1OB1\_C, with a total budget of 200 evaluations per target, reporting mean performance and one standard deviation over multiple random seeds.

We implement four LDM variants formed by two orthogonal design choices: two LLM base generation strategies, paired with two acquisition weighting rules. Here \textbf{Policy} denotes the indirect parameterisation of \S\ref{sec:algo}: the LLM parameterises a sampling policy (with internal complexity), and the policy implicitly induces the base measure $p_{\theta,\alpha}(\cdot\mid\C_t)$. The \textbf{Direct} base measure generates sequences token-by-token with a small sample pool of \(m=5\). The \textbf{Policy} base measure lets the LLM output structured search parameters to spawn a much larger candidate pool. Following \S\ref{sec:inference-time-search}, inference-time search includes both deterministic and stochastic variants: \textbf{Max} picks the highest-acquisition candidate, while \textbf{Softmax} samples proportionally to exponentiated acquisition scores to approximate the full tilted distribution. Table~\ref{tab:ldm-variants-antibody} summarises all combinations.

\begin{table}[htbp]
    \centering
    \begin{tabular}{ccc}
        \hline
        LDM Variant & Base Measure \(p_\theta(x \mid \mathcal{C}_t)\) & Reduction Rule \\
        \hline
        Direct\_Max & Direct token-level sequence generation & Max \\
        Direct\_Softmax & Direct token-level sequence generation & Softmax \\
        Policy\_Max & Structured search policy from LLM output & Max \\
        Policy\_Softmax & Structured search policy from LLM output & Softmax \\
        \hline
    \end{tabular}
    \caption{Implemented LDM algorithm variants for computational antibody design.}
    \label{tab:ldm-variants-antibody}
\end{table}

\begin{figure}[!ht]
\centering
\includegraphics[width=0.95\linewidth]{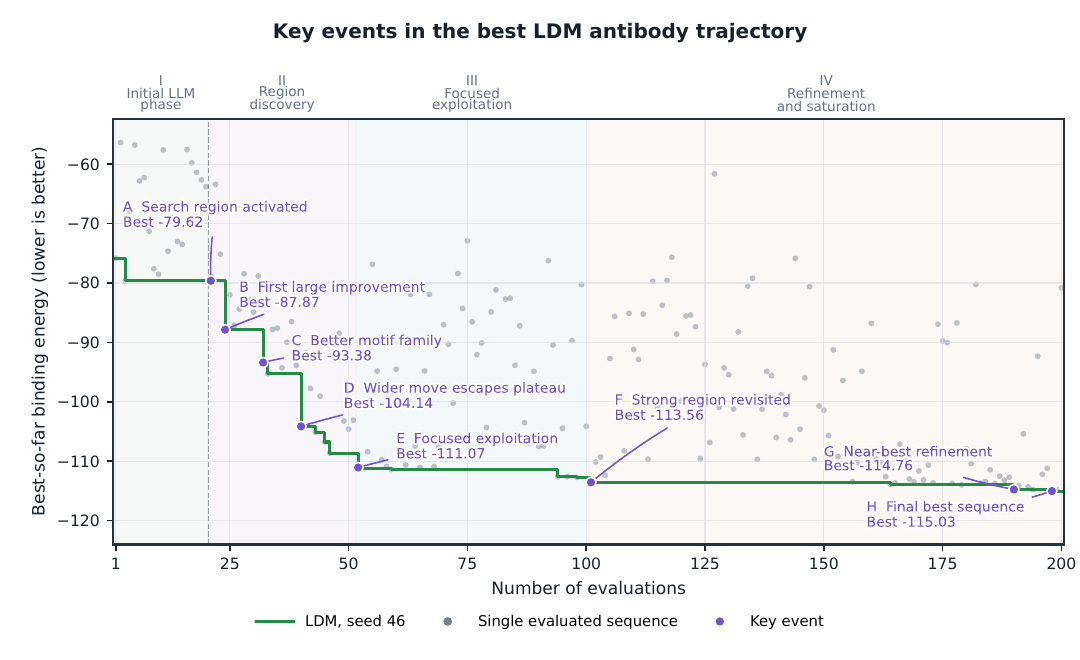}
\caption{Optimisation trajectory of the top-performing LDM run on antigen 1FBI\_X. The horizontal axis counts evaluation rounds; the vertical axis tracks the best binding energy found to date (lower = better). Step curve records global improvement milestones.}
\label{fig:antibody_trajectory}
\end{figure}

Figure~\ref{fig:antibody_trajectory} visualises a representative LDM optimisation trajectory on the 1FBI\_X antigen, illustrating the core exploit--explore--discover cycle described in \S\ref{sec:algo}. Initial LLM sampling quickly hits a local performance ceiling; the LDM detects stagnation and triggers discovery steps to expand the sequence families under consideration. Each numbered marker indicates a search-space update followed by further improvement and local refinement.

\subsubsection{Small-molecule drug discovery}
\label{sec:cs-molecule-results}

We select multi-objective small-molecule lead discovery as the third case study to test LDM's handling of the \emph{unknown unknown} regime of \S\ref{sec:ldm}. Unlike the antibody setting, the molecular design space is not given as a fixed finite set of candidates. Valid SMILES span a vast and open-ended chemical space, making exhaustive enumeration infeasible. Standard multi-objective Bayesian optimisation (MOBO) with string kernels lacks strong structural priors and saturates rapidly. Auxiliary expansion tools such as ReaSyn only generate analogues around fixed seed molecules, limiting exploration to narrow local chemical neighbourhoods. This setting tests whether LDM can discover new molecular scaffolds beyond the initial seed neighbourhoods.

Our multi-objective task optimises two conflicting rewards for the KRAS G12D target \citep{kumarOverviewKRASG12D2026}: the negated AutoDock Vina \citep{trott2010AutoDockVina} docking score (\(R_{\mathrm{Vina}}\)) and neural-network binding activity (\(R_{\mathrm{act}}\)), both maximised. We quantify performance via the hypervolume of the Pareto front, with a fixed reference point across all experiments and a total evaluation budget of 80 iterations.

Independent Gaussian Process surrogates are fitted for each reward dimension, using a subsequence string kernel to measure similarity between SMILES strings. Expected Hypervolume Improvement (EHVI) \citep{emmerich2011ehvi} is used as the multi-objective acquisition function. We evaluate two LDM variants that share the same tilted sampling pipeline but differ in candidate pool construction: (1) \textbf{Direct-Softmax}, where the LLM generates 128 SMILES candidates for acquisition-weighted resampling; and (2) \textbf{ReaSyn-Softmax}, where the LLM first proposes a small set of seed SMILES and ReaSyn generates local analogues to form the candidate pool before softmax selection.

\begin{figure}[!ht]
\centering
\includegraphics[width=1.0\linewidth]{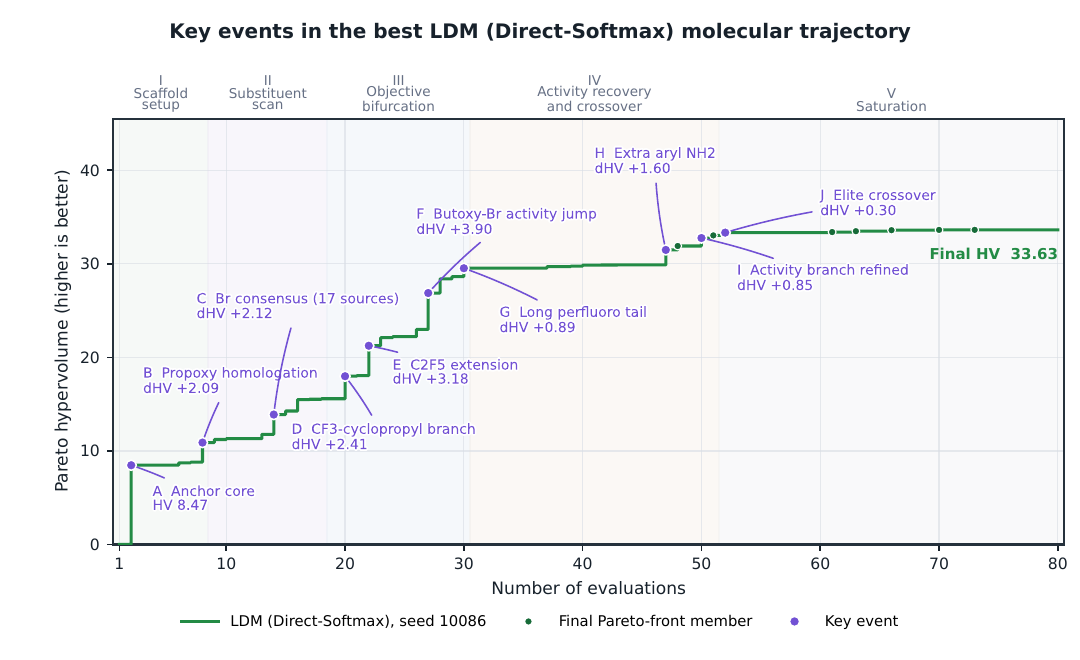}
\caption{Key milestone events along a Direct-Softmax molecular optimisation trajectory. X-axis counts evaluation steps; Y-axis plots cumulative Pareto hypervolume (higher = better). Step curve tracks global front improvements, with labelled markers denoting structural discovery events that open new families of active molecules. Shaded bands partition the optimisation into five sequential search phases. In each text box, \texttt{HV} represents Hyper-Volume, \texttt{Vina} represents the score of AutoDock Vina, and \texttt{NN} represents the neural network as the oracle activity prediction.}
\label{fig:direct_softmax_key_events}
\end{figure}

\begin{figure}[!ht]
\centering
\includegraphics[width=1.0\linewidth]{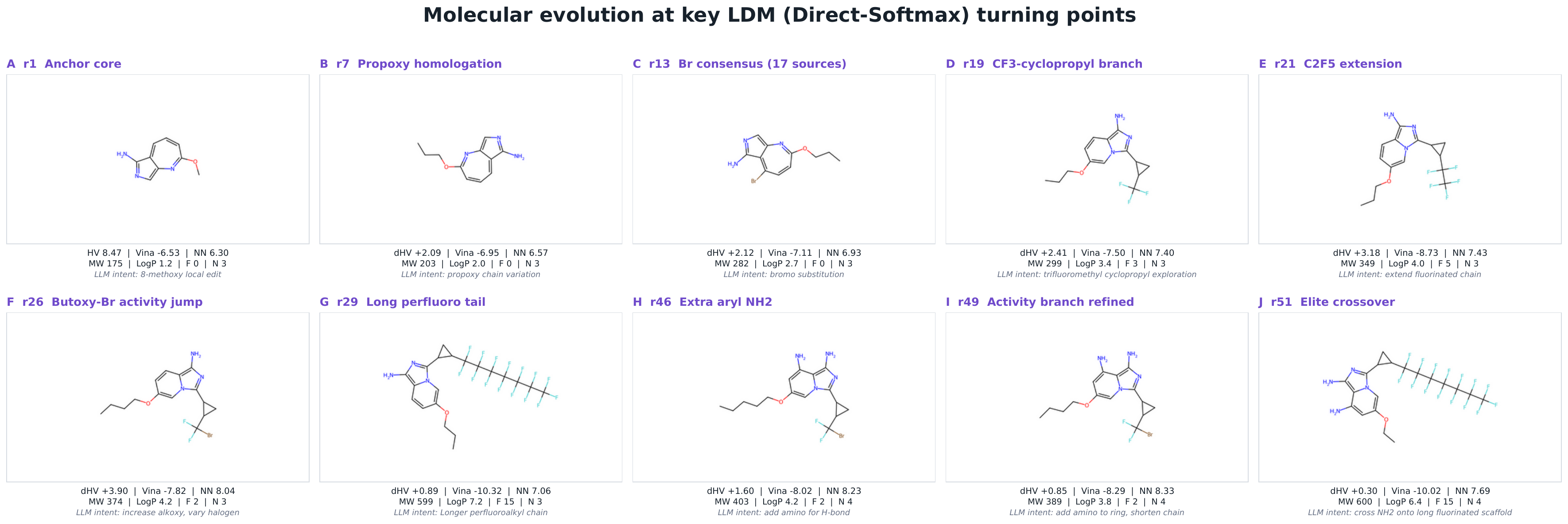}
\caption{Two-dimensional molecular structures corresponding to the ten core discovery milestones labelled in Figure~\ref{fig:direct_softmax_key_events}. Each structure is generated from canonical trajectory SMILES strings, arranged in chronological order of discovery to visualise the evolution of core chemical scaffolds and functional group substitutions.}
\label{fig:direct_softmax_molecular_evolution}
\end{figure}

Figures~\ref{fig:direct_softmax_key_events} and~\ref{fig:direct_softmax_molecular_evolution} together illustrate a complete LDM discovery workflow for small molecules. The hypervolume curve partitions optimisation into distinct stages: initial scaffold identification, substituent screening, objective bifurcation into docking- and activity-focused chemotypes, cross-family feature recombination, and final fine-tuning. The structural plot tracks how the model iteratively uncovers unrelated molecular backbones, a capability absent from standard BO or seed-local expansion tools.

\subsection{Ablation Studies}
\label{sec:ablation-main}

We next study how LDM depends on the test-time search budget and on its sampling and acquisition controls. Additional ablations, including the pure-LLM \texttt{autoresearch} loop, antibody hyperparameter sweeps, and single-step budget sweeps, are reported in Appendix~\ref{app:ablation}.

\subsubsection{Test-time search scaling and acquisition functions}

We ask three questions. First, does LDM exhibit the test-time scaling behaviour that motivates inference-time search in language-model reasoning \citep{snell2025scaling}, when the value is an expensive scientific objective approximated by a surrogate rather than a cheap verifier? Second, does \emph{discovery}, i.e., expanding the search frontier into previously unknown-unknown regions, contribute to the final performance? Third, does an acquisition function that explicitly balances exploitation and exploration (EI or UCB) outperform using only the posterior mean, which is purely exploitative?

We answer these questions in \texttt{autoresearch} as a representative setting. We keep the outer five-minute training budget fixed and vary only the cheap inner-loop budget of test-time search, which has two knobs: $N$, the number of candidate \texttt{train.py} programs the LLM proposes at each outer iteration, and $H$, the number of hypothesis/refinement branches each candidate is expanded into before the surrogate scores it. Together they control the number of acquisition-scored nodes per round, which we label \textsc{N$m$H$n$}: \textsc{N4H4} scores $4 \times 4 = 16$ nodes per round and \textsc{N8H8} scores $8 \times 8 = 64$. Because \texttt{autoresearch} is not a fixed-dimensional hyperparameter problem, LDM instantiates \emph{discovery} as a dynamic feature set: an extendable surrogate over a growing set of program features, so that newly discovered mechanisms enter the model as new coordinates (cf.\ the discovery operation of \S\ref{sec:ldm}). As an ablation, we remove the discovery mechanism, keeping the feature set fixed from the first iteration onward. We also compare acquisition functions---EI, UCB, and the posterior mean alone.

\begin{figure}[htbp]
    \centering
    \includegraphics[width=\linewidth]{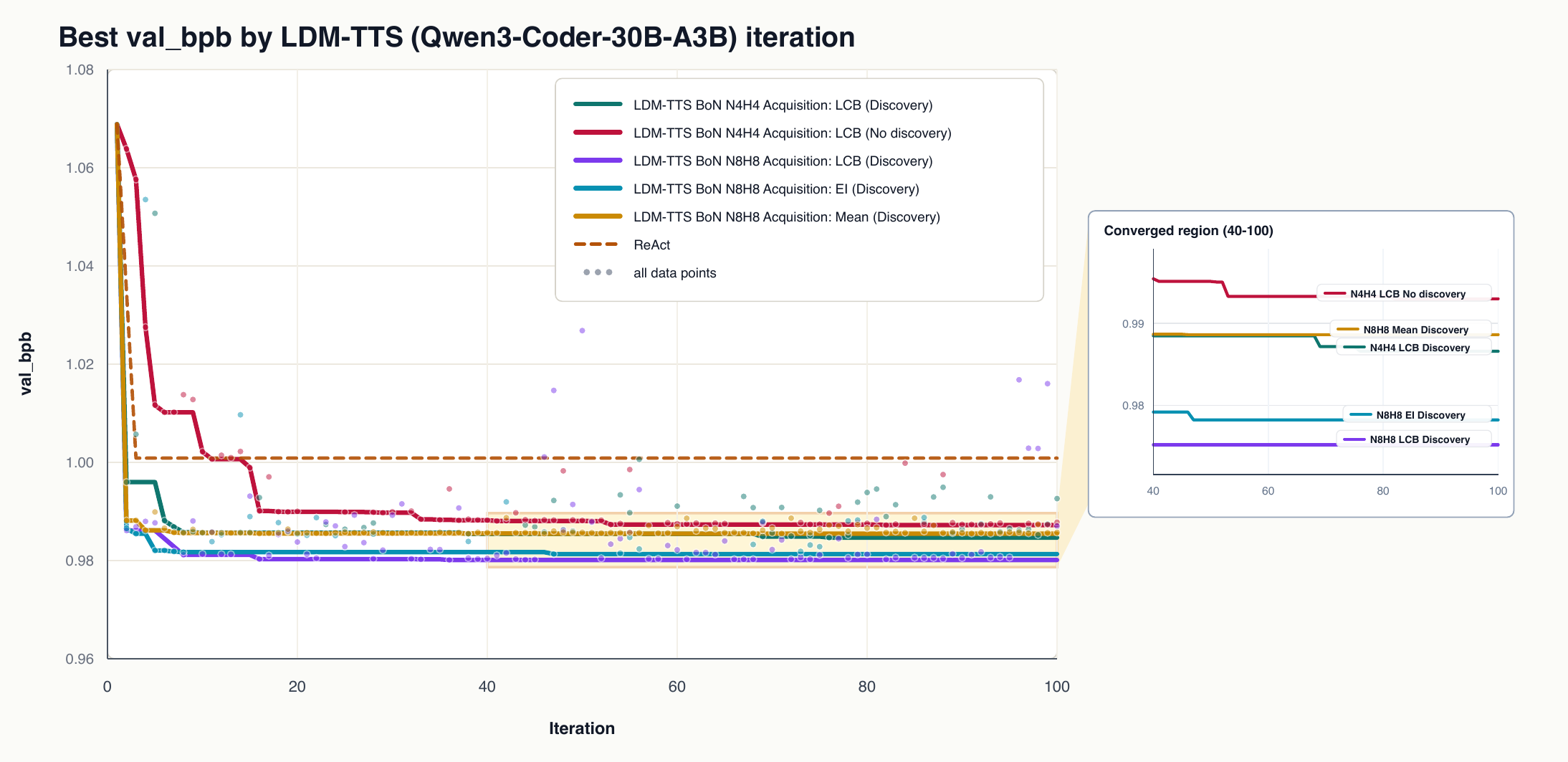}
    \caption{\textbf{Test-time search ablation on \texttt{autoresearch}.} Best validation \texttt{val\_bpb} found up to iteration $t$ (lower is better). Solid curves are LDM-TTS variants; the dashed curve is a ReAct-style no-acquisition reference \citep{yao2022react}; grey points are all evaluated candidate programs. Each LDM-TTS variant is named \textsc{N$m$H$n$} for its inner budget, where $N$ is the number of candidate programs the LLM proposes, and $H$ is the number of hypothesis/refinement branches each candidate is expanded into. }
    \label{fig:abl-autoresearch}
\end{figure}

Figure~\ref{fig:abl-autoresearch} supports three conclusions. First, \emph{increasing the test-time search budget helps}. The \textsc{N8H8} curve reaches the best plateau, while the lower-budget \textsc{N4H4} curves converge at a higher value; since the outer training budget is unchanged, the gain comes from spending more cheap computation before each expensive evaluation. Second, \emph{the discovery mechanism helps}. Without it, the search converges at a higher plateau, confirming that expanding the search frontier---not merely searching harder within a fixed support---is what sustains improvement after local plateaus. Third, \emph{the acquisition function matters}. Posterior-mean search is competitive early, when the most promising direction is already known, but it plateaus above both EI and UCB. EI improves on mean-only selection by rewarding improvement over the incumbent, while UCB performs best because it deliberately allocates part of the test-time budget to under-modelled program families. Overall, these uncertainty-aware acquisitions consistently outperform the exploitation-only posterior mean. These results validate the central claim of \S\ref{sec:method}: the LLM is the semantic search engine, but the calibrated, uncertainty-aware acquisition is what makes additional test-time compute useful.

Discovery is distinct from both the search budget and the acquisition function. The ReAct-style no-acquisition reference (dashed, Figure~\ref{fig:abl-autoresearch}) already shows that removing the acquisition value degrades search, and removing the discovery mechanism confines the search to a fixed support; neither reaches LDM's frontier. We push this further in Appendix~\ref{app:abl-pure-llm}: the pure LLM research loop---the $\eta=0$ limit of Eq.~\eqref{eq:core}, with no surrogate, no acquisition, and no uncertainty signal---plateaus near \texttt{val\_bpb}$\approx0.956$ after $875$ experiments, well above the LDM result of $0.93421$ (Figure~\ref{fig:autoresearch}). The bottleneck is not proposal expressivity, for the LLM can write useful training code and reflect on its own transcript; it is value. Without $\mu_t$, $\sigma_t$, and $\acq_t$, the run log remains a narrative memory rather than a searchable value landscape.

We run the same budget and acquisition sweeps for the molecular and antibody design tasks of \S\ref{sec:cs-molecule-results} and \S\ref{sec:cs-antibody-results}, with full results in Appendix~\ref{app:abl-molecule} and \ref{app:abl-antibody}. For molecular design, larger search budgets yield more effective discovery. The pattern is weaker for antibody design, where the scaling effect is small even though the optimum remains target-dependent across antigens (Appendix~\ref{app:abl-antibody}); a likely explanation is that the open-source LLM provides only weak sequence-level priors in this domain. Across acquisition functions, well-chosen acquisitions give consistent gains in both domains once the search budget is sufficiently large. Notably, the mean acquisition (a $0.5/0.5$ scalarisation of the Vina and activity objectives) outperforms EHVI under low search budgets for molecular design: EHVI scores candidates against the current Pareto front, so it needs a large screening pool to be informative, whereas scalarised acquisitions do not and remain effective when the pool is small.

\subsubsection{Hyperparameter sensitivity}

We next study the sensitivity of LDM to its sampling and acquisition temperatures on the KRAS G12D molecular design task of \S\ref{sec:cs-molecule-results}. The LLM sampling temperature $T_{\mathrm{LLM}}$ controls the diversity of LLM-generated SMILES candidates, while the acquisition temperature $\eta$ controls the strength of acquisition-guided selection. As $\eta\to0$ the policy collapses to the raw LLM reservoir and as $\eta\to\infty$ it reduces to acquisition-function argmax ($\eta=\infty$ in the experiments below).

\begin{figure*}[htbp]
    \centering
    \begin{subfigure}[t]{0.48\textwidth}
        \centering
        \includegraphics[width=\linewidth]{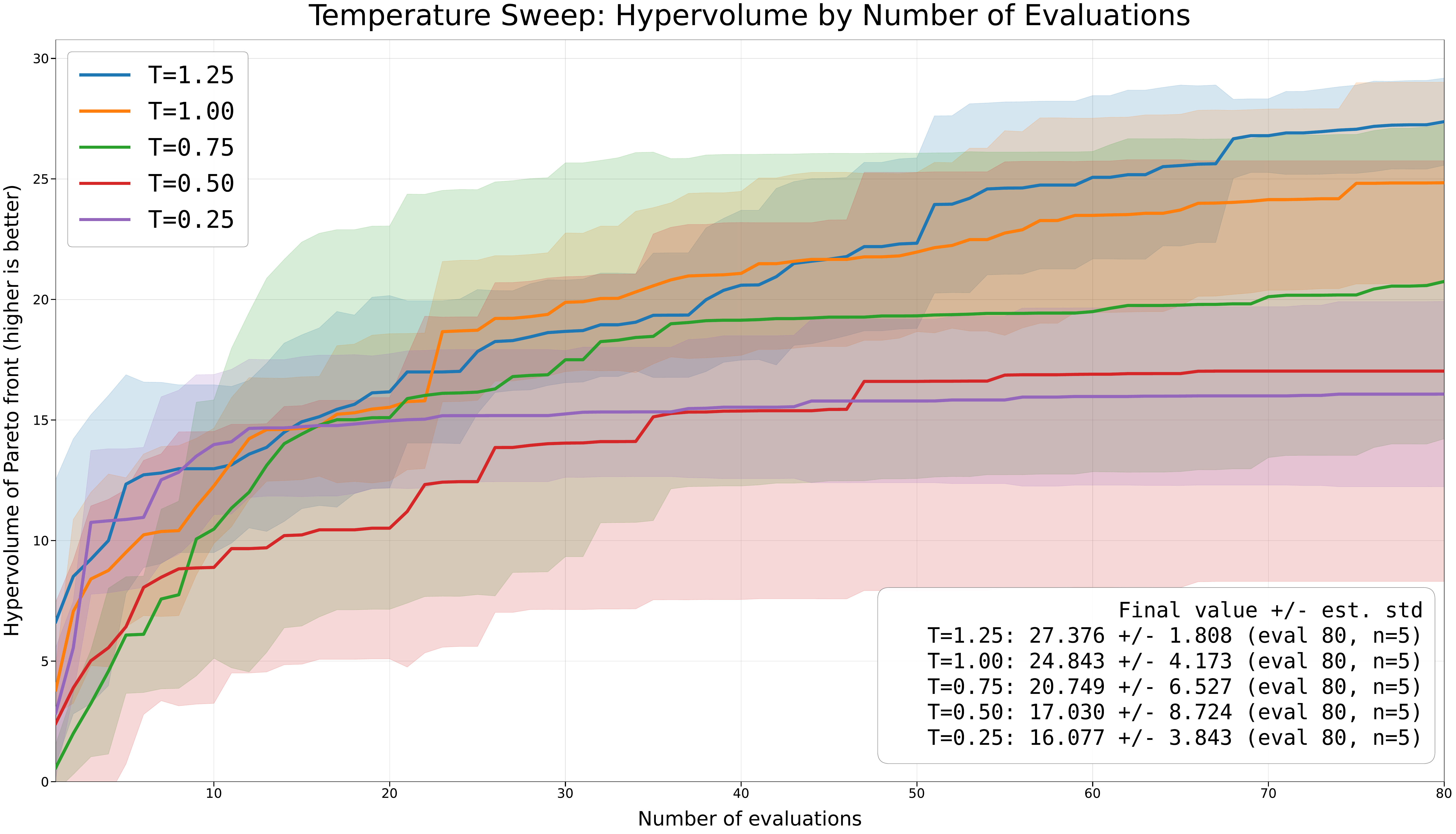}
        \caption{Effect of LLM sampling temperature $T_{\mathrm{LLM}}$ with fixed acquisition temperature $\eta=1$.}
        \label{fig:molecule-temperature-ablation}
    \end{subfigure}
    \hfill
    \begin{subfigure}[t]{0.48\textwidth}
        \centering
        \includegraphics[width=\linewidth]{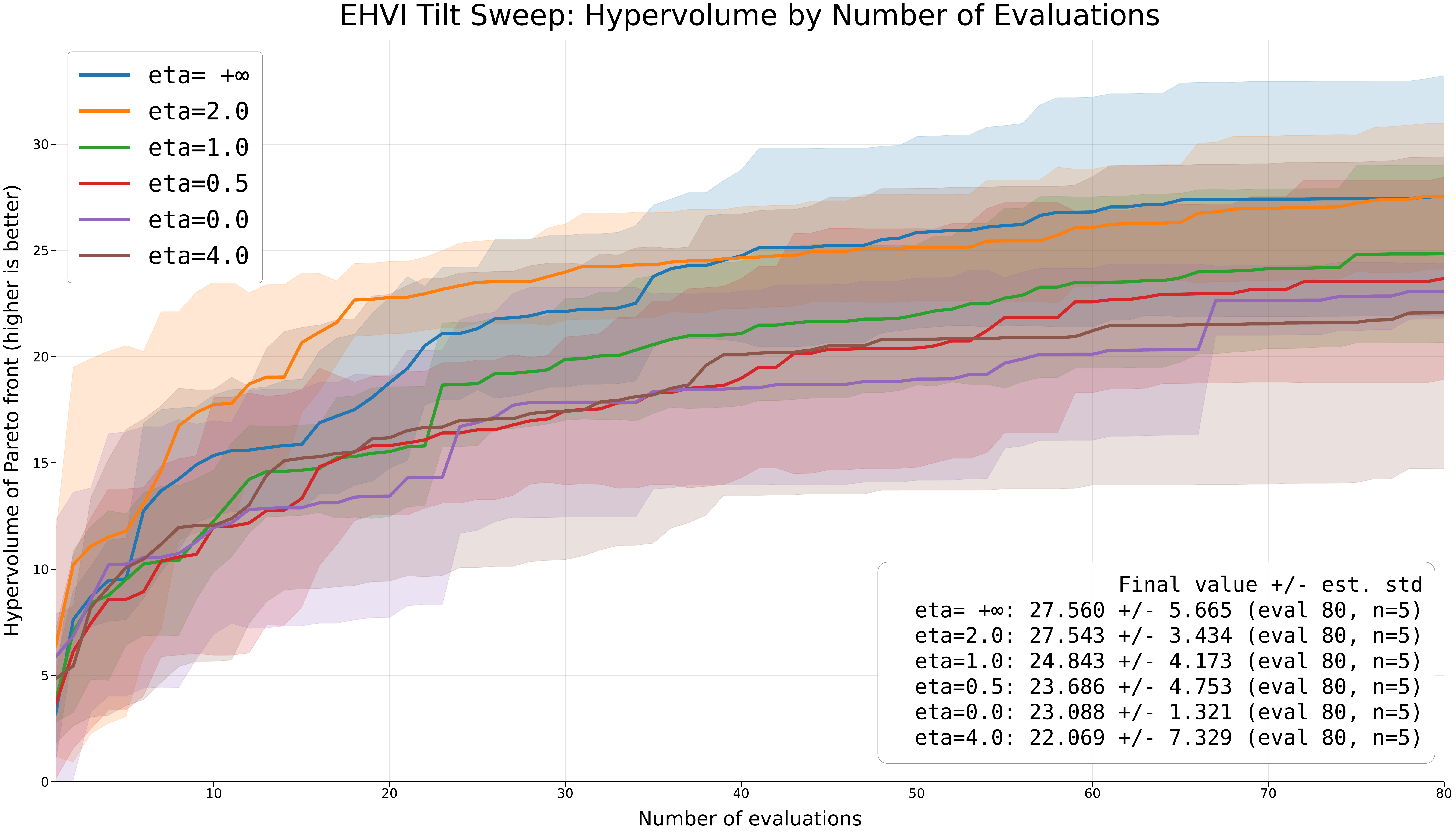}
        \caption{Effect of acquisition temperature $\eta$ with fixed LLM sampling temperature $T_{\mathrm{LLM}}=1$.}
        \label{fig:molecule-eta-ablation}
    \end{subfigure}

    \caption{\textbf{Hyperparameter ablations of LDM-TTS on molecular design.} Each curve reports the hypervolume of the Pareto front over the evaluation trajectory. Higher values indicate better multi-objective optimisation performance.}
    \label{fig:molecule-temperature-eta-ablation}
\end{figure*}

For the LLM sampling temperature ablation, we fix $\eta=1$ and vary the LLM-decoding temperature $T_{\mathrm{LLM}}$ in $\{0.25, 0.5, 0.75, 1, 1.25\}$. As shown in Figure~\ref{fig:molecule-temperature-ablation}, increasing $T_{\mathrm{LLM}}$ generally improves the final optimisation performance: the setting $T_{\mathrm{LLM}}=1.25$ achieves the highest final hypervolume, while lower temperatures degrade performance. For the acquisition temperature ablation, we fix $T_{\mathrm{LLM}}=1$ and vary $\eta\in\{0,0.5,1,\infty\}$. Figure~\ref{fig:molecule-eta-ablation} shows that increasing $\eta$ generally improves the final hypervolume. Overall, on molecular design, both sufficient LLM exploration and effective acquisition guidance are necessary for LDM, with the acquisition tilt $\eta$ acting as the dominant driver of final Pareto-front expansion.

Antibody CDRH3 design (\S\ref{sec:cs-antibody-results}) exhibits a different sensitivity pattern. The open-source LLM provides relatively weak sequence-level priors for this domain, so directly generated candidates are often uninformative in the large combinatorial search space. Policy-based LDM mitigates this limitation by using the LLM to parameterise a search region and then generating a larger candidate pool within that region. As a result, performance is less sensitive to additional test-time reweighting through the sampling and acquisition temperatures. This is consistent with the results in \S\ref{sec:cs-antibody-results}, where policy-mode LDM outperforms direct generation primarily through structured search-space parameterisation rather than strong sequence-level priors or aggressive acquisition weighting. Full temperature ablations for antibody design are provided in Appendix~\ref{app:abl-hyper-ldm}.

\subsection{Performance Comparison with Baselines}
\label{sec:performance}

We now compare LDM against the relevant baselines in each domain, focusing on the final optimisation curves.

\subsubsection{AutoResearch}

On \texttt{autoresearch}, the example trajectory in Figure~\ref{fig:autoresearch} (Section~\ref{sec:autoresearch}) already contrasts the LDM against the original Karpathy LLM-only loop under an identical H100 five-minute-per-run schedule: the no-discover baseline plateaus at \texttt{val\_bpb}$\approx0.9767$ after roughly $255$ runs, while LDM continues to $0.93421$ . Relative to the shared $\texttt{val\_bpb}=1.0069$ starting point, LDM reduces validation BPB by $0.0727$, a $2.4\times$ larger reduction than the LLM-only baseline's $0.0301$. Separately, a matched-hardware leaderboard run on B200 reaches $\texttt{val\_bpb}=0.902291$ in $59$ runs, compared with \textsc{forge} ($0.9264$) and \textsc{overmind} ($0.9274$) on the \texttt{autoresearch@home} leaderboard.\footnote{\url{https://www.ensue-network.ai/lab/autoresearch}} These improvements must come from processing more tokens or learning more effectively from each token under the fixed training-time budget; Table~\ref{tab:autoresearch-discovery} attributes them to the retained program changes.

\subsubsection{Antibody CDRH3 design}

The baseline methods considered in this study can be grouped into four categories: \textbf{Pure LLM} approaches, which perform direct autoregressive sequence generation; combinatorial Bayesian optimisation (BO) frameworks (\textbf{TURBO}, \textbf{COMBO}) that have been adapted to operate over discrete sequence spaces; the domain-specialised optimiser \textbf{AntBO}; and a naive random search strategy. Our goal is to position LDM as a competitive, general-purpose optimisation methodology, rather than to claim superiority over the domain-specialised AntBO approach.

\begin{figure}[ht!]
\centering
\includegraphics[width=1.0\linewidth]{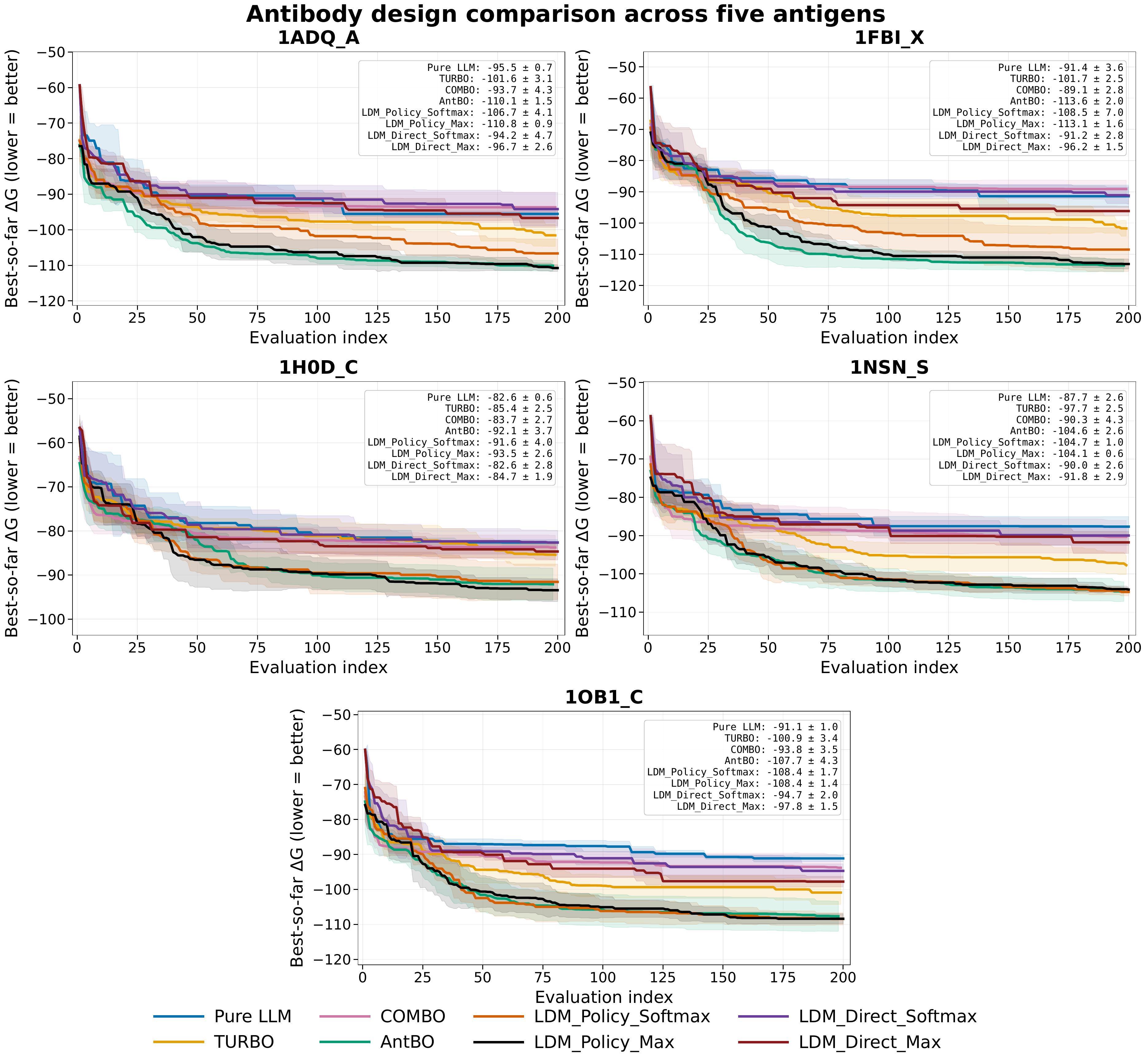}
\caption{Binding affinity performance across five antibody antigens. Axes: horizontal = number of sequence evaluations; vertical = best-so-far binding energy score (lower values indicate stronger binding). Curves correspond to all tested baselines and LDM variants.}
\label{fig:cs-antibody-results}
\end{figure}

Figure~\ref{fig:cs-antibody-results} compares all methods across the five protein targets. Policy-based LDM variants consistently outperform direct generation variants, achieving affinity levels comparable to AntBO, while Direct variants perform only marginally better than standalone LLMs.

Protein fitness landscapes are sparse and rugged, restricting the utility of small batches of directly generated sequences; even post-hoc acquisition reweighting cannot recover promising candidates the LLM fails to propose. In contrast, policy-mode LDM uses the LLM to define where to search rather than to emit final sequences, while the surrogate selects which candidates to evaluate; this expands the candidate pool without additional oracle calls.

\subsubsection{Small-molecule drug discovery}

\begin{figure}[!ht]
\centering
\includegraphics[width=0.6\linewidth]{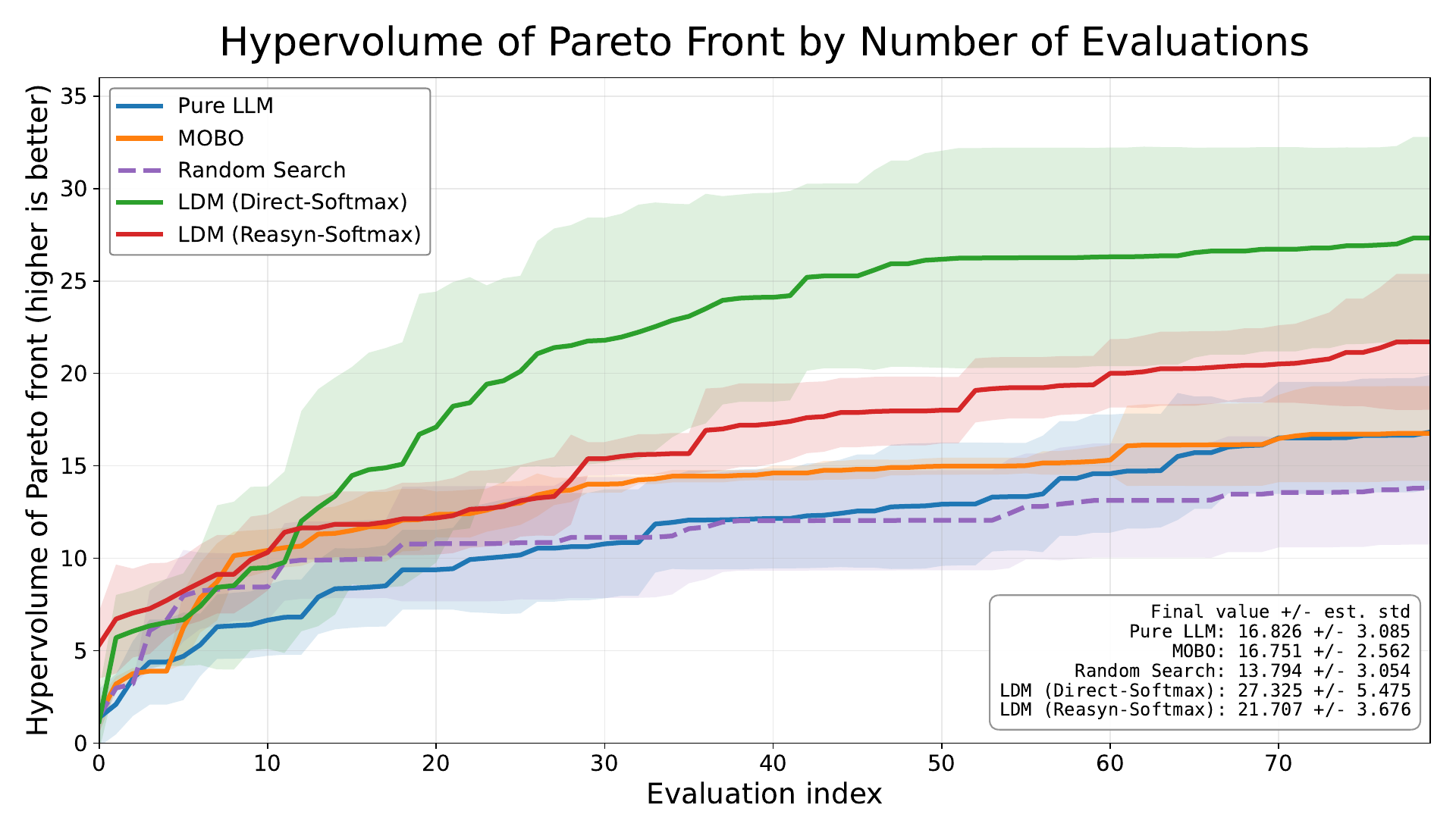}
\caption{Pareto hypervolume growth for molecular design. Horizontal axis = evaluation count; vertical axis = dominated hypervolume of the Pareto front (higher values denote better multi-objective trade-offs). Separate curves track performance of random search, standard MOBO, and the two LDM variants.}
\label{fig:cs-molecule-results}
\end{figure}

Baselines include random search and vanilla multi-objective Bayesian optimisation (MOBO), which also uses independent GPs and an EHVI acquisition function. Since the SMILES space is unbounded, for both baselines, we maintain a dynamic candidate pool. Specifically, starting from a set of seeding SMILES, we iteratively employ ReaSyn on historical SMILES (uniformly chosen in random search, and the historical best in MOBO), and add the generated analogues to the current pool.

Figure~\ref{fig:cs-molecule-results} shows that both LDM variants outperform all baselines by a substantial margin, with Direct-Softmax delivering the strongest hypervolume gains. Conventional MOBO plateaus after roughly 40 evaluations, trapped within limited chemical regions, while direct LLM generation sustains steady front expansion across the full budget.

The performance gap between the two LDM variants arises from a breadth--synthesizability trade-off. General-purpose LLMs carry broad pre-training knowledge of diverse SMILES structures; large-batch direct generation unlocks wider scaffold exploration. ReaSyn's local analogue generation improves synthetic tractability but restricts structural diversity, slowing Pareto expansion for early-stage lead discovery where novel scaffolds are prioritised.


\subsection{Learning Discovery Experience with Fine-Tuning}
\label{sec:finetune-experiments}

We evaluate discovery fine-tuning across three structured design domains: multi-turn program search in AutoResearch, combinatorial antibody CDRH3 design, and multi-objective molecular discovery. In every experiment, the fine-tuned proposer is placed inside the same acquisition-guided LDM loop as the base proposer. Fine-tuning therefore changes the proposal reservoir, whereas the online surrogate continues to provide value and uncertainty estimates from the current experimental history.

We compare action-only discovery fine-tuning with reasoning-augmented fine-tuning. The latter trains the proposer to generate a surrogate-grounded research-progress rationale before emitting the candidate action.
\begin{figure*}[!ht]
    \centering
    \includegraphics[width=0.82\textwidth]{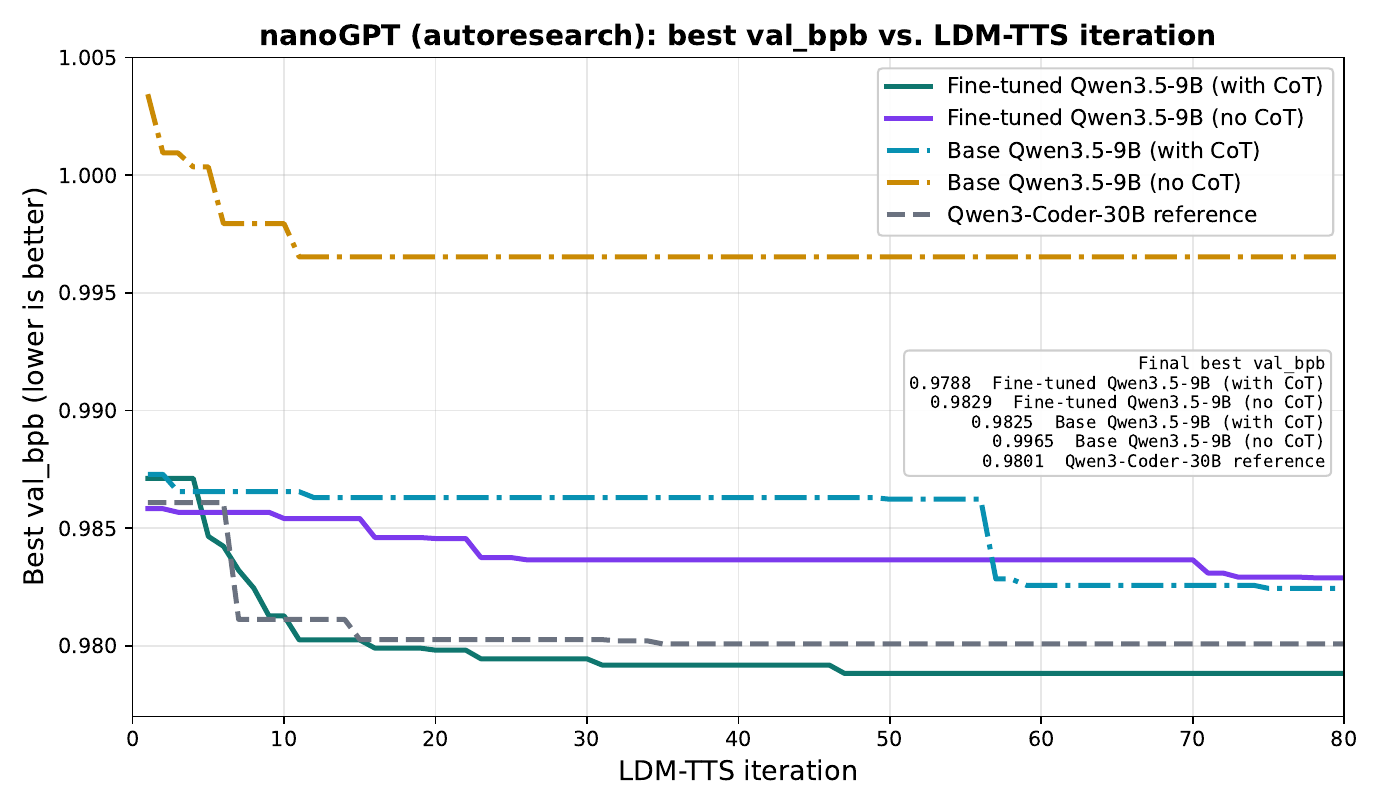}
    \caption{\textbf{Reasoning-augmentation ablation on AutoResearch.} Best validation bits-per-byte (lower is better) versus LDM-TTS iteration inside an identical acquisition-guided loop, comparing fine-tuned and base proposers with and without reasoning augmentation and the Qwen3-Coder-30B reference.}
    \label{fig:finetune-autoresearch}
\end{figure*}

\subsubsection{AutoResearch}

We first study fine-tuning in AutoResearch, where an agent repeatedly edits a persistent training program and evaluates each proposal through a fixed-budget nanoGPT training run. This setting requires the model to interpret an evolving experiment ledger, recognise performance plateaus, and propose changes that test new program mechanisms.

Figure~\ref{fig:finetune-autoresearch} compares fine-tuned and base Qwen3.5-9B proposers over 80 LDM-TTS iterations. The reasoning-augmented fine-tuned model achieves the best validation performance, reaching $\mathrm{val\_bpb}=0.9788$, lower than the Qwen3-Coder-30B reference at $0.9801$. Fine-tuning without reasoning reaches $0.9829$, the base model with reasoning reaches $0.9825$, and the base model without reasoning stalls at $0.9965$. These results indicate that fine-tuning enhances the quality and diversity of the program-search reservoir, and that the explicit trace of research progress further facilitates the model's ability to escape local program families and explore a broader region of the program space.

\subsubsection{Antibody CDRH3 Design}

We next evaluate antibody CDRH3 design, where the proposer searches a sparse and rugged sequence space under a limited binding-energy evaluation budget. We compare fine-tuned and base Qwen3.5-9B models with and without reasoning augmentation across five antigens. All variants use the same acquisition-guided LDM loop with expected-improvement acquisition and the same parallel evaluation budget.

Figure~\ref{fig:finetune-antibody} shows that fine-tuning improves over the corresponding base proposer across the five antigens. Both fine-tuned variants achieve lower binding energies than the base models, while the reasoning-augmented fine-tuned model obtains the lowest average binding energy among the four configurations. This result indicates that the distilled policy improves sequence-family proposals even in a large discrete design space.

\begin{figure*}[!ht]
    \centering
    \includegraphics[width=0.94\textwidth]{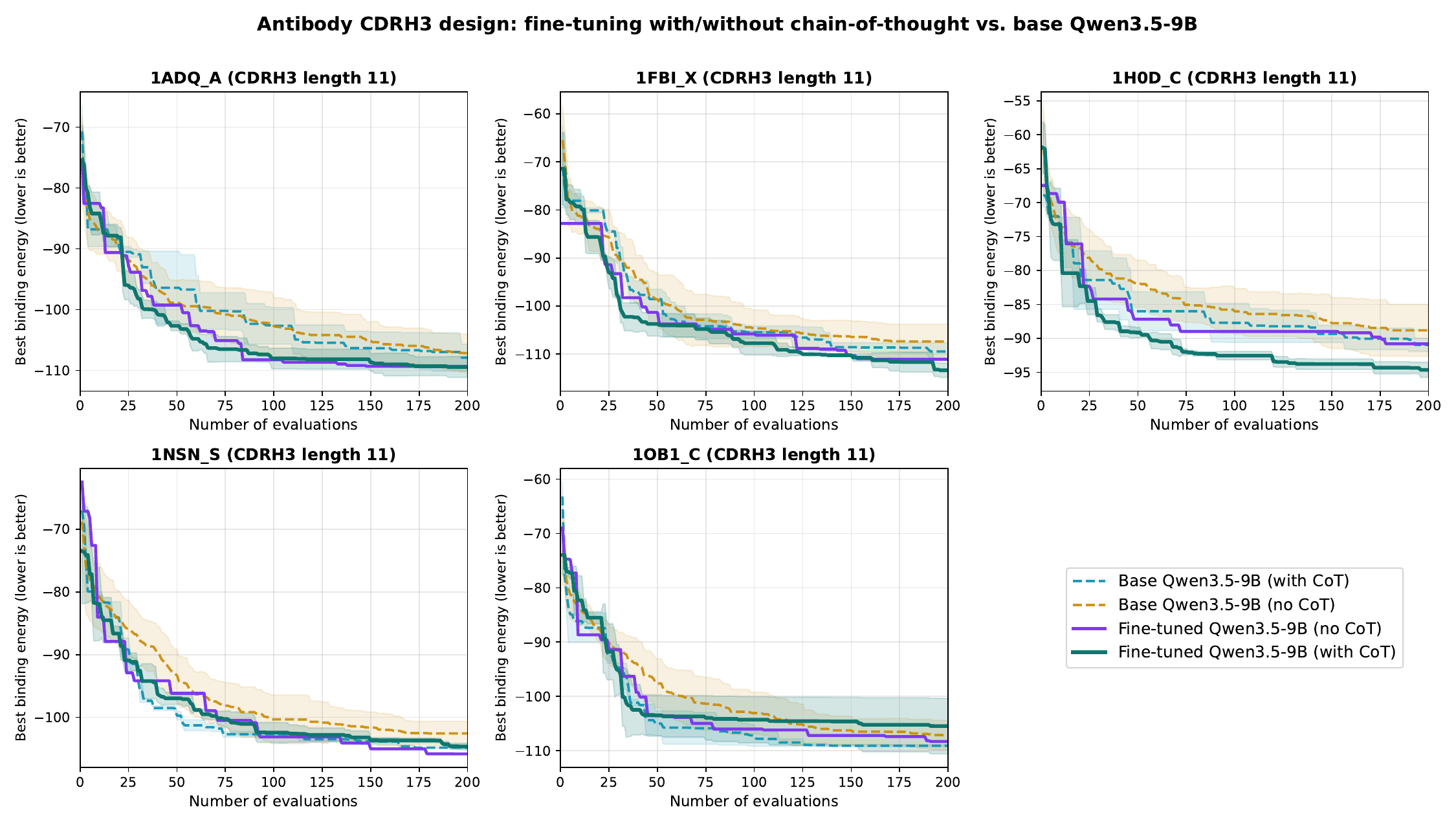}
    \caption{\textbf{Discovery fine-tuning for antibody CDRH3 design.} Best Absolut! binding energy, lower is better, versus the number of evaluations across five antigens. Fine-tuned Qwen3.5-9B variants with and without reasoning augmentation are compared with their corresponding base proposers in the same acquisition-guided LDM loop. Shaded bands denote mean $\pm$ standard deviation across available seeds.}
    \label{fig:finetune-antibody}
\end{figure*}

\subsubsection{Multi-objective Small-Molecule Discovery}

Finally, we evaluate cross-task transfer on the KRAS G12D molecular-design task. The task jointly optimises docking and neural activity, and performance is measured by best-so-far Pareto-front hypervolume over 80 expensive evaluations.

Figure~\ref{fig:finetune-molecule} compares five inference-time policies inside the same LDM loop. The reasoning-augmented fine-tuned Qwen3.5-9B model achieves the highest final hypervolume, $26.660 \pm 2.608$. It exceeds the same fine-tuned student without reasoning, $22.279 \pm 3.828$, and both DeepSeek V4 Flash teacher variants, which reach $24.548 \pm 3.589$ with reasoning and $23.582 \pm 2.832$ without reasoning. The untuned Qwen3.5-9B base proposer reaches $16.489 \pm 5.867$. The result shows that the distilled policy transfers beyond the source task and that reasoning augmentation improves this transfer, although the intervals overlap.

\begin{figure*}[h!]
    \centering
\includegraphics[width=0.82\textwidth]{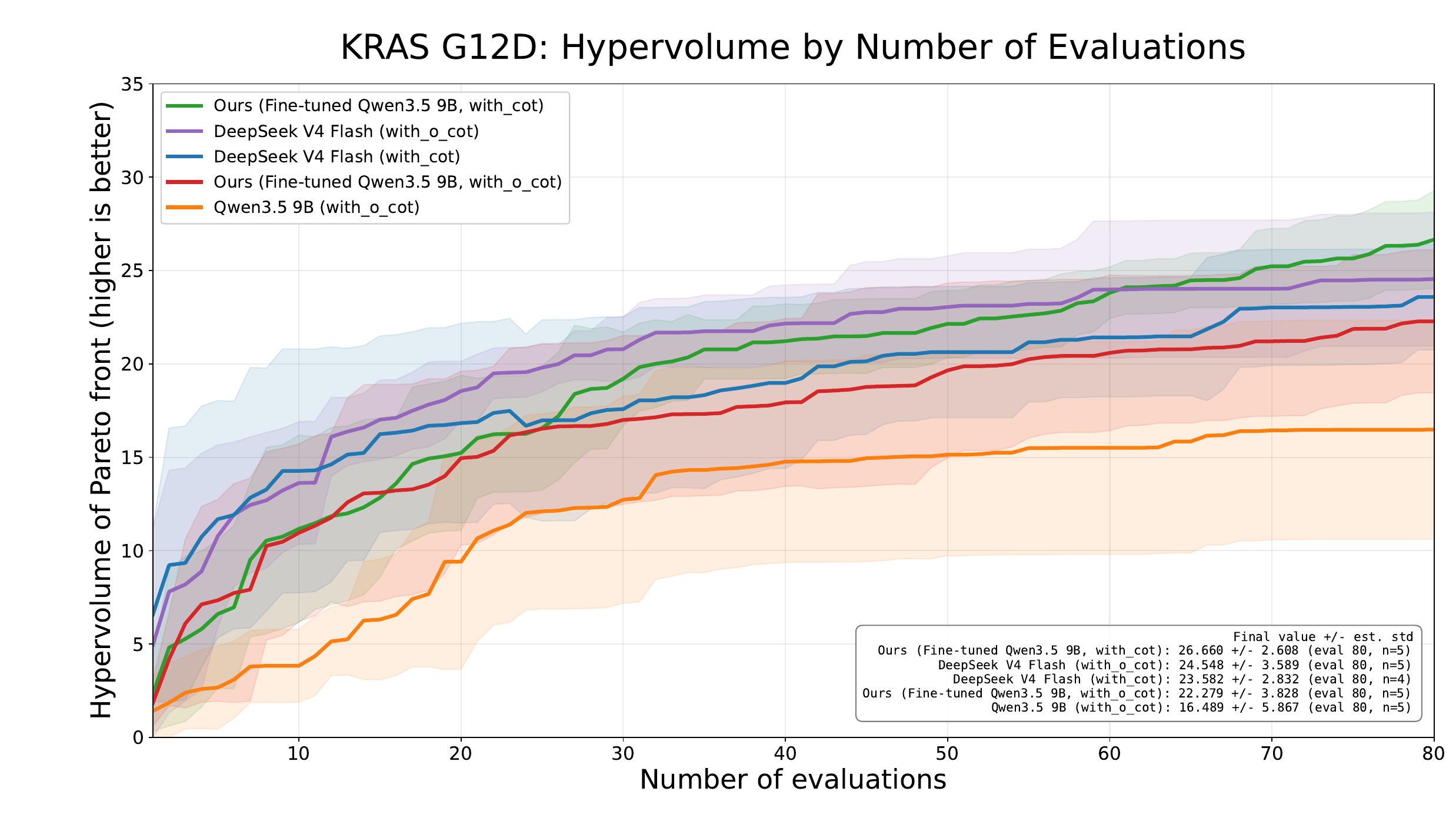}
    \caption{\textbf{Fine-tuning with reasoning-augmented discovery data on KRAS G12D.} Best-so-far Pareto-front hypervolume (higher is better) versus the number of expensive evaluations, comparing fine-tuned and base proposers inside the same LDM loop. Shaded bands denote one standard deviation.}
    \label{fig:finetune-molecule}
\end{figure*}

\subsubsection{Out-of-distribution Generalisation}

The fine-tuned policy transfers not only across the source task but also to unseen targets and tasks. In a case-level study, we hold two of the five antibody targets (\texttt{1FBI\_X}, \texttt{1H0D\_C}) out of training entirely, fine-tune on the other three, and evaluate on the held-out pair inside the same acquisition-guided loop. Figure~\ref{fig:protein-ood-mixed32} shows that this out-of-distribution model matches or exceeds the in-distribution model (which \emph{was} trained on these antigens) on both held-out targets, and clearly beats the base proposer. In a more demanding task-level study, we remove the entire protein task from training and fine-tune only on the \texttt{nanoGPT} and small-molecule trajectories; Figure~\ref{fig:app-task-ood} shows that this model, which has never seen an antibody, matches the in-distribution model on four of the five targets. Since these models never observed the test antigens, the gains cannot be memorisation---the distilled acquisition policy itself generalises. The quantitative details and analysis of both studies are reported in Appendix~\ref{app:finetuning-details}.

\begin{figure*}[!ht]
\centering
\includegraphics[width=0.98\linewidth]{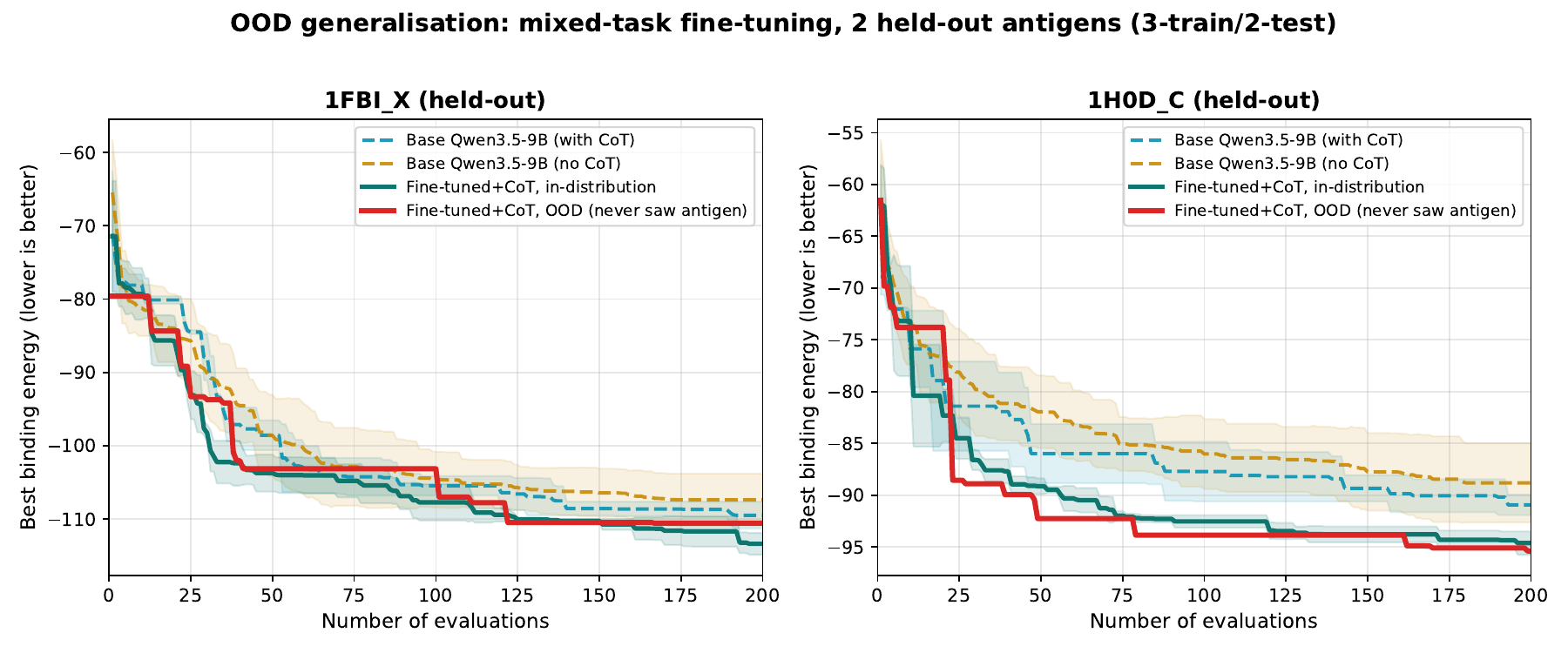}
\caption{\textbf{Out-of-distribution generalisation to held-out antigens (3-train/2-test).} Best binding energy versus number of evaluations on the two held-out antigens. The mixed-task fine-tuned model evaluated on antigens it never saw during training (red) matches or exceeds the in-distribution fine-tuned model (teal) and clearly beats base Qwen3.5-9B with and without chain-of-thought (dashed).}
\label{fig:protein-ood-mixed32}
\end{figure*}

\begin{figure*}[!ht]
\centering
\includegraphics[width=0.98\linewidth]{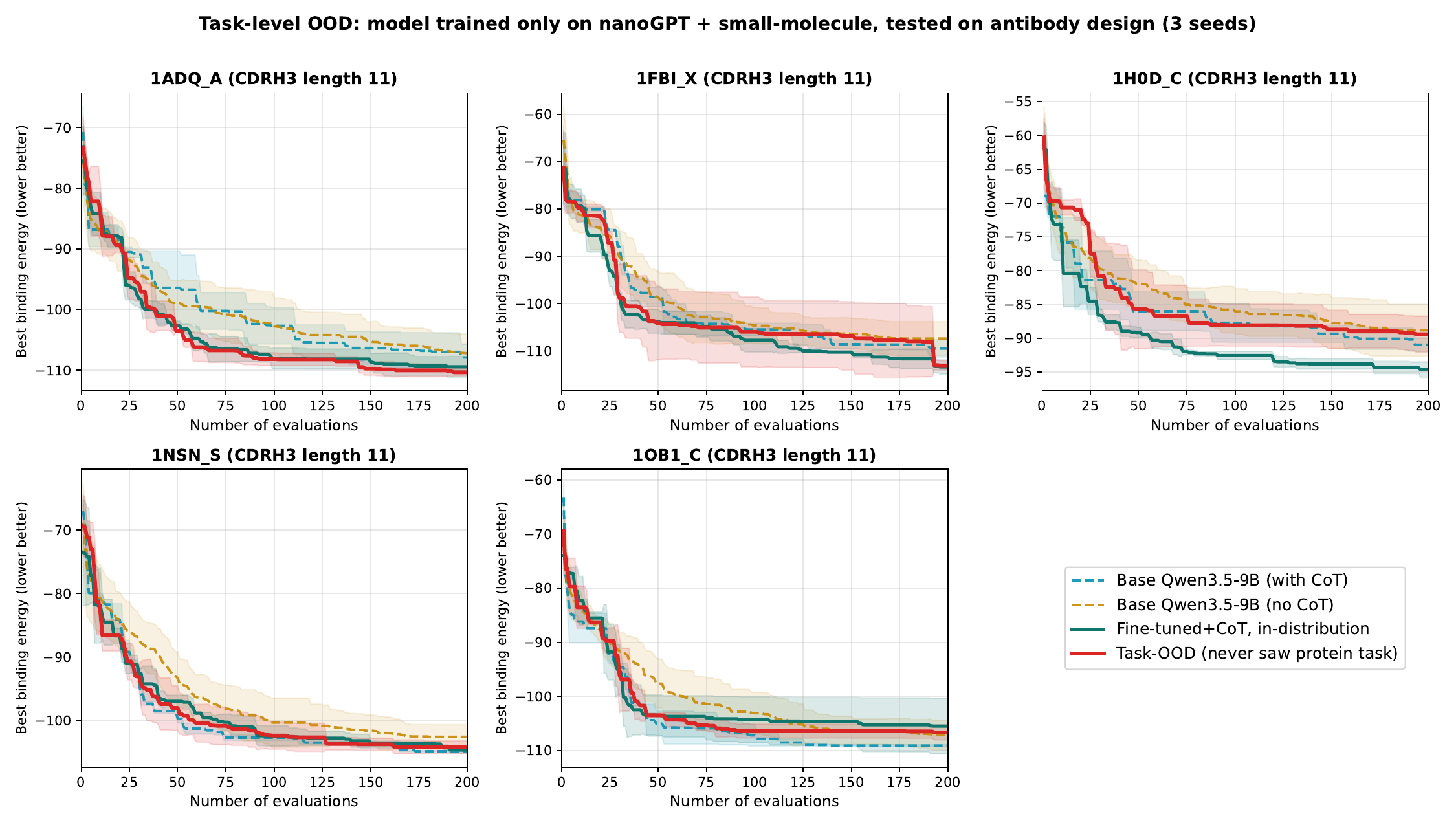}
\caption{\textbf{Task-level out-of-distribution generalisation.} Best binding energy versus number of evaluations on all five antibody targets. The task-level model (red) was fine-tuned on \emph{only} the \texttt{nanoGPT} and small-molecule trajectories and never saw the protein task, yet inside the antibody acquisition loop it matches or exceeds the in-distribution fine-tuned model (teal) on four targets (\texttt{1FBI\_X}, \texttt{1ADQ\_A}, \texttt{1OB1\_C}, \texttt{1NSN\_S}) and clearly beats base Qwen3.5-9B with and without chain-of-thought (dashed). The exception is \texttt{1H0D\_C}, whose narrow optimum requires \emph{de novo} motif design rather than candidate selection, the one capability that needs in-domain knowledge.}
\label{fig:app-task-ood}
\end{figure*}

Across programs, biological sequences, and molecules, fine-tuning improves the proposal reservoir within the online LDM loop. The gains are strongest when the model is trained to represent not only the next action but also the surrogate-grounded rationale for why that action advances scientific search, and they persist for targets and tasks that the distilled policy never encountered during training.

\subsection{Overall Discussion}
\label{sec:overall-discussion}

Across all three experimental domains, our evaluation identifies four consistent properties of the LDM framework, aligned with the states of scientific knowledge in \S\ref{sec:ldm}.

First, the case studies in \S\ref{sec:cases} consistently reproduce the exploit-explore-discover cycle defined in the LDM. In every domain, surrogate-guided optimisation delivers rapid initial gains before reaching a local plateau. The LDM detects stagnation and triggers a discovery step that expands or revises the active search space: new program mechanisms in \texttt{autoresearch}, new sequence families in antibody design, and new molecular scaffolds in small-molecule discovery. Performance improvement resumes after each search-space revision. This plateau-and-revision pattern distinguishes LDM from standard Bayesian optimisation, which operates on a fixed, pre-defined design space.

Second, the ablation studies in \S\ref{sec:ablation-main} isolate the individual contributions of three core components. Increasing the test-time search budget improves performance only when candidates are selected by a calibrated, uncertainty-aware acquisition function, not by raw LLM likelihood alone. Acquisition functions that balance exploitation and exploration (EI, UCB) consistently outperform the purely exploitative posterior mean. The discovery mechanism is an independent, necessary component: when disabled, search is confined to a fixed feature set and converges prematurely to a worse plateau. This confirms that expanding the search frontier, rather than only refining within a fixed space, drives sustained improvement in large, open-ended design problems.

Third, the baseline comparisons in \S\ref{sec:performance} show that combining Bayesian surrogate acquisition with LLM-driven proposal generation yields competitive performance across all three domains. The optimal LDM variant is domain-dependent: policy-mode LDM performs best in antibody design, where it expands candidate coverage via structured search-space parameterisation; direct-generation LDM performs best in code and molecular design, where the LLM’s pre-trained priors provide strong structural guidance. The unified tilted-distribution framework supports both modes without changes to the core acquisition logic.

Fourth, the fine-tuning experiments in \S\ref{sec:finetune-experiments} demonstrate that the LDM search policy can be distilled into the LLM reservoir. Fine-tuning improves proposal quality across all domains, with further gains from reasoning-augmented training that includes surrogate-grounded rationales. The distilled policy generalises out-of-distribution to unseen targets and entirely unseen tasks, indicating that it learns a general search strategy rather than memorising task-specific designs.

Taken together, these results frame scientific design as a sequential decision problem, rather than a pure prediction problem. The LLM constructs and adapts candidate proposals, while the surrogate and acquisition function allocate limited evaluation budget to the most informative candidates. Limitations of the current framework include the cubic scaling of exact Gaussian processes with observation count, and the risk of surrogate misspecification or reward exploitation. Future work will address scalability to larger experimental budgets and more robust surrogate modelling.

\section{Related Work}
\label{sec:related}

This section situates the LDMs within the broader literature on generative models, statistical optimisation, and their intersection for scientific discovery. The aim is not to provide an exhaustive survey, but a focused positioning of LDM within the key strands of related work that directly motivate its design and delineate its contributions. We structure the discussion in three parts. We first review the remarkable progress in LLMs, including their scaling laws, emergent reasoning capabilities, and post-training methods; we then highlight the limitations of pure generative models for discovery, particularly the lack of calibrated uncertainty and expensive, sparse reward signals. Next, we survey the statistical learning paradigm for scientific discovery, from multi-armed bandits to Bayesian optimisation (BO), and discuss its strengths and challenges in high-dimensional or open-ended, unrepresented scientific spaces. Finally, we examine the growing body of work on hybrid approaches that combine LLMs with BO, clarify the differences between LDM and the most closely related methods, and articulate the open problems our framework addresses.

\subsection{Large Language Models and Their Limitations in Scientific Discovery}

The early development of language models was marked by several landmark architectures, including the Transformer \citep{vaswaniAttentionAllYou2017}, GPT-1 \citep{radfordImprovingLanguageUnderstanding2018}, and BERT \citep{devlinBERTPretrainingDeep2019}. The past few years have witnessed transformative progress in large language models. Scaling laws \citep{kaplanScalingLawsNeural2020} have established a predictable correlation in LLMs' performance with the scale of model weights, data, and test-time compute. The remarkable success of GPT-3 \citep{brownLanguageModelsAre2020} further propelled research along the scaling direction. As models scale, they exhibit emergent abilities \citep{weiEmergentAbilitiesLarge2022} not present in smaller counterparts, including complex reasoning via chain-of-thought prompting \citep{weiChainofThoughtPromptingElicits2023,kojimaLargeLanguageModels2023}, in-context learning \citep{hanUnderstandingEmergentInContext2025}, and instruction following \citep{ouyangTrainingLanguageModels2022}. In addition, post-training methods, such as supervised fine-tuning (SFT) and reinforcement learning from human feedback (RLHF) \citep{ouyangTrainingLanguageModels2022}, can better adapt LLMs to downstream tasks and the users' preferences.

Since 2023, \emph{Test-time compute} has been a crucial development that pushes LLMs from language understanding to reasoners for complex problems, where the key insight is that the use of additional computation at inference time greatly improves output quality. Early research utilises verifiers \citep{cobbeTrainingVerifiersSolve2021, lightmanLetsVerifyStep2023} to guide reasoning to solve complex mathematical problems. Further, \emph{tree of thoughts} \citep{yao2023tree} performs deliberate search over intermediate reasoning states with self-evaluation. Self-Refine \citep{madaan2023selfrefine} iterates generate--critique--revise cycles. More generally, \citet{snell2025scaling} show that optimal allocation of test-time compute allows small models to outperform much larger models. The underlying search machinery, including UCT \citep{kocsis2006bandit} and the PUCT rule popularised by AlphaZero \citep{silver2017mastering}, has been applied to LLM decoding and training \citep{fengAlphazerolikeTreeSearchCan2024}. The combination of inference-time search with verifier-guided refinement underpins renowned systems such as OpenAI's o1 \citep{openaiOpenAIO1System2024} and DeepSeek-R1 \citep{deepseek-aiDeepSeekR1IncentivizingReasoning2025}. When cheap supervision is repeatedly available, \emph{reinforcement learning with verifiable rewards} (RLVR) \citep{deepseek-aiDeepSeekR1IncentivizingReasoning2025, shaoDeepSeekMathPushingLimits2024} has led to impressive results in mathematics, coding, and reasoning tasks. A unified tutorial perspective on these methods is provided by \citet{wangTutorialLLMReasoning2025}, and OpenR \citep{wang2024openr} further establishes a comprehensive technical framework.

Despite the significant progress of LLM reasoning, how LLMs engage in open-ended scientific discovery remains questionable. Studies show that LLMs are also poorly calibrated for scientific decisions---their internal confidence does not reliably reflect uncertainty about an external property \citep{guptaLLMsBayesianOptimization2025,kristiadiSoberLookLLMs2024}. Moreover, end-to-end execution benchmarks consistently show that frontier LLM agents cannot yet reliably discover novel, high-quality solutions: on ResearchClawBench the strongest autonomous agent scores only 21.5/50 across 40 real-paper-derived tasks \citep{xu2026researchclawbench}, and on PaperBench agents replicate ICML papers at 21.0\% against 41.4\% by PhD researchers \citep{starace2025paperbench}. Even at the ideation stage, novelty rankings that initially favour LLMs \citep{si2024novelideas} reverse once ideas are executed \citep{si2025executiongap}. LLM-generated ideas also exhibit narrow exploration, converging on the centroid of existing literature \citep{tang2026narrowexploration}. 

Therefore, scientific discovery essentially requires a novel paradigm of foundation models.  Crucially, the key challenge is not intrinsic LLM failures: when proposals are paired with automated evaluators and search, genuinely novel results are recovered \citep{novikov2025alphaevolve}. The bottleneck is the absence of a calibrated external value signal.  Test-time compute and reinforcement learning typically rely on cheap verifiers or reward models, such as ground-truth answers or trained networks from large data, to guide the search of reasoning steps. This is built upon access to rich feedback data or ground-truth answers, which, however, is usually scarce in scientific domains.

\subsection{Statistical Learning for Scientific Discovery}

A mature body of statistical learning approaches addresses the problem of optimising a black-box reward function with noisy evaluations. \emph{Multi-arm bandit} constitutes a representative framework conceptualising the exploration--exploitation trade-off \citep{lattimoreBanditAlgorithms2020}, where the upper confidence bound (UCB) \citep{auerUsingConfidenceBounds2002} provides a simple way to balance these competing objectives with sublinear regret guarantees. However, classical multi-arm bandit settings assume a finite set of arms, yet scientific design spaces are often effectively unbounded.

This challenge of unbounded space is then captured by the \emph{infinitely-armed bandit} \citep{berry1997bandit, wang2008infinitely} frameworks, where new arms must be discovered from a \emph{reservoir}. The performance in such settings is governed by the reservoir's tail of near-optimal arms, a quantity known as the \emph{discovery gap}. The LDM framework explicitly adopts this view, treating the LLM as an adaptive reservoir and analysing the discovery gap through a missing-mass argument \citep{mcallester2000convergence, berend2013concentration} in \S~\ref{sec:theory-main}. In comparison, LDM utilises an LLM as its reservoir, endowing the search process with broad scientific priors and rich semantic knowledge to inform candidate generation. In addition, the LLM’s context $\mathcal{C}_t$ encapsulates the full evaluation history, enabling in-context learning to synthesise insights from prior trials and propose progressively more effective candidate designs. The regret analysis in \S~\ref{sec:theory-main} decomposes the cost of search into a surrogate-estimation term and a reservoir-quality term, with the latter shrinking as more inference-time compute is spent.

Scientific discovery tasks are further constrained by costly evaluation procedures and limited computational budget. This demand motivates \emph{model-based} bandit variants that explicitly extract knowledge from evaluation history and make the most of such knowledge in decision making. To this end, \emph{Bayesian optimisation} (BO) \citep{mockusBayesMethodsSeeking1975,jonesEfficientGlobalOptimization1998,garnettBayesianOptimization2023} provides a powerful framework for such a variant that uses a probabilistic surrogate following the rigour of Bayesian inference, where typically a Gaussian process (GP) \citep{rasmussen2006gaussian} serves as the surrogate. The GP provides predictive mean and variance, which are combined into an \emph{acquisition function} that guides the selection of the next evaluation. For example, canonical implementations such as Efficient Global Optimisation (EGO) \citep{jonesEfficientGlobalOptimization1998} and GP-UCB \citep{srinivas2010gaussian} use expected improvement (EI) and UCB \citep{auerUsingConfidenceBounds2002} as their acquisition function, respectively. Theoretical analysis shows that GP-UCB enjoys sublinear cumulative regret in terms of the maximum information gain, a measure of how informative the observations are about the objective \citep{srinivas2010gaussian}. These theoretical guarantees make BO a principled tool for experimental design for scientific discovery \citep{yuEfficientPrincipledScientific2026}. Recent work further demonstrates this value in costly mathematical discovery: a BO--MCTS framework that searches sphere-packing SDP formulations achieved new state-of-the-art upper bounds in twelve dimensions, despite evaluations that can take days \citep{tutunov2025modelbased}.

BO has been successfully applied across chemistry, materials science, and biology, often requiring domain-specific adaptations. Examples include the PHOENICS \citep{hasePhoenicsBayesianOptimizer2018} and Gryffin \citep{haseGryffinAlgorithmBayesian2021} frameworks for chemical discovery. Combinatorial BO frameworks such as AntBO \citep{khan2022antbo} have been developed to handle the discrete sequence space of CDRH3 loops for antibody design. Benchmarking studies across multiple materials domains \citep{liangBenchmarkingPerformanceBayesian2021} have compared different surrogate models and acquisition functions, while applications in chemical synthesis \citep{shieldsBayesianReactionOptimization2021} and material discovery \citep{shoyebraihanAcceleratingMaterialDiscovery2024} further demonstrate the versatility of BO. For a comprehensive overview of BO for chemical problems, see \citep{wuRaceBottomBayesian2024}. To search among discrete sequences (e.g., SMILES and proteins), methods such as BOSS (Bayesian Optimisation over String Spaces) \citep{mossBOSSBayesianOptimization2020} use string kernels to directly construct GPs over the string space. To handle molecule design, latent-space approaches combine variational autoencoders with BO, though they struggle with validity and out-of-distribution generation \citep{griffiths2020constrained}. COMBO uses a graph Cartesian product to define a Gaussian-process BO method for combinatorial variables \citep{oh2019combo}; the later MCBO framework provides benchmarks for combinatorial and mixed-variable BO \citep{dreczkowskiFrameworkBenchmarksCombinatorial2023}. For multi-objective BO, model-assisted S-metric selection \citep{ponweiserMultiobjectiveOptimizationLimited2008} and EHVI \citep{emmerich2011ehvi} both build on hypervolume, while qEHVI/NEHVI address parallel and noisy settings \citep{daulton2020differentiable,daulton2021parallel}.

Despite these advances, pure statistical methods face fundamental challenges. First, they struggle to propose valid candidates in vast, unbounded spaces of practical scientific problems. Next, they cannot easily incorporate high-level semantic domain knowledge. Additionally, they are often unresponsive to user-specified design constraints or preferences. These limitations create a natural opportunity for complementarity with generative models, which can propose semantically meaningful candidates and be steered by domain knowledge. This brings us to the growing body of hybrid approaches that combine the strengths of LLMs and BO.

\subsection{Hybrid Approaches: LLM-Guided Discovery}

The recognition that LLMs and statistical learning offer complementary strengths has led to a growing body of hybrid methods. A high-level approach is OPRO \citep{yangLargeLanguageModels2024}, which uses an LLM as an optimiser by describing the task in natural language and iteratively generating solutions from a prompt containing a history of solution--score pairs. While powerful in prompt optimisation, OPRO lacks the calibrated uncertainty and principled exploration of a Bayesian surrogate.

Several works have integrated LLMs directly into the BO pipeline. LLAMBO \citep{liuLargeLanguageModels2024} uses LLMs for zero-shot warm-starting, surrogate modelling, and candidate sampling, showing gains in low-data hyperparameter optimisation regimes. BoChemian \citep{rankovicBoChemianLargeLanguage2023} and its extension \citep{rankovicLargeLanguageModels2025a} use LLM embeddings as features for BO over chemical reactions, demonstrating that fine-tuning LLMs with uncertainty-aware objectives can improve discovery rates. LABO \citep{chenLABOLLMAcceleratedBayesian2026} proposes a gating criterion to dynamically balance LLM predictions against experimental observations. Other works use LLMs to generate or refine kernels for GP surrogates \citep{suwandiAdaptiveKernelDesign2025}, handle natural language feedback \citep{ramos2023bayesian}, or perform analog circuit design \citep{yinADOLLMAnalogDesign2024,chenLLMEnhancedBayesianOptimization2024}. The LLM-as-warm-start paradigm, where the LLM suggests the initial candidates or a search region, is also explored in \citep{ramos2023bayesian, yinADOLLMAnalogDesign2024}. In catalysis, \citet{ramos2023bayesian} use frozen LLMs for in-context learning to enable BO without feature engineering, demonstrating the potential of LLMs to operate directly in language space.

Despite various attempts to combine the advantages of LLMs and BO, some recent studies have challenged their solidity. A sobering perspective is provided by \citet{guptaLLMsBayesianOptimization2025}, who find that current LLMs show no sensitivity to experimental feedback in BO tasks, and classical methods such as linear bandits and GP optimisation consistently outperform LLM agents. Similarly, \citet{kristiadiSoberLookLLMs2024} take a dispassionate stance, concluding that uncertainty taken directly from point-estimated LLMs is unreliable, and that LLMs are useful for BO over molecules primarily only when they serve as fixed feature extractors for principled GP surrogates and are pretrained or finetuned on domain-specific data. These critiques directly motivate LDM's design: we retain an explicit calibrated posterior (the GP surrogate) and use the LLM only as a candidate reservoir.

Another closely related work to LDM is dLLM \citep{yuan2026diffusion}, which for \emph{offline} black-box optimisation runs a masked-diffusion tree search where leaf candidates are scored by expected improvement under a GP fitted to an offline dataset. In contrast, LDM targets the standard \emph{online} recurrent loop with sequential, expensive evaluations. The core is that LDM must raise designs that are valuable not only for the current trial but also for gaining knowledge for the future.

In summary, LLMs provide powerful generative priors but lack calibration and robust optimisation; meanwhile, statistical methods provide principled, uncertainty-aware search but struggle to propose candidates in semantic, combinatorial spaces. LDM bridges this gap by making the acquisition function the value signal of an LLM-driven inference-time search, creating a principled, model-based framework for scientific discovery.

\section{Conclusion, Limitations, and Outlook}
\label{sec:conclusion}

The core challenge of scientific discovery lies in searching for optimal solutions within vast, structured, and open-ended design spaces under a limited budget of costly evaluations. Existing paradigms exhibit complementary limitations: large language models (LLMs) possess strong structured generative capabilities and domain priors, but cannot reliably estimate the true value of external objectives or their associated uncertainty. Bayesian optimisation, by contrast, delivers uncertainty-aware experimental decisions from sparse observations, yet struggles to autonomously propose valid candidates in complex discrete spaces. This work presents the Large Discovery Model (LDM), an experiment-grounded recurrent architecture that deeply couples the generative prior of LLMs with the value signal from a Gaussian process surrogate. Framing scientific design as a sequential inverse decision problem under an evaluation budget, LDM operates a closed loop of generation, evaluation, and update, which simultaneously expands the search frontier and allocates experimental resources precisely.

At the theoretical and algorithmic level, LDM establishes a KL-regularised, acquisition-guided optimisation framework and derives a closed-form acquisition-tilted optimal search distribution. This formulation extends the classical dichotomy between exploitation and exploration into a tripartite scientific search regime of exploitation, exploration and discovery, unifying the generative prior and empirical value signal within a single mathematical framework. We further provide a complete regret decomposition that delineates the roles of generative coverage gap, surrogate fitting error, and sampling strategy suboptimality. Algorithmically, the framework supports two candidate generation paradigms, direct generation and indirect parameterisation, and is complemented by a decomposed feedback mechanism to guide targeted iterative refinement, making it adaptable to design spaces of varying structure. Building on this core, we introduce a post-training amortisation scheme: through reasoning-augmented supervised fine-tuning, the high-budget test-time search policy is amortised into the model weights, substantially reducing inference-time computational cost while preserving search quality.

Empirical evaluation across three heterogeneous domains, including the auto-research scenario of neural-network training program, antibody CDRH3 sequence design, and multi-objective molecular optimisation, systematically validates the generality of the LDM framework and its core findings. First, the structured generative capacity of LLMs and the uncertainty calibration of Bayesian surrogates exhibit strong complementarity, with their integrated performance outperforming the reported LLM-only and BO-only baselines. Second, expansion of the search frontier is associated with improvements after observed local plateaus. Third, both test-time compute scaling and post-training fine-tuning improve search efficiency in the reported studies. Fine-tuning experiments further demonstrate stable gains across cross-target and cross-task transfer, indicating that the model learns generalisable discovery decision logic rather than task-specific memorisation.

Nevertheless, the current framework is subject to several constraints and limitations. The Gaussian process surrogate incurs cubic computational complexity in the number of observations, limiting its scalability under large experimental budgets. Test-time search relies on batch candidate sampling, which entails substantial LLM inference overhead. The kernel functions and feature representations of the surrogate still require manual customisation per domain, resulting in limited automation for cross-domain transfer. Additionally, framework performance is bounded by the coverage of the LLM's pretrained domain knowledge; in domains with weak priors, it is difficult to achieve substantial superiority over specialised methods. Finally, the current closed loop is updated solely based on in-process experimental observations, and has not yet systematically incorporated existing external knowledge and data, leaving a large body of unknown knowns unutilised in the search process.

For future work, a core direction is to address the utilisation of unknown knowns. We will introduce memory retrieval mechanisms and federated discovery frameworks to incorporate external knowledge bases, published experimental results, and data from parallel exploration processes into the search closed loop. Combined with techniques such as sparse Gaussian processes \citep{snelsonSparseGaussianProcesses2005,titsiasVariationalLearningInducing2009}, this will simultaneously address limited context capacity and high computational complexity as observational data scales. Building on this, we will explore end-to-end surrogate specification, in which the LLM autonomously learns kernel functions and prior settings for the surrogate, eliminating manual adaptation across domains. We will also advance discovery-oriented post-training paradigms, exploring preference learning and reinforcement learning to further internalise acquisition-guided decision logic into the model's generative distribution.

\bibliographystyle{plainnat}
\bibliography{references}

\newpage
\appendix

\begin{center}
  {\sffamily\bfseries\LARGE\color{ldmdarkred}
   Appendix of Large Discovery Models\par}
  \vspace{5pt}
  {\color{ldmedge}\hrule height 0.5pt}
\end{center}

\vspace{8pt}

\startcontents[appendix]
\printcontents[appendix]{l}{0}[2]{%
  \noindent{\sffamily\bfseries\large\color{ldmdarkred}Appendix Contents}\par
  \vspace{4pt}}

\section{Notation}
\label{app:notation}

\begingroup
\small
\renewcommand{\arraystretch}{1.15}
\begin{longtable}{@{}p{0.24\textwidth}p{0.70\textwidth}@{}}
\toprule
\textbf{Symbol} & \textbf{Meaning} \\
\midrule
\endfirsthead
\toprule
\textbf{Symbol} & \textbf{Meaning (continued)} \\
\midrule
\endhead
\midrule
\multicolumn{2}{r@{}}{\emph{Continued on the next page}}\\
\endfoot
\bottomrule
\endlastfoot
$\X$ & Design space. \\
$x$ & Candidate design. \\
$x^\star$ & Global maximiser of the unknown objective. \\
$R(x)$ & Unknown objective or reward evaluated at $x$. \\
$\D_t$ & Empirical evaluation history available at round $t$. \\
$\C_t$ & Broader search context used to condition candidate generation. \\
$r_i$ & Observed reward for evaluated design $x_i$. \\
$p_\theta$ & Base generative-model distribution. \\
$p_{\theta,\alpha}$ & Effective proposal induced by inference configuration $\alpha$; in the controlled base-measure-exponent case, $p_{\theta,\alpha}\propto p_\theta^\alpha$. \\
$\theta$ & Parameters of the generative model. \\
$\alpha$ & Inference configuration, including prompting, decoding, reasoning, refinement, and inference-time compute. \\
$A_t$ & Active proposal support at round $t$. \\
$\mu_t(x)$ & Surrogate posterior mean at $x$. \\
$\sigma_t(x)$ & Surrogate posterior standard deviation at $x$. \\
$a_t(x)$ & Acquisition value at $x$ (also written $\acq_t(x)$). \\
$\pi_t$ & Acquisition-tilted search policy. \\
$\eta$ & Acquisition-tilt strength. \\
$N$ & Candidate-pool size. \\
$b$ & Batch size for parallel evaluation. \\
$T$ & Sequential evaluation budget. \\
$B_t$ & Selected evaluation batch at round $t$. \\
$k$ & Number of candidates selected by Gumbel-top-$k$; $k(\cdot,\cdot)$ denotes the GP kernel when used with function arguments. \\
$m(\cdot)$ & GP prior mean function. \\
$k(\cdot,\cdot)$ & GP kernel function. \\
$\phi(\cdot)$ & Representation map used by the kernel. \\
$\epsilon_t$ & Observation noise at round $t$. \\
$\lambda$ & Noise regulariser in the information-gain term. \\
$\beta_t$ & UCB confidence parameter. \\
$\gamma_T$ & Maximum information gain over $T$ evaluations. \\
$\zeta_t$ & Near-UCB acquisition tolerance. \\
$\kappa_t(\zeta_t)$ & Proposal mass assigned to the near-UCB set. \\
$\rho_t$ & Per-round sampling-failure probability allocation. \\
$\mathrm d_\X$ & Distance on the design space. \\
$r_t^{\rm cov}$ & Reservoir coverage radius at round $t$. \\
$L$ & Local regularity constant. \\
$\chi$ & H\"older exponent in the local regularity condition. \\
$T_{\mathrm{LLM}}$ & LLM decoding temperature used in experiments; in the controlled exponent ablation, $T_{\mathrm{LLM}}=1/\alpha$. \\
$H$ & Number of refinement branches. \\
$N_t^{\rm good}$ & Number of retained high-value candidates at round $t$. \\
$r^{\rm thresh}$ & Reward threshold used by task-specific filtering. \\
\end{longtable}
\endgroup

\section{Detailed Proofs}
\label{sec:theory}

We give the full proof of Proposition~\ref{prop:gibbs}, Lemma~\ref{lem:gibbs-near-ucb} and Theorem~\ref{thm:regret}.
\subsection{Proof of Proposition~\ref{prop:gibbs}}
\label{prop_proof}
\begin{proof}
Write $J(q)=\E_{x\sim q}[\acq_t]-\tfrac1\eta\KL(q\|p_{\theta,\alpha})$. For any $q$,
\[
J(q)=\sum_x q(x)\acq_t(x)-\tfrac1\eta\sum_x q(x)\log\frac{q(x)}{p_{\theta,\alpha}(x)}
=-\tfrac1\eta\sum_x q(x)\log\frac{q(x)}{p_{\theta,\alpha}(x)e^{\eta\acq_t(x)}}.
\]
Define $\pi_t(x)\propto p_{\theta,\alpha}(x)e^{\eta\acq_t(x)}$ with normaliser $Z_t$. Substituting,
\[
J(q)=-\tfrac1\eta\,\KL(q\|\pi_t)+\tfrac1\eta\log Z_t\ \le\ \tfrac1\eta\log Z_t,
\]
with equality iff $q=\pi_t$, since $\KL(q\|\pi_t)\ge 0$ with equality iff $q=\pi_t$. As
$p_{\theta,\alpha}\propto p_\theta^{\alpha}$, the normalising constants combine into $Z_t$, giving
\eqref{eq:core}.
\end{proof}

\subsection{Proof of Lemma~\ref{lem:gibbs-near-ucb}}
\label{app:gibbs-lemma}

Fix a round $t$ and condition on the history $\C_t$. Under this conditioning, the effective
reservoir support $A_t$, the inference-configured proposal $p_{\theta,\alpha}(\cdot\mid\C_t)$, and
the acquisition function $\acq_t(x)=\mu_t(x)+\sqrt{\beta_t}\sigma_t(x)$ are fixed. Let
\[
Z_t=\sum_{u\in A_t}\exp\{\eta\,\acq_t(u)\}\,p_{\theta,\alpha}(u\mid\C_t)
\]
be the normalising constant of \eqref{eq:theory-pi}. For a tolerance $\zeta_t\ge0$, write
\[
U=U_t(\zeta_t)
=
\{x\in A_t:\acq_t(x)\ge a_{t,A}^\star-\zeta_t\}.
\]
Every $u\in U$ has acquisition at least $a_{t,A}^\star-\zeta_t$. Hence the contribution of $U$ alone gives the
lower bound
\begin{align}
Z_t
&=
\sum_{u\in A_t}\exp\{\eta\,\acq_t(u)\}\,p_{\theta,\alpha}(u\mid\C_t) \notag\\
&\ge
\sum_{u\in U}\exp\{\eta\,\acq_t(u)\}\,p_{\theta,\alpha}(u\mid\C_t) \notag\\
&\ge
\exp\{\eta(a_{t,A}^\star-\zeta_t)\}\,p_{\theta,\alpha}(U\mid\C_t) \notag\\
&=
\kappa_t(\zeta_t)\exp\{\eta(a_{t,A}^\star-\zeta_t)\}.
\label{eq:zt-lower}
\end{align}

Now fix $z\ge0$ and define the bad event
\[
B_z=
\{x\in A_t:a_{t,A}^\star-\acq_t(x)>\zeta_t+z\}.
\]
For every $x\in B_z$,
\[
\acq_t(x)<a_{t,A}^\star-\zeta_t-z.
\]
Therefore,
\begin{align}
\Prob(x_t\in B_z\mid\C_t)
&=
\pi_t(B_z) \notag\\
&=
\frac{\sum_{x\in B_z}\exp\{\eta\,\acq_t(x)\}\,p_{\theta,\alpha}(x\mid\C_t)}{Z_t} \notag\\
&\le
\frac{\exp\{\eta(a_{t,A}^\star-\zeta_t-z)\}\,p_{\theta,\alpha}(B_z\mid\C_t)}
{\kappa_t(\zeta_t)\exp\{\eta(a_{t,A}^\star-\zeta_t)\}} \notag\\
&\le
\frac{\exp\{-\eta z\}}{\kappa_t(\zeta_t)},
\label{eq:gibbs-tail}
\end{align}
where the first inequality uses \eqref{eq:zt-lower}, and the second uses $p_{\theta,\alpha}(B_z\mid\C_t)\le1$. This proves the
tail form
\[
\Prob\left(
a_{t,A}^\star-\acq_t(x_t)>\zeta_t+z
\mid
\C_t
\right)
\le
\frac{e^{-\eta z}}{\kappa_t(\zeta_t)}.
\]

To obtain the high-probability statement, take
\[
z=\frac1\eta\log\frac1{\rho_t\kappa_t(\zeta_t)}.
\]
Since $\rho_t\in(0,1)$ and $\kappa_t(\zeta_t)\le1$, this $z$ is non-negative. Substitution into
\eqref{eq:gibbs-tail} gives an upper bound of $\rho_t$ on the bad event, so with conditional probability at
least $1-\rho_t$,
\[
a_{t,A}^\star-\acq_t(x_t)
\le
\zeta_t+\frac1\eta\log\frac1{\rho_t\kappa_t(\zeta_t)}.
\]
This proves Lemma~\ref{lem:gibbs-near-ucb}.

\subsection{Proof of Theorem~\ref{thm:regret}}
\label{app:main-proof}

The proof decomposes regret into a discovery term and an optimisation term, then controls the optimisation term
using UCB calibration and Lemma~\ref{lem:gibbs-near-ucb}.

As stated earlier, we can control the simple regret by analysing the upper bound of the cumulative regret, which can be decomposed by the discovery gap and optimisation regret as below.
\begin{equation}
\regret_t
=
\underbrace{R(x^\star)-R_{t,A}^\star}_{\regret_t^{\rm disc}}
+
\underbrace{R_{t,A}^\star-R(x_t)}_{\regret_t^{\rm opt}}.
\label{eq:disc-opt-decomp}
\end{equation}

We begin with bounding the discovery gap. By Assumption~\ref{ass:coverage}, for every $x\in A_t$,
\[
R(x^\star)-R(x)
\le
L\,\mathrm d_\X(x,x^\star)^\chi.
\]
Taking the infimum over $x\in A_t$ gives
\begin{align}
\regret_t^{\rm disc}
&=
R(x^\star)-\sup_{x\in A_t}R(x) \notag\\
&=
\inf_{x\in A_t}\{R(x^\star)-R(x)\} \notag\\
&\le
L\inf_{x\in A_t}\mathrm d_\X(x,x^\star)^\chi \notag\\
&=
L\left(\inf_{x\in A_t}\mathrm d_\X(x,x^\star)\right)^\chi \notag\\
&=
L(r_t^{\rm cov})^\chi.
\label{eq:disc-gap-bound}
\end{align}
Thus, the discovery price is exactly controlled by the current reservoir coverage radius.

Next, we form the simultaneous sampling event. For each $t$, apply Lemma~\ref{lem:gibbs-near-ucb} with failure probability $\rho_t$. Define
\[
\xi_t
=
\zeta_t+\frac1\eta\log\frac1{\rho_t\kappa_t(\zeta_t)}.
\]
The lemma gives, conditionally on $\C_t$,
\[
\Prob(a_{t,A}^\star-\acq_t(x_t)\le \xi_t\mid \C_t)\ge1-\rho_t.
\]
Equivalently, the conditional failure probability is at most $\rho_t$. Iterating this bound over
$t=1,\ldots,T$ and applying a union bound yields an event $\mathcal E_{\rm samp}$ such that
\[
\Prob(\mathcal E_{\rm samp})\ge 1-\sum_{t=1}^T\rho_t\ge1-\delta_{\rm samp},
\]
and on $\mathcal E_{\rm samp}$,
\begin{equation}
a_{t,A}^\star-\acq_t(x_t)\le \xi_t
\qquad\text{for all }t\le T.
\label{eq:sampling-event}
\end{equation}

As such, we can bound the optimisation term on the good event. Let $\mathcal E_{\rm gp}$ be the event in Assumption~\ref{ass:ucb-calibration}. On this event, for every
$x\in\X$ and $t\le T$,
\begin{equation}
\begin{aligned}
R(x)\le \acq_t(x),
\qquad
R(x)\ge \mu_t(x)-\sqrt{\beta_t}\sigma_t(x).
\end{aligned}
\label{eq:ucb-two-sided}
\end{equation}
On $\mathcal E_{\rm gp}\cap\mathcal E_{\rm samp}$,
\begin{align}
R_{t,A}^\star
&=
\sup_{x\in A_t}R(x) \notag\\
&\le
\sup_{x\in A_t}\acq_t(x) \notag\\
&=
a_{t,A}^\star \notag\\
&\le
\acq_t(x_t)+\xi_t,
\label{eq:best-in-a-upper}
\end{align}
where the first inequality uses UCB optimism and the last inequality uses \eqref{eq:sampling-event}. Therefore
\[
\regret_t^{\rm opt}
=
R_{t,A}^\star-R(x_t)
\le
\acq_t(x_t)-R(x_t)+\xi_t.
\]
Using the lower side of \eqref{eq:ucb-two-sided} at the sampled point $x_t$,
\begin{align}
\acq_t(x_t)-R(x_t)
&=
\mu_t(x_t)+\sqrt{\beta_t}\sigma_t(x_t)-R(x_t) \notag\\
&\le
2\sqrt{\beta_t}\sigma_t(x_t).
\end{align}
Hence
\begin{equation}
\regret_t^{\rm opt}
\le
2\sqrt{\beta_t}\sigma_t(x_t)+\xi_t.
\label{eq:opt-gap-bound}
\end{equation}

Combining \eqref{eq:disc-opt-decomp}, \eqref{eq:disc-gap-bound}, and \eqref{eq:opt-gap-bound}, on $\mathcal E_{\rm gp}\cap\mathcal E_{\rm samp}$ we have, for every $t\le T$,
\[
\regret_t
\le
L(r_t^{\rm cov})^\chi
+
2\sqrt{\beta_t}\sigma_t(x_t)
+
\xi_t.
\]
Summing over $t$ and using the monotonicity $\beta_t\le\beta_T$,
\begin{equation}
\sum_{t=1}^T\regret_t
\le
L\sum_{t=1}^T(r_t^{\rm cov})^\chi
+
2\sqrt{\beta_T}\sum_{t=1}^T\sigma_t(x_t)
+
\sum_{t=1}^T\xi_t.
\label{eq:sum-delta-before-info}
\end{equation}

By the standard information-gain variance bound for GP-UCB,
\[
\sum_{t=1}^T\sigma_t^2(x_t)
\le
C_\lambda\gamma_T,
\qquad
C_\lambda=\frac{2}{\log(1+\lambda^{-1})}.
\]
Therefore, by Cauchy--Schwarz,
\begin{equation}
\sum_{t=1}^T\sigma_t(x_t)
\le
\sqrt{T\sum_{t=1}^T\sigma_t^2(x_t)}
\le
\sqrt{T C_\lambda\gamma_T}.
\label{eq:variance-sum}
\end{equation}
Substituting \eqref{eq:variance-sum} into \eqref{eq:sum-delta-before-info} gives
\[
\sum_{t=1}^T\regret_t
\le
L\sum_{t=1}^T(r_t^{\rm cov})^\chi
+
2\sqrt{T C_\lambda\beta_T\gamma_T}
+
\sum_{t=1}^T\xi_t.
\]
Finally, substituting the definition of $\xi_t$,
\[
\frac{1}{T}\sum_{t=1}^T \regret_t
\le
\frac{L}{T}\sum_{t=1}^T(r_t^{\rm cov})^\chi
+
2\sqrt{\frac{C_\lambda\beta_T\gamma_T}{T}}
+
\frac1T\sum_{t=1}^T
\left[
\zeta_t+
\frac1\eta\log\frac1{\rho_t\kappa_t(\zeta_t)}
\right].
\]

Assumption~\ref{ass:ucb-calibration} gives $\Prob(\mathcal E_{\rm gp})\ge1-\delta_{\rm gp}$, and the sampling
argument gives $\Prob(\mathcal E_{\rm samp})\ge1-\delta_{\rm samp}$. A final union bound gives
\[
\Prob(\mathcal E_{\rm gp}\cap\mathcal E_{\rm samp})
\ge
1-\delta_{\rm gp}-\delta_{\rm samp}.
\]
This completes the proof of Theorem~\ref{thm:regret}.

\section{Technical Implementation Details}
\label{sec:impl-details}

This appendix elaborates on the core algorithmic and implementation details of the LDM framework, covering Gaussian process surrogate modelling, acquisition function design and batch sampling schemes. All implementations strictly follow the acquisition-tilted search framework presented in the main text, and only expand on the specific execution logic for surrogate model training and sampling.

\subsection{Gaussian Process Surrogate}
\label{sec:llm-bo}

LDM employs a Gaussian Process (GP) as the probabilistic surrogate model. It fits the posterior distribution of the black-box reward function from historical experimental observations, and outputs calibrated predictive means and epistemic uncertainty to provide quantitative inputs for acquisition function computation.

\subsubsection{Prior Specification}
Let $\mathcal{X}$ denote the design space, $R(x)$ the true black-box reward function. Each experimental observation contains independent Gaussian noise:
\[
r_i = R(x_i) + \epsilon_i,\quad \epsilon_i \sim \mathcal{N}(0, \sigma_n^2)
\]
where $\sigma_n^2$ is the observation noise variance.

The GP prior is defined as:
\[
R(x) \sim \mathcal{GP}\left(m(x),\, k(x, x')\right)
\]
where:
\begin{itemize}
  \item The prior mean function $m(x)$ adopts a constant prior, i.e. $m(x) = m_0$, with $m_0$ a scalar parameter that can be optimised from data;
  \item $k(x, x')$ is a positive-definite kernel function that characterises the correlation between samples in the design space. Its specific form is adapted to the design space, for example an RBF kernel for continuous parameter spaces or a string kernel for sequence spaces;
  \item The kernel hyperparameters, observation noise variance $\sigma_n^2$ and prior constant $m_0$ together form the set of parameters to be estimated for the GP.
\end{itemize}

\subsubsection{Posterior Inference}
Let $\mathcal{D}_t = \{(x_i, r_i)\}_{i=1}^{t}$ denote the cumulative historical observation dataset at iteration $t$. We define the $t \times t$ kernel matrix $\mathbf{K}_t$ and the $t$-dimensional cross-covariance vector $\mathbf{k}_t(x)$:
\[
\mathbf{K}_t =
\begin{bmatrix}
k(x_1, x_1) & k(x_1, x_2) & \dots & k(x_1, x_t) \\
k(x_2, x_1) & k(x_2, x_2) & \dots & k(x_2, x_t) \\
\vdots      & \vdots      & \ddots & \vdots      \\
k(x_t, x_1) & k(x_t, x_2) & \dots & k(x_t, x_t)
\end{bmatrix},\quad
\mathbf{k}_t(x) =
\begin{bmatrix}
k(x_1, x) \\
k(x_2, x) \\
\vdots \\
k(x_t, x)
\end{bmatrix} \in \mathbb{R}^t
\]

By Bayesian conditioning, the posterior reward distribution for any candidate $x \in \mathcal{X}$ is Gaussian, with closed-form expressions for the posterior predictive mean and predictive variance:
\begin{align}
\mu_t(x) &= m_0 + \mathbf{k}_t(x)^\top \left(\mathbf{K}_t + \sigma_n^2 \mathbf{I}\right)^{-1} \left(\mathbf{r}_t - m_0 \cdot \mathbf{1}\right) \label{eq:postmean} \\
\sigma_t^2(x) &= k(x, x) - \mathbf{k}_t(x)^\top \left(\mathbf{K}_t + \sigma_n^2 \mathbf{I}\right)^{-1} \mathbf{k}_t(x). \label{eq:postvar}
\end{align}
where $\mathbf{r}_t = (r_1, \dots, r_t)^\top$ is the vector of observed rewards, and $\mathbf{1}$ is an all-ones vector of length $t$. The posterior mean estimates the expected reward of a candidate, while the posterior variance quantifies the epistemic uncertainty of the prediction.

\subsubsection{Hyperparameter Estimation}
The GP hyperparameters are fitted via \emph{Type-II maximum likelihood estimation}, whereby optimal hyperparameter values are determined by maximising the marginal log-likelihood of the observed data. The marginal log-likelihood of the observed data $\mathcal{D}_t$ is given by:
\[
\log p(\mathbf{r}_t \mid \mathcal{D}_t) = -\frac{1}{2} (\mathbf{r}_t - m_0 \cdot \mathbf{1})^\top \left(\mathbf{K}_t + \sigma_n^2 \mathbf{I}\right)^{-1} (\mathbf{r}_t - m_0 \cdot \mathbf{1}) - \frac{1}{2}\log\left|\mathbf{K}_t + \sigma_n^2 \mathbf{I}\right| - \frac{t}{2}\log 2\pi
\]

Optimisation is performed using gradient-based algorithms such as L-BFGS. The parameters to be optimised include kernel hyperparameters such as lengthscales and output variance, observation noise variance $\sigma_n^2$, and the prior constant $m_0$. This method automatically adapts to the smoothness and noise level of the design space, improving the fitting accuracy and uncertainty calibration of the surrogate model. Hyperparameter optimisation is re-run after each iteration with new observations to update the surrogate model.

\subsection{Acquisition Functions}
Acquisition functions quantify the experimental value of each candidate based on the GP posterior mean and uncertainty, and act as the core value signal guiding search direction and balancing exploitation and exploration. LDM supports three acquisition functions for single- and multi-objective settings, all of which can be directly embedded into the acquisition-tilted search framework.

\subsubsection{Expected Improvement}
Expected Improvement (EI) \citep{jonesEfficientGlobalOptimization1998} measures the expected gain of a candidate relative to the current best observed value. It is one of the most widely used acquisition functions in single-objective optimisation, and takes the form:
\begin{equation}
a^{\text{EI}}(x) = \left(\mu_t(x) - r_t^* - \xi\right) \Phi(z) + \sigma_t(x) \varphi(z), \quad z = \frac{\mu_t(x) - r_t^* - \xi}{\sigma_t(x)}
\label{eq:ei}
\end{equation}
where $r_t^* = \max_{i \le t} r_i$ is the best reward observed so far, $\xi \ge 0$ is an exploration-adjustment hyperparameter, and $\Phi(\cdot)$ and $\varphi(\cdot)$ denote the cumulative distribution function and probability density function of the standard normal distribution, respectively. EI naturally balances exploitation and exploration, prioritising candidates with higher expected returns.

\subsubsection{Upper Confidence Bound}
The Upper Confidence Bound (UCB) \citep{srinivas2010gaussian} computes a weighted sum of the predictive mean and uncertainty. It has a concise form and enjoys rigorous theoretical guarantees on cumulative regret, and takes the form:
\begin{equation}
a^{\text{UCB}}(x) = \mu_t(x) + \sqrt{\beta_t}\, \sigma_t(x)
\label{eq:ucb}
\end{equation}
where $\beta_t > 0$ is the exploration weight coefficient, which may be set via theoretical formulae or tuned as a hyperparameter to control the exploration intensity of the algorithm. By adjusting $\beta_t$, UCB can flexibly switch between exploitation-biased and exploration-biased search strategies.

\subsubsection{Expected Hypervolume Improvement}

For multi-objective optimisation tasks, \textbf{Expected Hypervolume Improvement (EHVI)} is described alongside model-assisted S-metric selection \citep{ponweiserMultiobjectiveOptimizationLimited2008} and the EHVI formulation and computation of Emmerich et al. \citep{emmerich2011ehvi}. Ponweiser et al. introduce model-assisted S-metric selection based on hypervolume contribution; Emmerich et al. formulate and compute hypervolume-based expected improvement (EHVI). EHVI computes the expected hypervolume gain of the current empirical Pareto front after adding a candidate, and automatically balances exploitation of favourable objective trade-offs and exploration of under-sampled design regions, without requiring pre-specified objective weights.

In the multi-objective setting, the core acquisition-tilted search framework remains unchanged; only the scalar acquisition value $a_t(x)$ is replaced with the EHVI value, enabling seamless extension of LDM to multi-objective discovery.

\subsection{Batch Acquisition Sampling}
\label{sec:batch}

In parallel experimental settings, multiple candidates must be selected for simultaneous evaluation at each iteration. LDM applies Gumbel-top-$k$ to the finite candidate pool, producing a weighted sample without replacement (the Plackett--Luce batch distribution) \citep{kool2019gumbeltopk}. This realises finite-pool acquisition-tilted selection while avoiding duplicate selections in the evaluation batch.

The implementation procedure is as follows:
\begin{enumerate}
  \item For each candidate $x_i$ in the pool generated by the LLM, compute its unnormalised tilted weight:
  \[
  w_i = p_{\theta,\alpha}(x_i \mid \mathcal{C}_t) \cdot \exp\left\{\eta \cdot a_t(x_i)\right\}
  \]
  where $p_{\theta,\alpha}(x_i \mid \mathcal{C}_t)$ is the conditional generation probability of the LLM, $a_t(x_i)$ is the acquisition function value, and $\eta$ is the acquisition tilt strength coefficient.

  \item Sample independent Gumbel noise $g_i \sim \text{Gumbel}(0, 1)$ for each weight, rank candidates by $\log w_i + g_i$ in descending order, and select the top $k$ candidates to form the evaluation batch. Each selected candidate is removed from the pool after selection.
\end{enumerate}

Conditional on the generated pool, the ordered batch follows the corresponding weighted without-replacement distribution, retaining the acquisition weighting and avoiding duplicate selections in parallel evaluation pipelines.

\section{Ablation Studies}
\label{app:ablation}

These regimes stress different parts of the framework. \texttt{autoresearch} is a \emph{multi-turn code-editing} task: the agent maintains a research state, repeatedly edits a persistent artifact \texttt{train.py}, and can change the representation of the search space as new mechanisms are discovered. Small-molecule and CDRH3 design are closer to \emph{single-step proposal} tasks: at each BO round, the LLM proposes or parameterises a candidate pool, the surrogate acquisition selects from that pool, and the expensive oracle evaluates the selected designs.

The ablations therefore have two complementary purposes. First, in \S\ref{app:abl-pure-llm}, we push the boundary of a pure current-agent research loop on nanoGPT by removing the calibrated BO value entirely. This tests how far an LLM agent can go when it can read its own logs and reflect, but cannot search against a surrogate posterior. Second, in \S\ref{app:abl-ldm}, we keep the value model and study LDM test-time search itself. On nanoGPT this means varying the internal LDM-TTS settings, such as inner search budget, the discovery mechanism, and acquisition choice. On molecules and CDRH3, it means asking whether open-source LLM proposers benefit from larger proposal and acquisition-selection budgets in single-step design rounds.

\subsection{Pure LLM-based research loop}\label{app:abl-pure-llm}

This ablation asks whether the LLM research loop alone is sufficient. Operationally, it is the $\eta=0$ limit of Eq.~\eqref{eq:core}: the agent still reads the history, edits \texttt{train.py}, launches experiments, and keeps improvements, but it does not receive a BO posterior, an acquisition score, or an uncertainty-directed suggestion for what to try next. We deliberately make this a strong LLM-only baseline rather than a straw man: it is meant to push the boundary of what a current advanced agent system can do on this task. The only extra scaffolding is a short reflection instruction that asks the Codex agent to monitor the result log as a whole, ask whether the current iteration verifies the previous hypothesis, extract any new mechanistic insight, and plan the next iteration. This is not human-in-the-loop steering of the scientific direction; it is an end-to-end LLM workflow whose context explicitly asks the agent to behave like a careful experimentalist. Thus the ablation tests the best version of the claim that an LLM, supplied with its own experimental transcript and enough task knowledge to edit the architecture and optimiser, can bootstrap an autonomous research process without an explicit value model.

\begin{figure}[t]
    \centering
    \includegraphics[width=\linewidth]{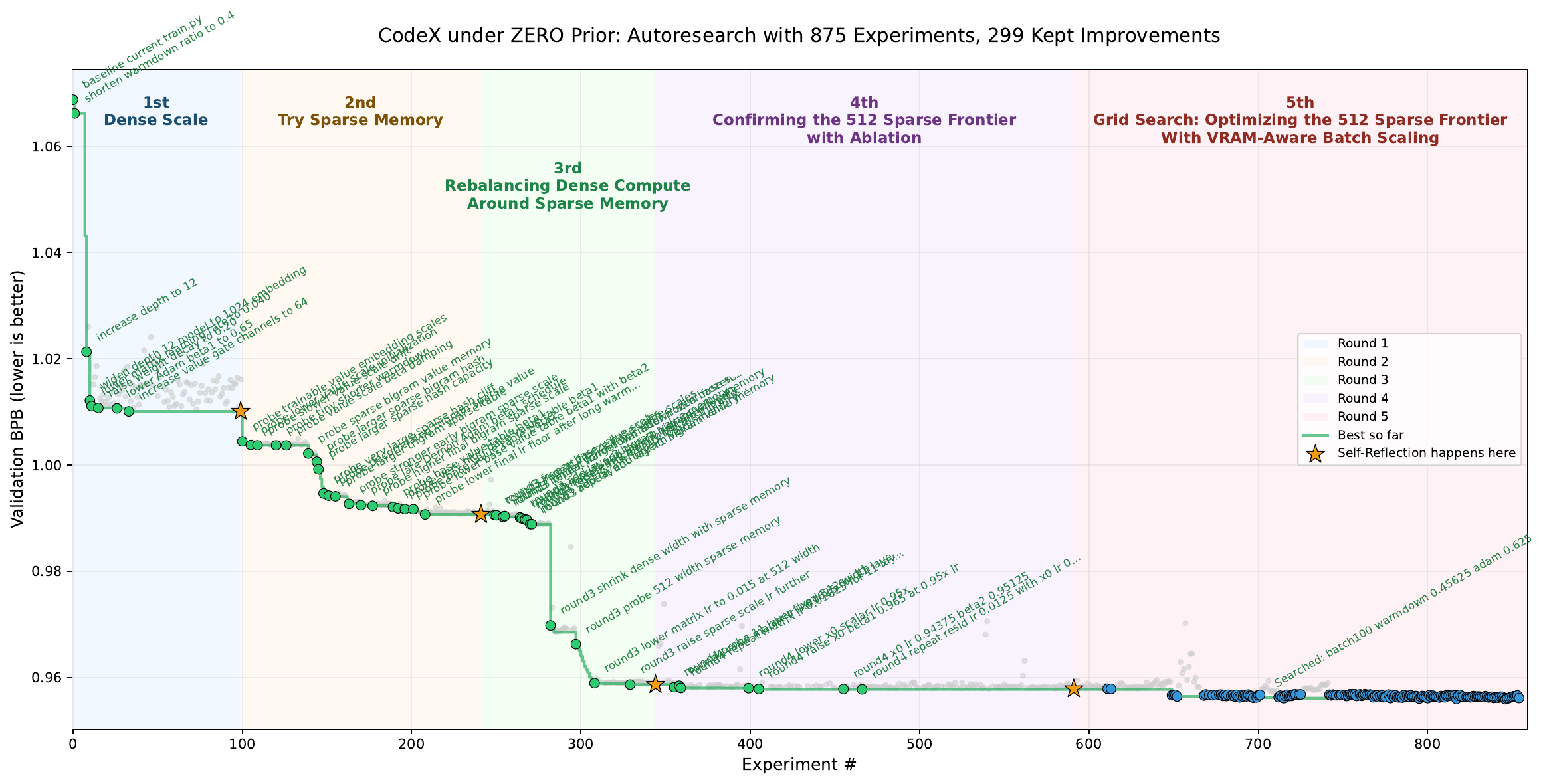}
    \caption{\textbf{Pure LLM research loop on \texttt{autoresearch}.} Validation \texttt{val\_bpb} (lower is better) across $875$ no-BO, no-acquisition experiments. Grey points are individual trials and the green step curve is the best value found so far; yellow stars mark self-reflection checkpoints where the agent summarises the ledger and enters a new research regime. The coloured bands show the loop's staged progress: dense scaling, sparse-memory invention, dense-compute rebalancing around sparse memory, ablation-based confirmation of the $512$-sparse frontier, and final VRAM-aware batch/schedule search. The curve demonstrates that a pure LLM loop can do real sequential research, but also its ceiling: after the large regime shifts it plateaus near $\texttt{val\_bpb}\approx0.956$, above the LDM result of $0.93421$ in Figure~\ref{fig:autoresearch}.}
    \label{fig:abl-pure-llm}
\end{figure}

Figure~\ref{fig:abl-pure-llm} shows that the pure LLM loop is not inert. It discovers several coherent research phases: first dense scaling, then sparse memory, then a rebalance of dense compute around the sparse memory idea, followed by confirmation and grid search around the $512$-sparse frontier. The best-so-far curve falls from the initial $\texttt{val\_bpb}\approx1.069$ to $\approx1.01$ after the first phase, to $\approx0.991$ after the second reflection, and to $\approx0.959$ after the sparse-frontier breakthrough. In other words, a modern code agent with a task-specific prompt and access to its own run log can already produce human-readable, scientifically reasonable progress: it can form hypotheses, translate them into code edits, test them, preserve improvements, and revise its research state.

The archived loop logs show what this ``research work'' consisted of. The agent did not merely enumerate nearby hyperparameters. After each phase it wrote a reflection report: it compressed the ledger into mechanistic claims, labelled memories by regime and mechanism, retrieved positive and negative examples, and used the labelled memories to choose the next experiment family. Table~\ref{tab:pure-llm-trace} summarises the resulting trajectory. The important point is the \emph{statefulness}: the loop changes what it believes the problem is. It begins with ordinary dense scaling, decides that dense capacity has saturated under the five-minute budget, invents sparse associative value memory as a cheaper capacity axis, then shrinks the dense core so that the sparse-memory model can train in time, and finally switches to repeat-first, VRAM-aware frontier tuning once improvements fall into the noise band.

\begin{table}[t]
    \centering
    \footnotesize
    \begin{tabular}{p{0.11\linewidth}p{0.20\linewidth}p{0.37\linewidth}p{0.14\linewidth}}
        \toprule
        Phase & Working hypothesis & Evidence generated by the loop & Best \texttt{val\_bpb} \\
        \midrule
        Dense scaling & Capacity is the early bottleneck, but only while the model can still be trained inside the fixed budget.
        & Depth $10$ and $12$ give large drops; widening the depth-$12$ model helps; depth $14$ worsens near the memory ceiling. The reflection turns this into the rule: dense capacity helps until the update budget becomes the bottleneck.
        & $1.010136$ \\
        Sparse value memory & Dense capacity is saturated, but cheap associative capacity may still help if routed through a bounded value path.
        & Trainable value scales, sparse hashed bigram memory, larger hash tables, decorrelated trigram memory, and phase-specific damping improve the frontier; broader routing and optimiser rewires are recorded as negative memories.
        & $0.990738$ \\
        Dense-compute rebalance & Sparse memory works best when the dense core is smaller and no longer steals the fixed training budget.
        & Causal ablations show trigram memory, early bigram memory, and trainable sparse scales matter. Shrinking dense width to $640$ and then $512$ creates the large breakthrough, after which matrix LR, final LR floor, and sparse-scale LR are retuned.
        & $0.958666$ \\
        Frontier confirmation & The $512$ sparse-memory regime is a narrow compute/update-budget island, and remaining gains require repeats.
        & Widths below and above $512$ lose; removing early ngram memory increases throughput but hurts validation, showing the memory path is causal. Restored-anchor repeats estimate a noise floor of roughly $3$--$5\times10^{-4}$ BPB.
        & $0.957766$ \\
        VRAM-aware batch tuning & Spare VRAM should buy batch/throughput and coupled schedule tuning, not a return to dense scaling.
        & The context retrieves the dense-shrink successes and negative memories against larger dense models, then grid-searches batch size, warmdown, final LR, Adam damping, and weight decay while logging steps, tokens, MFU, and loss slopes.
        & $0.955923$ \\
        \bottomrule
    \end{tabular}
    \caption{\textbf{Mechanistic trace of the pure LLM research loop.} Each row is a phase extracted from the loop logs. The LLM acts like a lightweight experimental scientist: it writes a hypothesis, tests code-level interventions, records both successes and failures, updates a regime label, and carries those memories into the next phase.}
    \label{tab:pure-llm-trace}
\end{table}

This trace is the strongest evidence that the baseline is a genuine research loop. It performs causal ablations (for example, removing trigram memory, early bigram memory, or sparse scales), estimates the noise floor by repeating restored anchors, and eventually augments the ledger with trajectory diagnostics such as steps, tokens, MFU, train/eval gap, and loss slopes. It also learns what not to do: the contexts explicitly retrieve failed routing rewires, failed dense-width probes, crashes, and repeated optimiser mismatches as negative memories. In short, the loop accumulates procedural knowledge, not just a list of winning commits.

The limitation is equally visible. After the third breakthrough, the remaining $\approx500$ experiments are mostly local confirmation and grid search around the same frontier. The loop continues to keep small improvements, but the envelope is nearly flat, ending around $\texttt{val\_bpb}\approx0.956$. This is better than the no-discover Karpathy baseline in Figure~\ref{fig:autoresearch}, which stalls at $0.9767$, but it is still far from the LDM curve, which reaches $0.93421$ under the same H100 five-minute-run setting. The pure LLM loop can reason about its own transcript and name plausible new regimes, but between those reflections it has no calibrated estimate of where improvement is likely, where uncertainty remains high, or when a plateau is evidence that the current boundary is exhausted.

This is the ablation's main conclusion. The bottleneck is not proposal expressivity: the LLM can write useful training code, exploit prior knowledge about how the architecture can be modified, and remember enough of the transcript to pursue multi-stage ideas. The bottleneck is value. Without $\mu_t$, $\sigma_t$, and the acquisition $\acq_t$, \textbf{the run log remains a narrative memory rather than a searchable value landscape}. LDM supplies exactly that missing object: the surrogate turns the same history into a posterior over the program space, and the acquisition turns that posterior into a decision value for both local refinement and boundary-moving discover steps. The pure LLM loop therefore validates the reservoir side of the framework and shows how far current agent systems can already go, while also showing why the BO-value side is necessary for the performance gains in the main experiment.

\subsection{Test-time Search for LDM}
\label{app:abl-ldm}

This ablation asks whether LDM exhibits the same test-time scaling behaviour that motivates inference-time search in language-model reasoning \citep{snell2025scaling}, but in the harder setting where the value is an expensive scientific objective approximated by a surrogate. We keep the outer evaluation budget fixed and vary only the cheap inner-loop budget of LDM-TTS. Concretely, the inner budget has four knobs: (i) how many candidate programs the LLM proposes at each outer iteration, (ii) how many hypothesis or refinement branches the BO inner loop expands for each candidate before scoring, (iii) how many of those scored candidates the acquisition retains for the next round, and (iv) which acquisition function scores them. The detailed \texttt{autoresearch} test-time-search sweep, with its N4H4/N8H8 budgets and EI vs. posterior-mean comparison, already appears in Section~\ref{sec:ablation-main}; here we keep only the molecule and CDRH3 single-step budget sweeps.

The theory in \S\ref{sec:theory-main} predicts exactly this dependence: increasing the near-acquisition mass available to the tilted sampler reduces the LDM sampling shortfall term, but only if the added candidates are also evaluated by a calibrated value model rather than by language plausibility alone.

\subsubsection{Small-molecule drug discovery}
\label{app:abl-molecule}

The molecular ablation repeats the same question in a chemically open-ended space. Here the expensive evaluation is a two-objective KRAS G12D score, and performance is the dominated Pareto hypervolume under the Vina/activity objectives of \S\ref{sec:cs-molecule-results}. We use the Direct-Softmax LDM pipeline with a Qwen3.5-9B proposer and vary two inner-loop budgets. A label \texttt{proposer}\{K\}\_\texttt{bo}\{M\} means that the LLM proposes a pool of $K$ candidate SMILES strings and the BO/EHVI inner loop is allowed to score up to $M$ candidates and retain the best-scoring ones before the next expensive molecular evaluation is chosen. Thus $K$ controls the breadth of the LLM reservoir draw, while $M$ controls how much of that breadth is exposed to the calibrated multi-objective value.

\begin{figure}[htbp]
    \centering
    \includegraphics[width=0.8\linewidth]{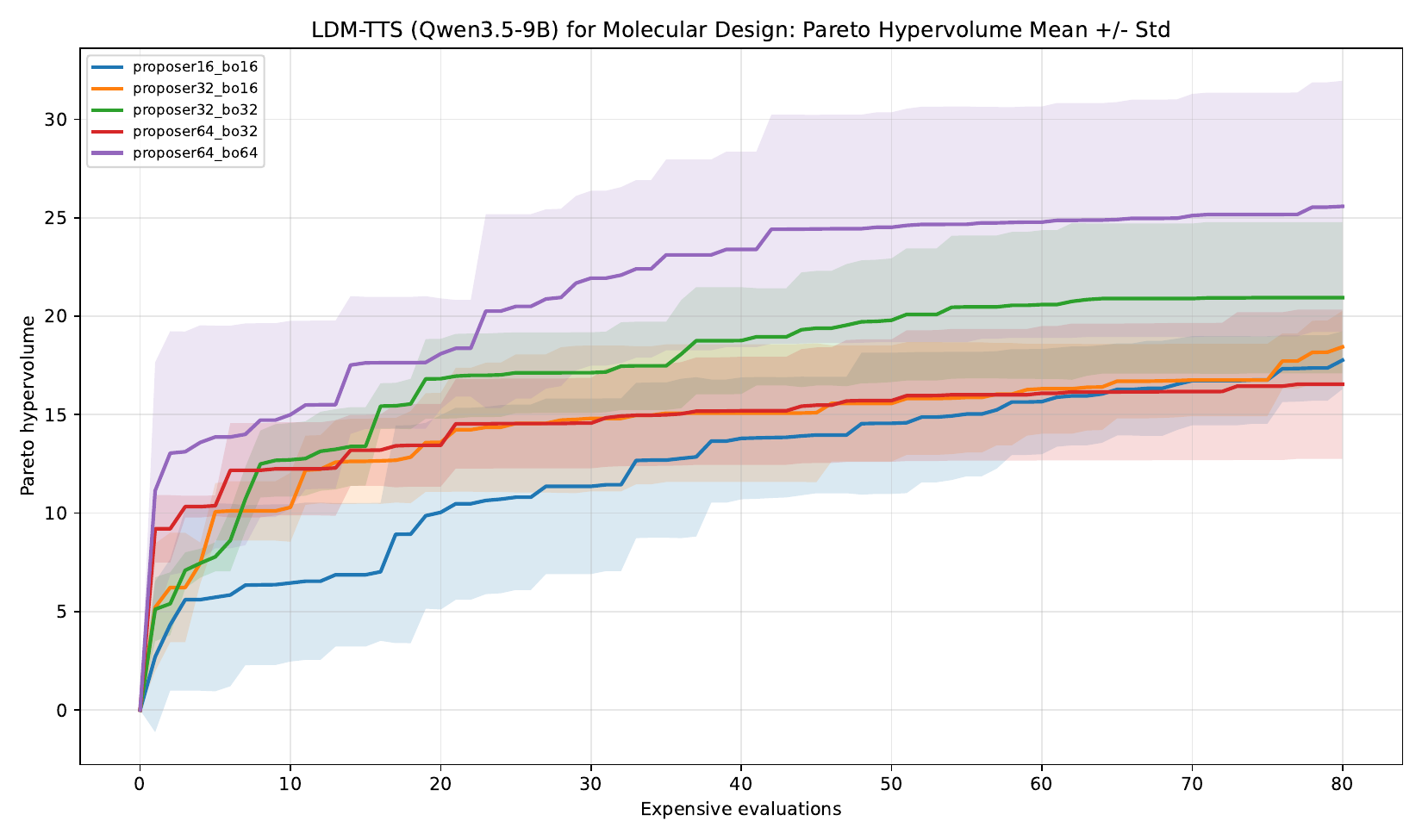}
    \caption{\textbf{Test-time search scaling for molecular design.} Pareto hypervolume (higher is better) over expensive molecular evaluations for Qwen3.5-9B Direct-Softmax LDM-TTS. Curves show the mean over independent runs, with shaded bands denoting one standard deviation; legend entries report the LLM proposal budget and BO/EHVI acquisition-scoring budget. Balanced scaling of both budgets improves hypervolume most: \texttt{proposer64\_bo64} reaches the strongest final Pareto front, while increasing proposals without matching acquisition bandwidth gives smaller or unstable gains.}
    \label{fig:abl-molecule-tts}
\end{figure}

\begin{figure}[htbp]
    \centering
    \includegraphics[width=\linewidth]{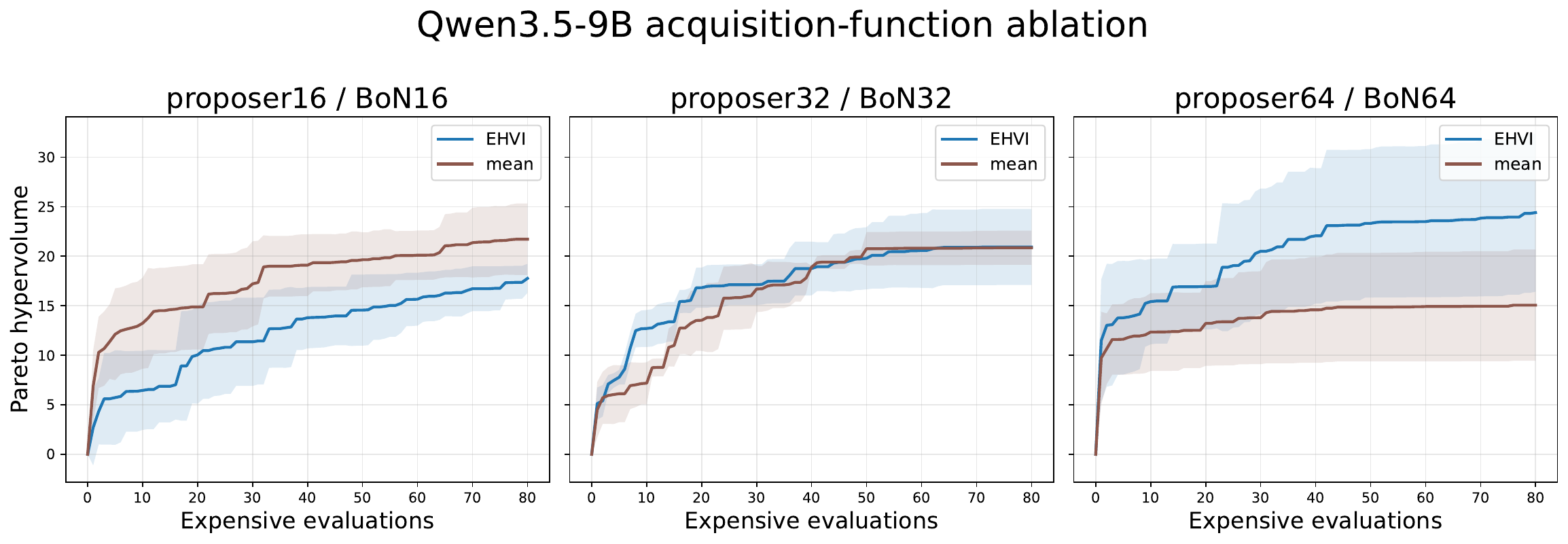}
    \caption{Abaltion on acquisition function over different search budgets for molecular design.}
    \label{fig:placeholder2}
\end{figure}

Figure~\ref{fig:abl-molecule-tts} shows that test-time compute also scales in the molecular setting, but only when both inner budgets grow together. Increasing the LLM proposal budget from $16$ to $32$ improves the curve, and increasing the BO/EHVI scoring budget from $16$ to $32$ improves it further: \texttt{proposer32\_bo32} dominates the smaller-budget settings through most of the run. \texttt{proposer64\_bo64} makes the largest early jump and keeps the highest final hypervolume; the asymmetric \texttt{proposer64\_bo32} does worse, which shows that broad SMILES generation is useful only when the EHVI side has enough bandwidth to filter the resulting candidates.

Figure~\ref{fig:placeholder2} adds the acquisition-function dimension. At small budgets, mean- or UCB-style acquisitions (a $0.5/0.5$ weighted scalarisation of the Vina and activity objectives) outperform the strict EHVI; as $n$ grows, EHVI overtakes them and continues to improve, while mean-style acquisitions degrade or stagnate. The mechanism is bandwidth: EHVI is built on the strict Pareto frontier and therefore needs a large screening budget to be effective, because multi-objective improvements correspond to sparse non-dominated points. Mean-style acquisitions collapse the two objectives into one scalar, so at small budgets the per-step randomness from a thin screening budget partially substitutes for the missing Pareto signal, smoothing out the optimisation and keeping the curve flat; once the budget grows, EHVI's strict-screening advantage takes over.

\subsubsection{Antibody CDRH3 design}
\label{app:abl-antibody}




\begin{figure}
    \centering
    \includegraphics[width=\linewidth]{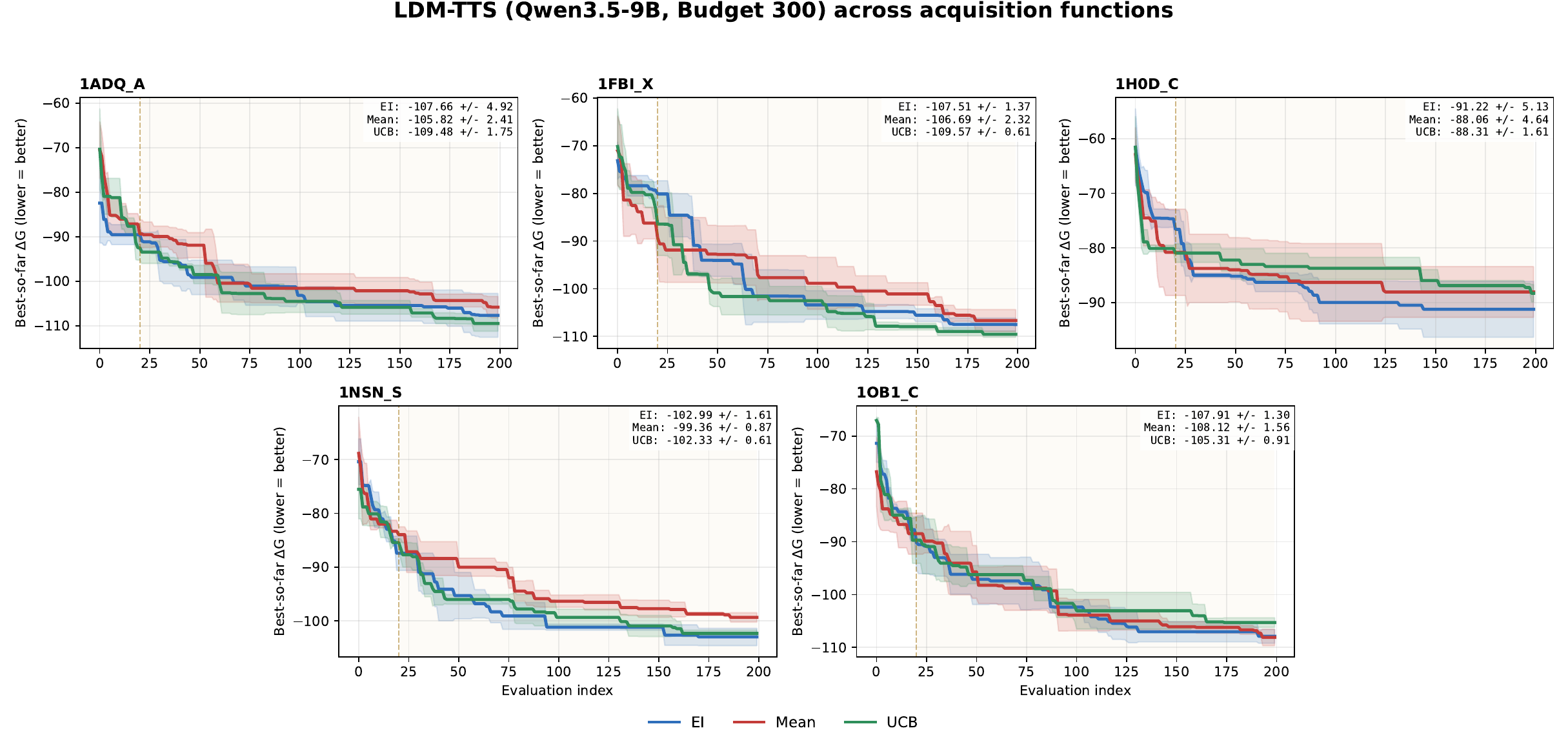}
    \caption{Comparison of different acquisition functions for LDM-TTS with a fixed proposer budget of 300 on Antibody CDRH3 design.}
    \label{fig:abl-antigen-acquisition}
\end{figure}

\begin{figure}
    \centering
    \includegraphics[width=\linewidth]{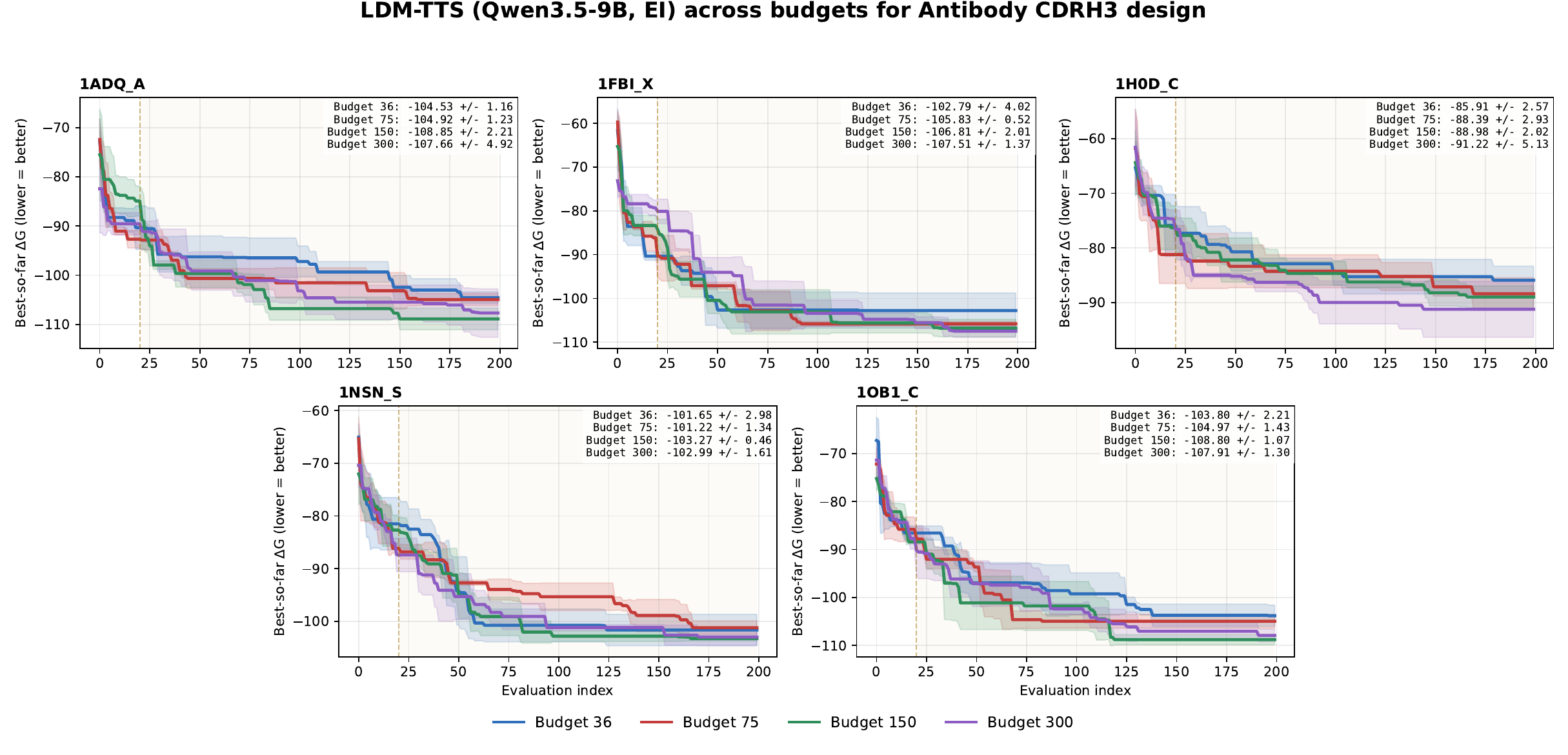}
    \caption{Effect of LDM-TTS proposer budget scaling on Antibody CDRH3 design. Each budget denotes the number of candidate sequences or sequence-neighbourhood samples generated before an expensive oracle evaluation. Larger inference-time proposal budgets improve the attainable binding-energy plateau across antigen targets.}
    \label{fig:abl-antigen-tts}
\end{figure}

Figure~\ref{fig:abl-antigen-acquisition} compares different acquisition functions on the five antigens at a fixed proposer budget of $300$. LDM's calibrated acquisitions (EI, UCB and variants) are broadly better than mean-style and random baselines, mirroring the molecular trend above. The result shows that a calibrated acquisition still extracts value from a broader candidate pool even when the open-source LLM has only weak biological priors, and qualifies the §6.5.2 antibody observation: hyperparameter sweeps have only weak sensitivity in this regime, but the choice of acquisition function still matters once the BO side has enough candidates to screen.

The CDRH3 ablation is the sequence-design counterpart of the molecular budget-scaling experiment. The outer oracle budget is again fixed, but here the design space is a known, sparse, discrete sequence space rather than an open-ended SMILES space. We therefore ask a narrower question: does the LDM-TTS proposer budget on its own improve the best sequence found? A label \texttt{Budget} \(N\) denotes an inner proposer budget of \(N\) candidate CDRH3 sequences or neighbourhood samples generated before the next expensive oracle query. The CDRH3 ablation fixes the BO inner-loop side at its maximum feasible setting: after validity checks, deduplication, and history filtering, the acquisition is allowed to score all remaining candidates. So the experiment changes only the cheap LLM-side proposal breadth, not the number of oracle evaluations.

Figure~\ref{fig:abl-antigen-tts} shows the budget effect across antigen landscapes, which is weaker than that observed in molecular design. The smallest budget, \texttt{Budget} $36$, improves quickly at the beginning but usually reaches a higher plateau. Increasing the proposer budget to $75$ already improves the final binding energy on most targets, and the larger $150$ and $300$ budgets give the best final result on all five antigens. The strongest setting is target-dependent: budget $150$ is best on 1ADQ\_A, 1NSN\_S, and 1OB1\_C, while budget $300$ is best on 1FBI\_X and 1H0D\_C. This non-monotonicity is expected in a rugged sequence landscape with finite seeds and a noisy surrogate: after the acquisition has enough candidates to cover the useful local neighbourhoods, further expansion can add redundant or lower-quality regions as well as genuinely novel ones. The main conclusion is therefore not that the largest budget always wins, but that the low-budget regime is systematically acquisition-starved and that moderate-to-large test-time pools substantially lower the attainable binding-energy plateau.

This result complements the main CDRH3 comparison in Figure~\ref{fig:cs-antibody-results}. There, policy-mode LDM outperforms direct token-level generation because the open-source LLM has a weaker biological sequence prior than it has for code or SMILES; emitting a few complete CDRH3 loops does not expose enough useful reservoir mass for BO to exploit. The ablation here isolates what happens once the LDM pipeline can expand the candidate pool: additional inference-time candidates become useful precisely because they are not selected by language plausibility alone. They are filtered by the surrogate acquisition, which converts a larger discrete reservoir into stronger best-so-far binding energy without spending extra oracle calls.

Together, the molecule and CDRH3 ablations show that test-time scaling for single-step design is useful only when the added samples are exposed to a calibrated value model. Richly pretrained SMILES models can absorb large direct proposal batches; sparse biological sequence spaces need enough sequence or neighbourhood coverage before acquisition filtering becomes effective. This complements the nanoGPT ablation: in multi-turn code search, LDM-TTS improves by searching a changing program landscape; in single-step design, it improves by scaling the candidate pool a calibrated acquisition can choose from.

\begin{figure*}[t]
    \centering
    \begin{subfigure}[t]{0.48\textwidth}
        \centering
        \includegraphics[width=\linewidth]{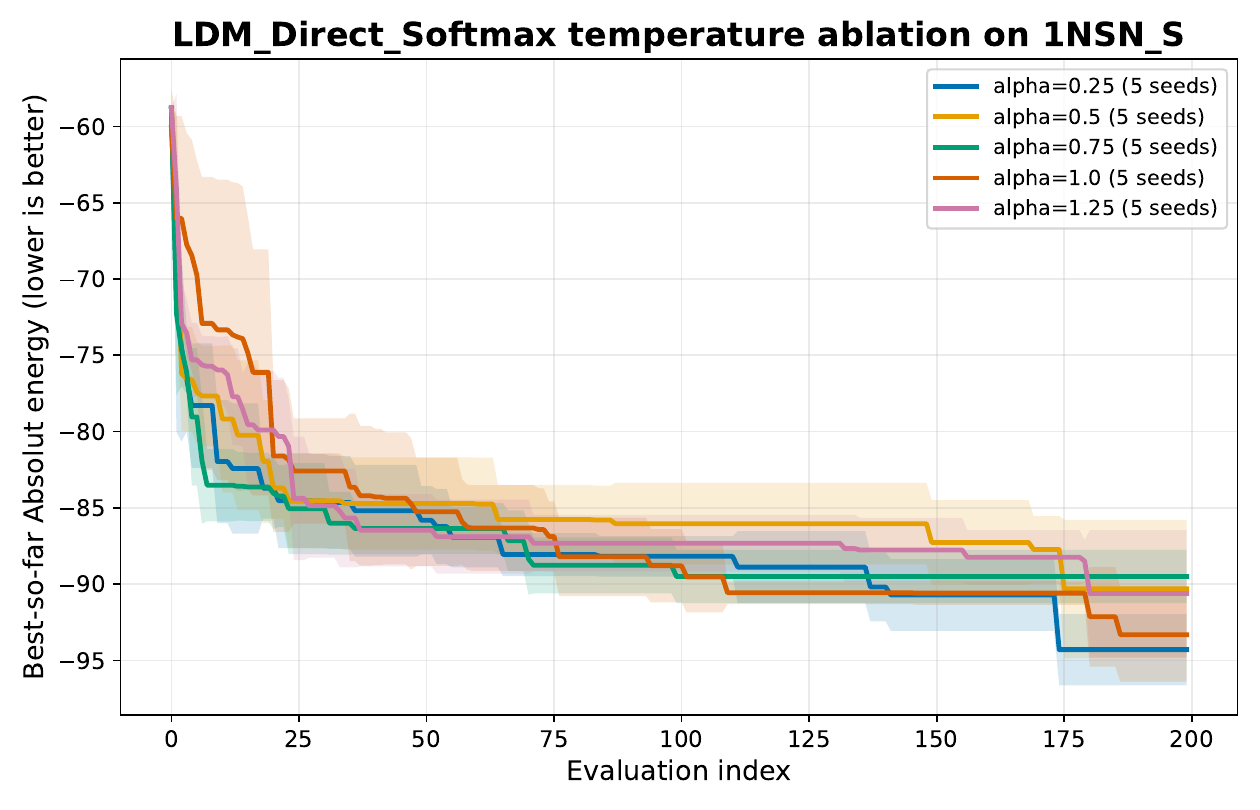}
        \caption{Effect of the LLM sampling temperature $T_{\mathrm{LLM}}$ in LDM\textsubscript{Direct}\textsubscript{Softmax}; for this controlled base-measure-exponent ablation, $T_{\mathrm{LLM}}=1/\alpha$. The acquisition temperature is fixed to $\eta=1$.}
        \label{fig:ablation-direct-temperature}
    \end{subfigure}
    \hfill
    \begin{subfigure}[t]{0.48\textwidth}
        \centering
        \includegraphics[width=\linewidth]{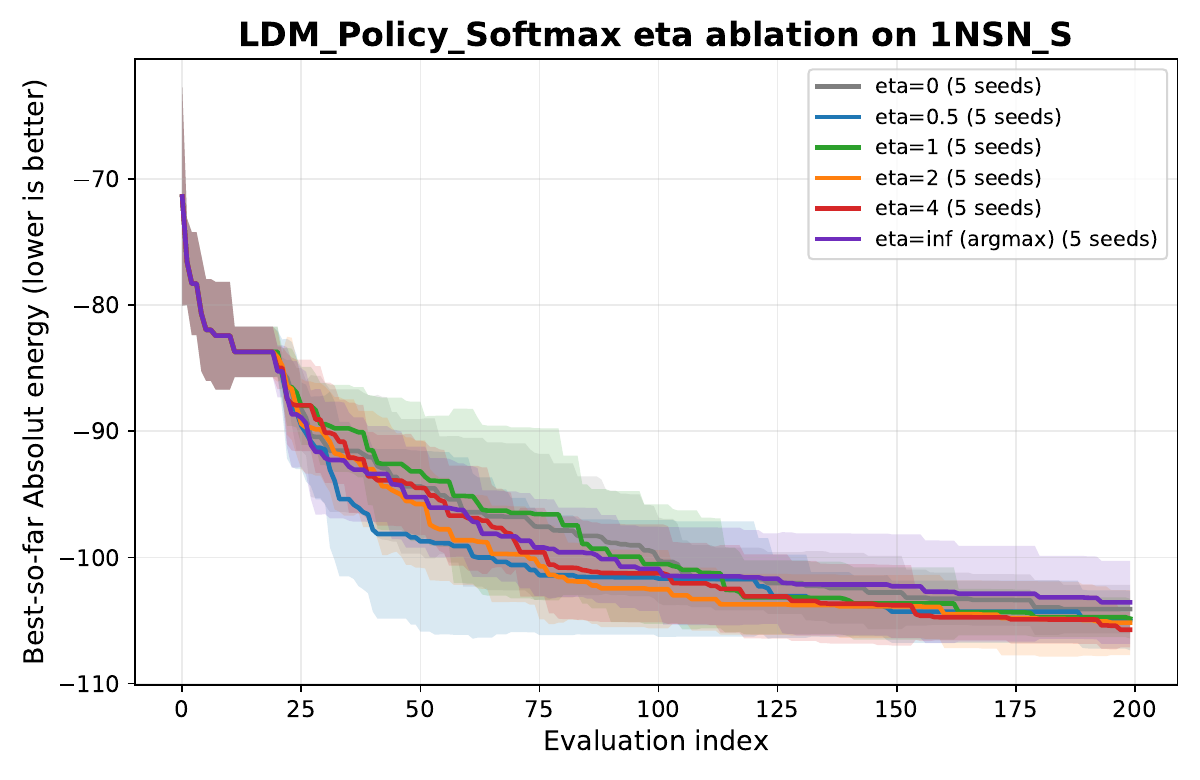}
        \caption{Effect of the acquisition temperature $\eta$ in LDM\textsubscript{Policy}\textsubscript{Softmax}. The LLM sampling temperature is fixed to $T_{\mathrm{LLM}}=1$ (and hence $\alpha=1$ in the controlled exponent setting).}
        \label{fig:ablation-policy-eta}
    \end{subfigure}

    \caption{Hyperparameter ablations on the antibody-design task $1\mathrm{NSN\_S}$. Each curve reports the mean best-so-far Absolut energy over five random seeds, with the shaded region showing one standard deviation. Lower values are better.}
    \label{fig:hyperparameter-ablations}
\end{figure*}

\subsection{Hyperparameter ablations for LDM.}
\label{app:abl-hyper-ldm}

We here supplement the sensitivity analysis of the two LDM variants on antigen $1\mathrm{NSN\_S}$ using five random seeds. For LDM\textsubscript{Direct}\textsubscript{Softmax}, we vary the LLM sampling temperature $T_{\mathrm{LLM}}$$\in\{0.25,0.5,0.75,1,1.25\}$, with $T_{\mathrm{LLM}}=1/\alpha$ for this controlled base-measure-exponent ablation, while fixing the acquisition temperature to $\eta=1$. For LDM\textsubscript{Policy}\textsubscript{Softmax}, we fix $T_{\mathrm{LLM}}=1$ and vary the acquisition temperature $\eta\in\{0,0.5,1,2,4,\infty\}$. Here, $\eta=0$ corresponds to uniform selection among the policy representatives, while $\eta=\infty$ is equivalent to acquisition-function argmax. The per-temperature curves are plotted in Figure~\ref{fig:hyperparameter-ablations}; we discuss their interpretation below.

This overall behaviour is consistent with the discussion in Section~\ref{sec:ablation-main}. Figure~\ref{fig:hyperparameter-ablations} shows the per-temperature curves. The open-source LLM lacks biological sequence domain priors, so its autoregressive proposals over the $20^{11}$ CDRH3 space are effectively near-random; in the policy variants the LLM already implements an implicit form of acquisition-aware importance sampling by expanding the candidate neighbourhood around promising incumbents rather than committing to a single sequence, which dampens the marginal effect of further reweighting at test time. Consequently, global sweeps over the sampling temperature $T_{\mathrm{LLM}}$ and the acquisition temperature $\eta$ on antigen $1\mathrm{NSN\_S}$ produce only weak sensitivity: Direct-Softmax slightly favours low $T_{\mathrm{LLM}}$ and Policy-Softmax shows a marginal edge near $\eta=4$, but all error bars overlap. Adjusting the two temperature knobs does not significantly change the final binding energy, which is consistent with the observation in \S\ref{sec:cs-antibody-results} that the policy-mode LDM gain over direct generation comes from the LLM's structural priors rather than from acquisition reweighting alone.

\begin{figure}[t]
\centering
\includegraphics[width=1\textwidth]{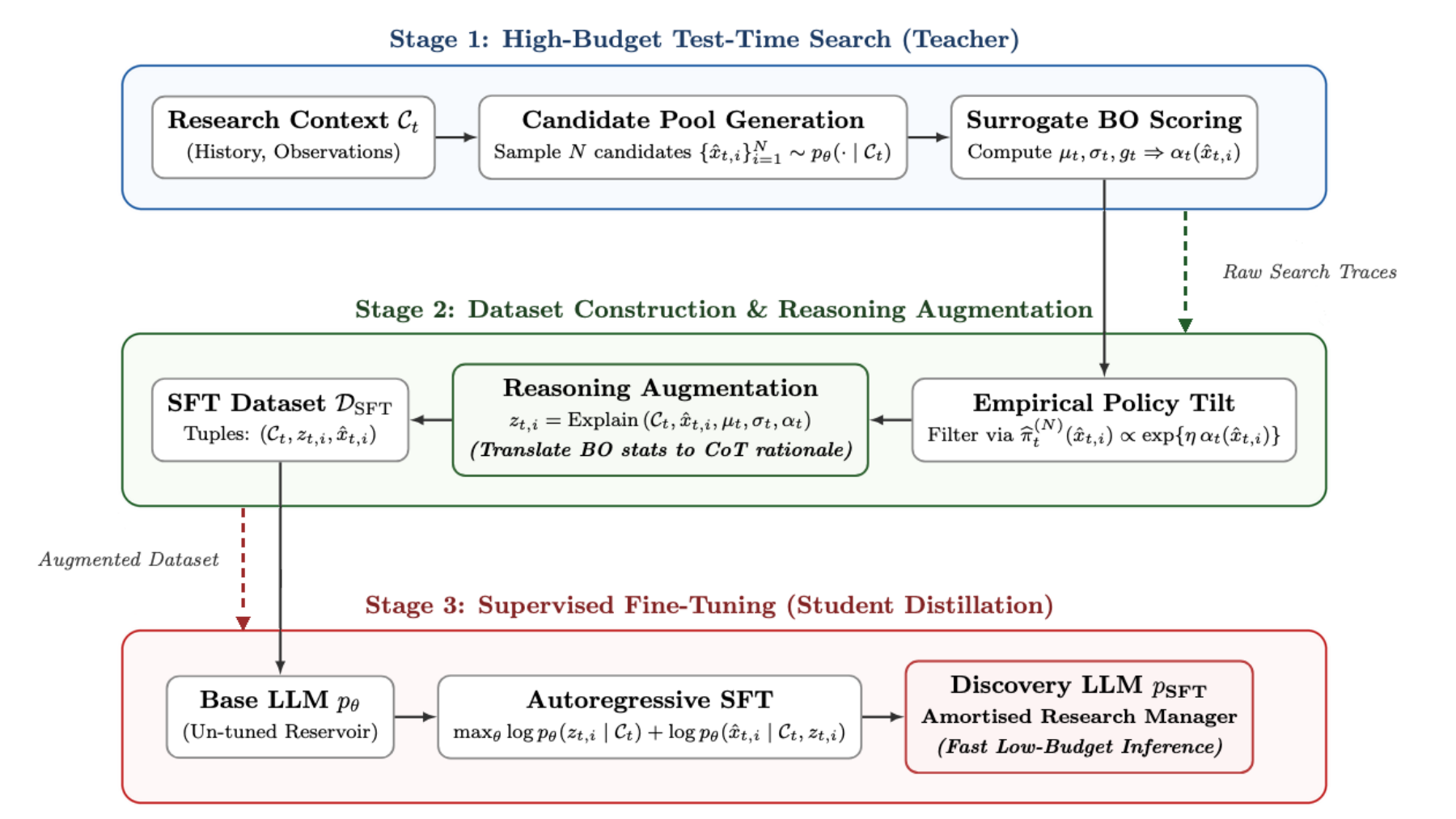}
\caption{\textbf{Overview of the LDM training pipeline workflow.}
\textbf{Stage 1:} High-budget test-time search generates candidate pools scored
by surrogate acquisition values ($\alpha_t$). \textbf{Stage 2:} High-value
candidates are filtered via an empirical policy tilt $\widehat{\pi}_t^{(N)}$,
and numerical BO statistics are translated into semantic reasoning rationales
($z_{t,i}$) via reasoning augmentation. \textbf{Stage 3:} The base model is
fine-tuned on the CoT target $(z_{t,i}, \hat{x}_{t,i})$ using standard
next-token prediction, compiling the high-budget search policy into an
amortised student model ($p_{\text{SFT}}$) capable of fast, low-budget
inference.}
\label{fig:app-training-pipeline}
\end{figure}

\section{Additional Fine-Tuning Analyses}
\label{app:finetuning-details}

Section~4.4 describes how high-budget LDM test-time search is distilled into a
student proposer, and Section~\ref{sec:finetune-experiments} reports the main fine-tuning results on
AutoResearch, antibody design, and small-molecule discovery. This appendix
provides complementary evidence that is not repeated in the main text.

\subsection{Fine-Tuning as Value Distillation}
\label{app:finetune-values}

Modern LLM fine-tuning paradigms can be distinguished by the
decision-theoretic value they distill into the model weights. Chat LLMs distill
human preference value, reasoning LLMs distill verification value, and
Discovery LLMs distill epistemic acquisition value. The object distilled by
LDM is therefore not the measured reward $R(x)$ alone, nor the static
memorisation of a high-reward molecule, protein, or program. Instead, the goal
is to distill the acquisition strategy that decides which scarce experiment is
worth running under uncertainty. Figure~\ref{fig:app-training-pipeline} summarises this three-stage training pipeline: high-budget test-time search, reasoning-augmented dataset construction, and supervised distillation into the student proposer.

Table~\ref{tab:reasoning-augmentation} illustrates how numerical acquisition
evidence is translated into semantic supervision across the three discovery
domains.

\begin{table}[H]
\centering
\small
\begin{tabular}{@{}p{0.18\linewidth}p{0.36\linewidth}p{0.40\linewidth}@{}}
\toprule
\textbf{Domain} & \textbf{Acquisition evidence (Numerical)} & \textbf{Example augmented rationale (Semantic)} \\
\midrule
\texttt{autoresearch} code search &
Best run \texttt{state\_0019} reaches \texttt{val\_bpb} $0.9811$ by lowering \texttt{TOTAL\_BATCH\_SIZE} and tuning \texttt{ASPECT\_RATIO}; several near-best runs also lower the batch size. The current parent already sits at the minimum batch size with only $5$ active parameters and no improvement over the last $4$ rounds. &
Lowering the total batch size has been a robust positive move---more steps fit inside the fixed $30$\,s budget---but that lever is now exhausted and the active parameters have \textbf{plateaued over the last four rounds}, so the set is near its ceiling. Rather than re-tuning saturated knobs, I should \textbf{activate a currently-inactive parameter} (e.g.\ \texttt{MATRIX\_LR}) to open a new optimisation dimension. \\
\addlinespace
Antibody CDRH3 design &
Best sequence \texttt{ADGHTKQNPRW} at $-64.28$; the last four candidates share the fixed prefix \texttt{ADGHTKQNPR} and scan only the final position (W $-64.28$, Y $-62.69$, A $-59.38$, S $-57.96$). &
The history is a single-point scan of the last residue on an already-refined scaffold, with the bulky aromatic W best so far. The informative move is to keep the refined prefix and test an \textbf{untested residue of similar hydrophobic character}---leucine (L)---which may reproduce W's favourable hydrophobic contact; propose \texttt{ADGHTKQNPRL}. \\
\addlinespace
Small-molecule discovery &
$29$ molecules evaluated and the loop is flagged \texttt{stalled}; the top actives ($\sim\!8.1$) share one isoxazolone/thiazolone fused core with isopropoxy and cyano groups, docking $\sim\!-8$. &
Because the loop is stalled and near-duplicates are disallowed, adding a methyl to a known winner carries no information. The higher-value move is a \textbf{recombination or scaffold hop}---e.g.\ combining the pyridine-ring variant with the morpholine side chain into a molecule not yet seen---so the candidate probes a genuinely new region of the Pareto frontier rather than exploiting a saturated one. \\
\bottomrule
\end{tabular}
\vspace{0.5em}
\caption{\textbf{Reasoning augmentation turns numerical acquisition labels into semantic training signals.}
The rationales are faithful (translated and condensed) excerpts of reasoning traces generated by the self-hosted DeepSeek V4 Flash teacher inside the high-budget LDM-TTS loop, one per domain. They describe the epistemic value of managing research progress by reading the evaluated history, diagnosing stagnation, and choosing an information-adding move, rather than focusing only on the domain-specific final solution.}
\label{tab:reasoning-augmentation}
\end{table}

Beyond the aggregate translation above, a single deployment-time trace shows the
student \emph{reading} the surrogate signal directly. On the stalled KRAS G12D
loop, it inspects the per-candidate posterior before acting:

\begin{tcolorbox}[colback=black!3,colframe=black!25,boxrule=0.4pt,arc=3pt,
  left=7pt,right=7pt,top=5pt,bottom=5pt]
\small
\emph{Small-molecule discovery:} ``the
strong scaffolds share a cyclopropyl--benzene core with trifluoromethoxy and
nitrile (Vina $\approx-8.1$, activity $\approx7.0$), so I keep that backbone; but
\textbf{the surrogate's activity uncertainty is still high ($\approx0.3$--$0.5$),
so there is room to explore}, so I vary the polar group (N, N-dimethylamide,
sulfonamide) and swap the ring (thiophene, furan, thiazole, pyridine)\ldots'',
then proposing candidates that keep the proven scaffold while probing the
high-uncertainty regions.
\end{tcolorbox}

\noindent The student thus grounds its explore/exploit decision on the
surrogate's predicted mean and uncertainty: it exploits the proven scaffold
while steering exploration toward the coordinates the surrogate is least certain
about, exactly the acquisition behaviour the fine-tuning is meant to distil.

\subsection{Cross-Molecule Transfer without Reasoning Augmentation}
\label{app:finetune-no-augmentation}

We test the minimal form of discovery fine-tuning without an explicit
reasoning target. Training data are collected from high-budget LDM-TTS traces
on a source molecular optimisation task. The Qwen3.5-9B proposer is fine-tuned
on acquisition-weighted examples and then evaluated on a different molecular
optimisation task within the same acquisition-guided LDM loop.

This ablation retains the same LDM state and action schema as the
reasoning-augmented setting, but omits the explicit research-progress
rationale $z_{t,i}$.

This setting shows positive transfer across molecular tasks. The base
Qwen3.5-9B LDM reaches a final hypervolume of $16.489\pm5.867$, whereas the
fine-tuned model reaches $22.279\pm3.828$ under the same evaluation protocol.
Because the discovery data are collected from one molecule and evaluated on
another, the improvement cannot be explained by memorising the final molecules
from the source task. It instead shows that acquisition-weighted fine-tuning
can transfer part of the research-management policy even without an explicit
reasoning channel.

\subsection{Supplementary Quantitative Tables}
\label{app:finetune-tables}

Table~\ref{tab:app-smallmol-cot} retains the G12C results not shown in the
main-text KRAS G12D learning curve. Table~\ref{tab:app-finetune-cot-ablation}
collects the numerical cross-domain comparison underlying the three main-text
figures.

\begin{table}[H]
\centering
\small
\begin{tabular}{@{}lcc@{}}
\toprule
\textbf{Inference policy} & \textbf{KRAS G12C} & \textbf{KRAS G12D} \\
 & HV $\uparrow$ & HV $\uparrow$ \\
\midrule
Qwen3.5-9B base (without\_cot) & $11.64\pm2.63$ & $16.49\pm5.87$ \\
Fine-tuned Qwen3.5-9B (without\_cot) & $14.73\pm3.09$ & $22.28\pm3.83$ \\
DeepSeek V4 Flash (without\_cot) & $18.97\pm1.76$ & $24.55\pm3.59$ \\
DeepSeek V4 Flash (with\_cot) & $19.98\pm2.00$ & $24.07\pm3.12$ \\
\textbf{Fine-tuned Qwen3.5-9B (with\_cot)} & $\mathbf{20.98\pm2.92}$ & $\mathbf{26.66\pm2.61}$ \\
\bottomrule
\end{tabular}
\caption{\textbf{Reasoning-augmentation ablation on small-molecule discovery.}
Final Pareto-front hypervolume at an evaluation budget of 80, averaged over five seeds.}
\label{tab:app-smallmol-cot}
\end{table}

\begin{table}[H]
\centering
\setlength{\tabcolsep}{4pt}
\scriptsize
\begin{tabular}{@{}lcccc@{}}
\toprule
\textbf{Model} & \textbf{CoT} & \textbf{AutoResearch} & \textbf{Antibody} & \textbf{SmallMol} \\
 & & val\_bpb $\downarrow$ & $-E_{\mathrm{bind}}$ $\uparrow$ & HV $\uparrow$ \\
\midrule
Base Qwen3.5-9B & off & $0.9965$ & $102.6$ & $16.49$ \\
Base Qwen3.5-9B & on  & $0.9825$ & $104.4$ & --- \\
LDM-SFT & off & $0.9829$ & $105.1$ & $22.28$ \\
\textbf{LDM-SFT} & \textbf{on} & $\mathbf{0.9788}$ & $\mathbf{105.5}$ & $\mathbf{26.66}$ \\
\midrule
\multicolumn{2}{@{}l}{Qwen3-Coder-30B ref.} & $0.9801$ & --- & --- \\
\bottomrule
\end{tabular}
\caption{\textbf{Cross-domain reasoning-augmentation ablation.}
Terminal metrics for the three fine-tuning studies.}
\label{tab:app-finetune-cot-ablation}
\end{table}

\subsection{In-Distribution Single-Task Fit}
\label{app:finetune-indistribution}

Before evaluating the distilled policies in the acquisition loop, we verify
that each single-task model fits its reasoning-augmented training data. We
track training and held-out losses for \texttt{autoresearch} program search,
small-molecule KRAS design, and antibody CDRH3 design. All three losses
converge within a few dozen steps, while the held-out curves closely track the
training curves. The absolute held-out losses differ by domain: $0.06$ for
\texttt{autoresearch}, $0.38$ for small-molecule design, and $0.70$ for
antibody CDRH3 design. These differences reflect the different action spaces
and target lengths, rather than a failure to fit the data.

\begin{figure}[H]
\centering
\includegraphics[width=0.95\linewidth]{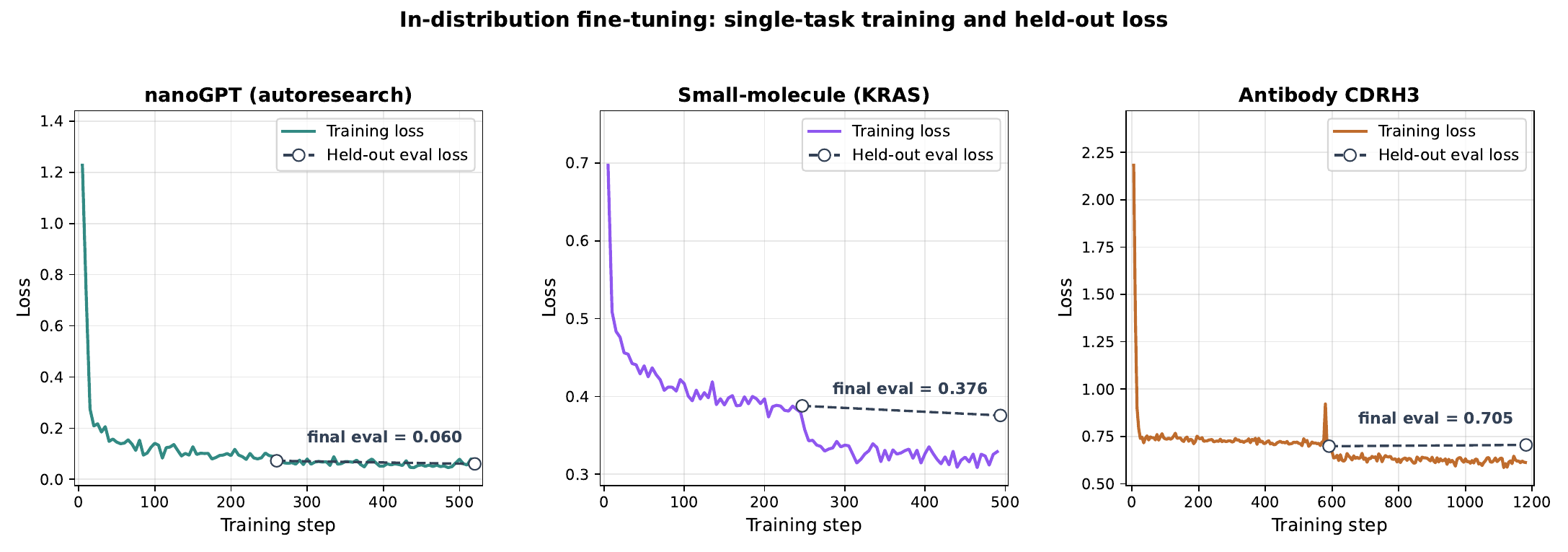}
\caption{\textbf{In-distribution single-task fine-tuning.} Training loss
(solid) and held-out evaluation loss (markers) versus training step for the
\texttt{autoresearch}, small-molecule, and antibody CDRH3 policies. All three
converge within a few dozen steps, and the held-out loss tracks the training
loss, confirming a faithful in-distribution fit before acquisition-loop
evaluation.}
\label{fig:app-indist-curves}
\end{figure}

\subsection{Case-level Out-of-distribution Generalisation}
\label{sec:finetune-ood}

To test whether it learns
a \text{transferable} acquisition policy rather than memorising antigen-specific patterns, we hold two of the
five antibody targets (\texttt{1FBI\_X}, \texttt{1H0D\_C}) out of training entirely, fine-tune the mixed-task
model on the other three, and evaluate on the held-out pair inside the same acquisition-guided loop, antigens
whose sequences and binding landscapes were never seen during training.

Figure~\ref{fig:protein-ood-mixed32} shows the out-of-distribution model matches or exceeds the
\emph{in-distribution} model (which \emph{was} trained on these antigens) on both held-out targets, and clearly
beats base Qwen3.5-9B with and without chain-of-thought: it reaches $-110.57$ on \texttt{1FBI\_X}
(base $-109.5$, in-distribution $-113.4$) and $-95.4$ on \texttt{1H0D\_C}, where it even surpasses the
in-distribution model ($-94.6$) and the base ($-90.9$). Since the model never observed these antigens, the gain
cannot be memorisation---the distilled acquisition policy transfers to unseen targets. 

\subsection{Task-level Out-of-distribution Generalisation}
\label{app:task-ood}

The held-out-antigen study of Section~\ref{sec:finetune-ood} still keeps antibody design inside the training
mixture; it only withholds particular antigens. A more demanding question is whether the distilled behaviour
survives when the entire protein task is removed. To answer it we fine-tune a chain-of-thought model on the \texttt{nanoGPT} and small-molecule trajectories, so that it never encounters an antibody, an
Absolut binding landscape, or a single CDRH3 sequence, and then drop it, unchanged, into the antibody acquisition
loop across all five antigens. Whatever competence the model shows here cannot be protein knowledge it does not
possess; it can only be a task-agnostic acquisition policy carried over from the other two domains.

The outcome, in Figure~\ref{fig:app-task-ood}, is striking. On four of the five targets this model performs on
par with the \emph{in-distribution} model that was trained directly on these antigens, reaching $-113.1$ on
\texttt{1FBI\_X} (against $-113.4$), $-110.3$ on \texttt{1ADQ\_A} ($-109.4$), $-106.6$ on \texttt{1OB1\_C}
($-105.5$), and $-104.2$ on \texttt{1NSN\_S} ($-104.7$). A model that has never seen an antibody thus matches one
trained on the very antigens it is tested on. The single place it falls short is \texttt{1H0D\_C}, where it
reaches only $-89.4$, no better than the base proposer ($-90.9$) and well behind the in-distribution model
($-94.6$), and with noticeably higher variance across seeds.

That the transferred behaviour is genuinely acquisition-guided is visible in the traces themselves. Even the
task-level model, which never saw a protein, conditions on the search state, anchors on the best result so far,
and proposes a move aimed at improving on it:

\begin{tcolorbox}[colback=black!3,colframe=black!25,boxrule=0.4pt,arc=3pt,left=7pt,right=7pt,top=5pt,bottom=5pt]
\small
\emph{Task-level OOD model}: ``the current best sequence is \texttt{LQWVSYGIPIS} at
$-86.83$, which gives me a baseline; the other observed sequences score between $-70$ and $-79$, so this target
is sensitive to small changes and I should look for a variant that approaches or surpasses the best\ldots'', then
proposing \texttt{QTFQNDKGLFI}, which keeps a hydrophobic core (F, L, I) together with hydrogen-bonding (Q, N) and
charged (K, D) residues.
\end{tcolorbox}

Reading the chain-of-thought traces further makes the one exception clear. Inside the loop the model plays two roles at once: it
\emph{selects} among candidates according to the acquisition signal, and it \emph{authors} new candidates
\emph{de novo}. These two roles come apart precisely on \texttt{1H0D\_C}, whose landscape is unusually shallow
(its best attainable energy is around $-95$, against $-104$ to $-113$ for the other targets) and whose optimum
lies in a narrow basin that can be reached only by deliberately designing a specific aromatic/hydrophobic motif.
Faced with the same situation, the two models reason quite differently:

\begin{tcolorbox}[colback=black!3,colframe=black!25,boxrule=0.4pt,arc=3pt,left=7pt,right=7pt,top=5pt,bottom=5pt]
\small
\emph{In-distribution model:} ``the best sequences mix a hydrophobic core (L, V, I, F) with polar and aromatic
residues\ldots let me construct a sequence: start with Y for $\pi$-stacking, keep a hydrophobic core in the
middle, and place an aromatic residue at the terminus\ldots'', yielding \texttt{YRMQDQPWSLQ}, a sequence that is
absent from the candidate pool.

\smallskip
\emph{Task-level OOD model:} ``I need to pick one sequence from the candidate pool of more than a thousand\ldots I
notice that \texttt{HLFQSCVSGPL} is already in the pool (id 169); it carries suitable hydrophobic and polar
residues and violates none of the constraints, so I will choose it.''
\end{tcolorbox}

\noindent Where the in-distribution model composes the required motif from scratch, the task-level model has
never learned the sequence priors of the protein domain, and so retreats to picking the most promising candidate
the pool already offers; on \texttt{1H0D\_C} this ceiling coincides with the base model's. On the remaining four
antigens the optima can be found by selection alone, so carrying over the acquisition policy is sufficient and no
gap appears.

Taken together, the task-level experiment separates two abilities that ordinary fine-tuning leaves entangled.
What our distillation transfers is the one we set out to capture, namely knowing how to \emph{use} the
acquisition signal to explore and exploit, and it does so across tasks that share nothing but that structure.
What it does not, and should not be expected to, transfer is \emph{de novo} domain design, which depends on
in-domain data and returns the moment such data is available, as both the in-distribution and case-level results
confirm. This clean separation, in which acquisition generalises while domain-specific design does not, accounts
for exactly where the task-level model succeeds and where it does not, and shows that the recipe distils
precisely what ought to be transferable. The corresponding learning curves for both OOD studies are shown in
Figure~\ref{fig:protein-ood-mixed32} and Figure~\ref{fig:app-task-ood} in Section~\ref{sec:finetune-experiments}.

\subsection{Distilling Acquisition Value into the Weights}
\label{app:gp-vs-nogp}

In a standard model-based loop the epistemic value lives \emph{outside} the model: a Gaussian process computes
the posterior mean, uncertainty, and acquisition value, and the language model simply reads those numbers off and
acts. The proposer is a value \emph{reader}, not a value \emph{holder}; take the surrogate away and its
competence goes with it. What we are really after is the opposite. We want the acquisition value distilled into
the model's own weights, so that the proposer develops a genuine discovery intuition: an internal,
experience-formed sense of which regions are exhausted, which still hold uncertainty or information, and whether
the moment calls for exploiting the current best or striking out somewhere new. This is why we hide the GP
values during fine-tuning. Denied the numbers, the model has to learn the value$\Leftrightarrow$language
mapping: it must put the epistemic situation into words and act on it (``the last rounds repeated with no gain, so
this region is saturated; try a new direction'') rather than consume a scalar such as $\sigma=0.004$. Hiding the
surrogate is not a handicap but the whole point: at deployment the acquisition loop hands the proposer no GP
values anyway, so the intuition has to live in the model itself; and once it does, it becomes a task-agnostic
search skill that carries over to discovery problems the model was never trained on
(Section~\ref{app:task-ood}). It is, in the end, the language-model counterpart of an experienced researcher who
\emph{senses} that a line of inquiry is played out, instead of computing an expected-improvement score.

Concretely, reasoning augmentation can be run at two levels. In the first, the surrogate's posterior mean,
uncertainty, and acquisition value are embedded in the prompt, so the model can read the epistemic state
directly. In the second, which is the variant we adopt, these Gaussian-process values are \emph{hidden}:
the prompt carries only the observed history of candidates actually tried and their real outcomes, and the model
must infer the epistemic state itself. The two regimes reach the \emph{same} acquisition decisions; what differs
is where the uncertainty signal comes from. The traces below, both emitted on the \texttt{autoresearch} task
immediately before an explore move, make the contrast concrete.

\begin{tcolorbox}[enhanced,breakable,colframe=reasonblueline,colback=reasonblue!25,boxrule=0.9pt,arc=4pt,
    left=8pt,right=8pt,top=7pt,bottom=7pt]
{\bfseries With GP values in the prompt: uncertainty is \emph{read}.}\\[2pt]
\small
``Looking at the surrogate's feedback: the latest candidate \texttt{state\_1808} was predicted at
mean $=0.98147$ with very low uncertainty, so the surrogate is confident about this region; yet the true result
$0.98209$ did not improve. So I should step outside the region the surrogate is most certain about and explore
one it is less sure of.''
\par\smallskip
{\itshape The epistemic signal is taken directly from the GP's numeric mean and uncertainty.}
\end{tcolorbox}

\begin{tcolorbox}[enhanced,breakable,colframe=discgreenline,colback=discgreen!18,boxrule=0.9pt,arc=4pt,
    left=8pt,right=8pt,top=7pt,bottom=7pt]
{\bfseries GP hidden (ours): uncertainty is \emph{inferred from experience}.}\\[2pt]
\small
``The best is $0.98456$, but the spread across all 18 evaluated results is only ${\sim}0.0002$, likely just
run-to-run noise, and there has been no improvement for 15 consecutive rounds. \texttt{HEAD\_DIM} and
\texttt{WINDOW\_PATTERN} no longer give a meaningful signal, so I will open a new dimension
(\texttt{DEVICE\_BATCH\_SIZE}).''
\par\smallskip
{\itshape With no GP values in the prompt, the model reconstructs the same ``this region is exhausted /
low-information'' judgement from the raw observed history, then makes the same explore decision.}
\end{tcolorbox}

The decision is identical; only the provenance of the uncertainty estimate changes. This matters less than it
might seem, because the GP numbers are barely used even when they are supplied: across the reasoning-augmented
training traces, only $1.3\%$ explicitly cite a surrogate number, whereas $97.2\%$ reason from the observed
outcomes such as stall length, noise level, and best-so-far. Hiding the GP values therefore changes the reasoning almost
not at all, while buying two concrete advantages: the training prompt now matches deployment, where the
acquisition loop exposes no GP values to the proposer, and, empirically, out-of-distribution transfer improves
(Section~\ref{app:task-ood}).

Table~\ref{tab:gp-vs-nogp-results} confirms the last point quantitatively. Both fine-tuned proposers are
task-level out-of-distribution, trained only on \texttt{nanoGPT} and small-molecule trajectories with no
protein data, and they differ only in the augmentation regime: GP values embedded in the prompt versus GP hidden
(ours). Hiding the GP values does not hurt and in fact helps: the GP-hidden proposer attains the best average
binding energy and the best result on three of the five antigens, including a decisive $-115.9$ on
\texttt{1FBI\_X} and a clear recovery on the hard \texttt{1H0D\_C} target ($-93.1$ versus $-89.4$).

\begin{table}[H]
\centering
\setlength{\tabcolsep}{10pt}
\renewcommand{\arraystretch}{1.2}
\begin{tabular}{@{}lcccccc@{}}
\toprule
\textbf{Proposer} & \texttt{1ADQ} & \texttt{1FBI} & \texttt{1H0D} & \texttt{1NSN} & \texttt{1OB1} & \textbf{Avg} \\
\midrule
Base +CoT               & $-107.9$ & $-109.5$ & $-90.9$ & $-104.8$ & $-109.1$ & $-104.4$ \\
In-distribution FT      & $-109.4$ & $-113.4$ & $\mathbf{-94.6}$ & $-104.7$ & $-105.5$ & $-105.5$ \\
\midrule
\multicolumn{7}{@{}l}{\itshape Task-level OOD}\\
With GP                 & $\mathbf{-110.3}$ & $-113.1$ & $-89.4$ & $-104.2$ & $-106.6$ & $-104.7$ \\
\textbf{Without GP}     & $-110.2$ & $\mathbf{-115.9}$ & $-93.1$ & $\mathbf{-104.9}$ & $\mathbf{-110.4}$ & $\mathbf{-106.9}$ \\
\bottomrule
\end{tabular}
\caption{\textbf{GP-in-prompt vs GP-hidden reasoning augmentation.} Best Absolut $-E_{\mathrm{bind}}$ on five
held-out antibody targets; the two OOD rows differ only in whether GP values appear in the prompt. Best per
column in bold. Baselines and GP-in-prompt: 3 seeds; GP-hidden: seed~42.}
\label{tab:gp-vs-nogp-results}
\end{table}

\stopcontents[appendix]

\end{document}